%% file: DRO-MARL-finite.tex
\documentclass[10pt,english]{article}
\usepackage{geometry}
\usepackage[T1]{fontenc}
\usepackage[latin9]{inputenc}
\usepackage{bm}
\usepackage{amsmath}
\usepackage{amssymb} 
\usepackage[unicode=true,
 bookmarks=false,
 breaklinks=false,pdfborder={0 0 1},colorlinks=false]
 {hyperref}
\hypersetup{
 colorlinks,citecolor=blue,filecolor=blue,linkcolor=blue,urlcolor=blue}
\usepackage{xcolor,colortbl}
\definecolor{Gray}{gray}{0.85}
\usepackage{enumitem}
\makeatletter
\usepackage{amsthm}
\usepackage{cite}  
\usepackage{comment}
\usepackage[sort,compress]{natbib}
\usepackage{booktabs,mathtools}
\usepackage{graphicx}
\usepackage[linesnumbered,ruled,vlined]{algorithm2e}
\usepackage{algorithmic}

\usepackage{multirow}
\usepackage{dsfont}
\usepackage{color}
\usepackage{float}

\definecolor{yjc}{RGB}{225,0,100}
\definecolor{lxs}{RGB}{138,43,226}

\definecolor{own_pink}{RGB}{217,25,169}
\definecolor{own_blue}{RGB}{0,100,223}

\definecolor{own_pink}{RGB}{217,25,169}
\definecolor{own_blue}{RGB}{0,100,223}

\allowdisplaybreaks

\input{macro.tex}

\title{Sample-Efficient Robust Multi-Agent Reinforcement Learning \\in the Face of Environmental Uncertainty}

 \author{
	Laixi Shi\thanks{Department of Computing Mathematical Sciences, California Institute of Technology, CA 91125, USA.}\\
 	Caltech 
 	\and
 	Eric Mazumdar\footnotemark[1] \\ 	 
	Caltech
	\and
 	Yuejie Chi\thanks{Department of Electrical and Computer Engineering, Carnegie Mellon University, Pittsburgh, PA 15213, USA.} \\ 
	CMU 
	\and
	Adam Wierman\footnotemark[1] \\
	Caltech
 	} 

\date{\today}
\begin{document}
\theoremstyle{plain} \newtheorem{lemma}{\textbf{Lemma}}
\newtheorem{proposition}{\textbf{Proposition}}
\newtheorem{theorem}{\textbf{Theorem}}
\newtheorem{assumption}{Assumption}
\newtheorem{corollary}{Corollary}[theorem] 

\theoremstyle{remark}\newtheorem{remark}{\textbf{Remark}}

\maketitle
 
  \sloppy
 \begin{abstract}
To overcome the sim-to-real gap in reinforcement learning (RL), learned policies must maintain robustness against environmental uncertainties. While robust RL has been widely studied in single-agent regimes, in multi-agent environments, the problem remains understudied---despite the fact that the problems posed by environmental uncertainties are often exacerbated by strategic interactions. This work focuses on learning in distributionally robust Markov games (\rmgs), a robust variant of standard Markov games, wherein each agent aims to learn a policy that maximizes its own worst-case performance when the deployed environment deviates within its own prescribed uncertainty set. This results in a set of robust equilibrium strategies for all agents that align with classic notions of game-theoretic equilibria. Assuming a non-adaptive sampling mechanism from a generative model, we propose a sample-efficient model-based algorithm (\DRNVI) with finite-sample complexity guarantees for learning robust variants of various notions of game-theoretic equilibria. We also establish an information-theoretic lower bound for solving RMGs, which confirms the near-optimal sample complexity of \DRNVI with respect to problem-dependent factors such as the size of the state space, the target accuracy, and the horizon length.
\end{abstract}


\allowdisplaybreaks

\setcounter{tocdepth}{2}
\tableofcontents

\input{intro}

\input{background}
\input{general-formulation}

\input{results}

\input{related-work}

\input{conclusion}


\section*{Acknowledgement}

The work of L. Shi is supported in part by the Resnick Institute and Computing, Data, and Society Postdoctoral Fellowship at California Institute of Technology. The work of Y. Chi is supported in part by the grants ONR N00014-19-1-2404 and NSF CCF-2106778. The work of A. Wierman is supported in part from the NSF through CNS-2146814, CPS-2136197, CNS-2106403, NGSDI-2105648. The authors thank Shicong Cen, Gen Li, and Yaru Niu for valuable discussions.


%


\bibliography{bibfileRL,bibfileDRO,bibfileGame}
\bibliographystyle{apalike}

\appendix

\input{appendix}

\end{document}

%% file: macro.tex
\DeclareMathOperator{\ind}{\mathds{1}}  

\newcommand{\defn}{\coloneqq}

\newcommand{\no}{0} 
\newcommand{\ror}{\sigma} 
\newcommand{\unb}{\cU} 

\newcommand{\pmin}{P}
\newcommand{\pmhat}{\widehat{P}}
\newcommand{\Pv}{\underline{P}}
\newcommand{\Phatv}{\underline{\widehat{P}}}

\newcommand{\ac}{\mathcal{A}}
\newcommand{\rew}{r}
\newcommand{\DRNVI}{{\sf DR-NVI}\xspace}

\newcommand{\rmg}{RMG\xspace}
\newcommand{\rmgs}{RMGs\xspace}
\newcommand{\allA}{\prod_{i=1}^n A_i}

\newcommand{\ba}{\bm{a}}

\newcommand{\cA}{\mathcal{A}}
\newcommand{\cB}{\mathcal{B}}

\newcommand{\cD}{\mathcal{D}}

\newcommand{\cF}{\mathcal{F}}
\newcommand{\cG}{\mathcal{G}}

\newcommand{\cM}{\mathcal{M}}
\newcommand{\cN}{N}

\newcommand{\cP}{\mathcal{P}}

\newcommand{\cS}{{\mathcal{S}}}

\newcommand{\cU}{\mathcal{U}}

\newcommand{\cW}{\mathcal{W}}
\newcommand{\cX}{\mathcal{X}}
\newcommand{\cY}{\mathcal{Y}}

\newcommand{\mymid}{\,|\,} 

\usepackage{scalerel,stackengine}
\stackMath
\newcommand\reallywidehat[1]{%
\savestack{\tmpbox}{\stretchto{%
  \scaleto{%
    \scalerel*[\widthof{\ensuremath{#1}}]{\kern-.6pt\bigwedge\kern-.6pt}%
    {\rule[-\textheight/2]{1ex}{\textheight}}
  }{\textheight}%
}{0.5ex}}%
\stackon[1pt]{#1}{\tmpbox}%
}
\newcommand\reallywidecheck[1]{%
\savestack{\tmpbox}{\stretchto{%
  \scaleto{
    \scalerel*[\widthof{\ensuremath{#1}}]{\kern-.6pt\bigwedge\kern-.6pt}%
    {\rule[-\textheight/2]{1ex}{\textheight}}
  }{\textheight}%
}{0.5ex}}%
\stackon[1pt]{#1}{\scalebox{-1}{\tmpbox}}%
}

%% file: intro.tex
\section{Introduction}

Many real-world applications of artificial intelligence naturally involve multiple agents in dynamically evolving environments. Examples include ecosystem protection \citep{fang2015security}, board games \citep{silver2017mastering}, strategic management \citep{saloner1991modeling}, and autonomous driving \citep{zhou2020smarts} among many others. One of the most promising algorithmic paradigms for addressing these problems is that of (deep) multi-agent reinforcement learning (MARL) \citep{silver2017mastering,vinyals2019grandmaster,lanctot2019openspiel} through a {\em decision-making} perspective. In full generality, it allows for agents with misaligned and possibly conflicting interests to optimize their own long-term rewards in an unknown dynamic environment, while taking one another into account. As such, MARL can often be modeled as learning in Markov games (MGs) \citep{littman1994markov,shapley1953stochastic}.  Due to the game-theoretic nature of MGs, one often relies on solution concepts which take the form of equilibria --- strategies/policies that are stable under rational deviations for all agents --- like Nash equilibria (NE) \citep{nash1951non,shapley1953stochastic}, correlated equilibria (CE) \citep{aumann1987correlated}, and coarse correlated equilibra (CCE) \citep{ aumann1987correlated,moulin1978strategically}.

\subsection{Environmental uncertainty in MARL}

However, the equilibria of MGs can be very sensitive to environmental perturbations. Environmental uncertainties caused by system noise, model mismatch, and sim-to-real gaps can cause dramatic changes to both the qualitative outcomes of the game as well as agents' payoffs. While this problem is present in single-agent RL, the need for robustness is even more acute in the multi-agent setting where the game-theoretic interactions can cause instabilities~\citep{slumbers2023game}.  Indeed, playing an equilibrium solution learned in the simulated environment might lead to a catastrophic drop in a single agent's payoff or even all agents' payoffs when the deployed environment deviates slightly from what is expected \citep{balaji2019deepracer,zhang2020robust,zeng2022resilience,yeh2021sustainbench}, a point we illustrate in the following example.

\begin{figure*}[t]
\centering
	\includegraphics[width=0.95\linewidth]{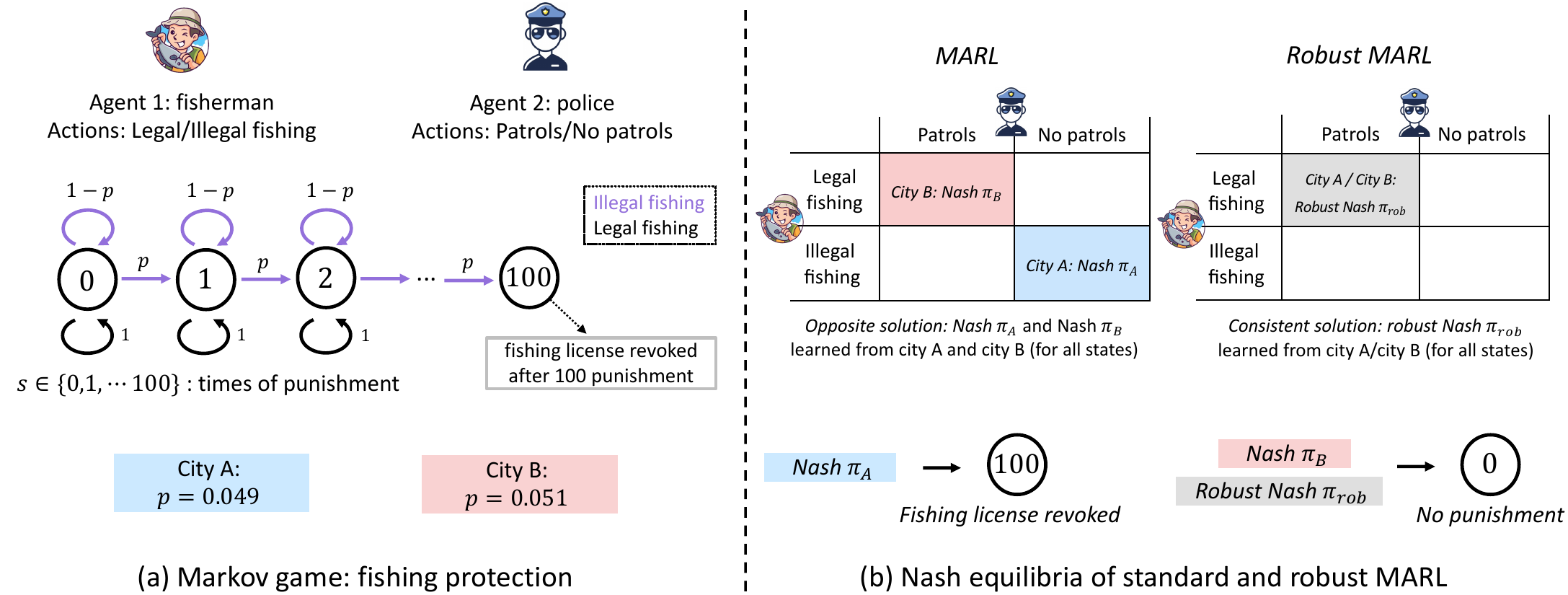} 
	\caption{A two-player general-sum Markov game modeling preventing illegal fishing.  
 (a) shows the state space (circles) and the simplified transitions; the fisherman arrives at distinct states by executing different Nash equilibrium solutions $\pi_A$ (from city A) or $\pi_B$ (from city B). 
	(b) in two slightly different environments (city A versus city B), it shows the solutions $\pi_A, \pi_B$ of the standard game, and the consistent solution {\em robust Nash} $\pi_{rob}$ of a robust variant of the game (detailed in Appendix~\ref{proof:solution-for-example}).}\label{fig:intro}
\end{figure*}

\paragraph{Example: fishing protection.} \textit{
To emphasize the impact of model uncertainty in MARL, in Figure~\ref{fig:intro}, we present a concrete example of a simple two-player game that models the interaction between a fisherman and law enforcement trying to prevent illegal fishing. The state $s\in \{0,1,\cdots,100\}$ represents the number of punishments received by the fisherman, with the license being revoked at $s=100$. The environment is governed by a model parameter $p$. We observe from Figure~\ref{fig:intro}(b) that for slightly perturbed environments, city A ($p=0.049$) and city B ($p=0.051$), the solutions of the MGs are two Nash equilibria with drastically different outcomes: no punishment under policy $\pi_B$ learned from city B (in red) and a revoked license under policy $\pi_A$ learned from city A (in blue). More details are presented in Appendix~\ref{proof:solution-for-example}. 
} The example above illustrates how the standard formulation of a MG can be vulnerable to model uncertainties and result in unstable solutions with divergent outcomes. As such, robustness and stability become a pressing need and key challenge for the deployment of MARL algorithms.

To address this, we consider robust MARL problems as (distributionally) robust Markov games (RMGs) --- a robust counterpart of standard MGs \citep{zhang2020robust,kardecs2011discounted}. The natural solution concepts for RMGs are equilibria not only between agents, but also between multiple natural adversaries that choose the worst-case environments within some prescribed uncertainty set for each agent. By design, they exhibit more robustness and consistency in the face of unmodeled disturbances. To illustrate this, consider the example in Figure~\ref{fig:intro}, where one can observe that the solutions of a RMG ($\pi_{rob}$ in gray) remains consistent and stable across similar environments city A and city B.

Despite some recent efforts \citep{zhang2020robust,kardecs2011discounted,ma2023decentralized,blanchet2023double}, a fundamental understanding of learning in RMGs is lacking. Indeed, while the robust formulation of single-agent RL has been well studied~\citep{iyengar2005robust,nilim2005robust,shi2023curious,xu2023improved}, understanding how to efficiently learn equilibrium policies in robust Markov games remains an open question. We focus on understanding and achieving near-optimal sample efficiency in robust MGs, reflecting the fact that in many large-scale applications, agents must learn from samples from an unknown but potentially extremely large environment \citep{silver2016mastering,vinyals2019grandmaster,achiam2023gpt}.
While some attempts have been made to design sample-efficient algorithms for robust MARL \citep{wang2023finite,blanchet2023double}, the current solutions are still far from optimal. With that in mind, we investigate the following open question:
\begin{quote}
{\em Can we achieve robustness and near-optimal sample efficiency in MARL simultaneously? }
\end{quote}

\subsection{Main contributions} 
To address the open question, this work concentrates on designing algorithms for robust MGs with near-optimal sample complexity guarantees. We consider three solution concepts for RMGs, which are robust variants of standard equilibria --- robust NE, robust CE, and robust CCE. We focus on a class of RMGs, where the  uncertainty sets of the environment are constructed following an {\em agent-wise (s,a)-rectangularity} condition for computational tractability \citep{iyengar2005robust,wiesemann2013robust} (see Section~\ref{sec:framework-robust-marl}). Such a condition allows each agent to independently consider its uncertainty set according to their personal interest. We consider total variation (TV) distance as the distance metric for the uncertainty set, motivated by its practical \citep{pan2023adjustable,lee2021optidice} and theoretical appeal \citep{panaganti2021sample,shi2023curious,blanchet2023double}.

Concretely, our study focuses on finite-horizon RMGs with $n$ agents. We denote the episode length by $H$, the size of the state space by $S$, the size of the $i$-th agent's action space by $A_i$, and use $\ror_i \in (0,1]$ to represent the uncertainty level of the $i$-th agent.  We assume access to a generative model that can draw samples from the nominal environment in a non-adaptive manner. The goal is to find an $\varepsilon$-approximate equilibrium for RMGs --- a joint policy such that each agent's benefit is at most $\varepsilon$ away under rational deviations. The main contributions are summarized as follows.

\begin{itemize}
    \item {\em Near-optimal sample complexity upper bound.}
 We design a model-based algorithm --- distributionally robust Nash value iteration (\DRNVI), which can provably find any solution among $\varepsilon$-approximate robust-$\{$NE, CCE, CE$\}$ with high probability, when the sample size exceeds
\begin{align}
       \widetilde{O} \left(\frac{SH^3\prod_{i=1}^n A_i  }{   \varepsilon^2} \min \Big\{H,  ~\frac{1}{\min_{1\leq i\leq n} \ror_i}\Big\}\right).\label{eq:upper-intro}
   \end{align}
This significantly improves upon prior art \citep{blanchet2023double} $\widetilde{O} \big(S^4 \left(\prod_{i=1}^n A_i\right)^3 H^4 /\varepsilon^2\big)$\footnote{Note that  \citet{blanchet2023double} targets a different (and more challenging) setting with offline data. We translate the results of \citet{blanchet2023double} to the generative setting we consider.} \citep{blanchet2023double} by at least a factor of $\widetilde{O}\big(S^3 \left(\prod_{i=1}^n A_i\right)^2\big)$, and further delineates the impact of the uncertainty levels.
Our results are derived by addressing the intricate statistical dependencies arising from game-theoretical interactions among agents, a challenge not present in robust single-agent RL. Additionally, we employ distributionally robust optimization to address the nonlinear payoffs of agents in RMGs, which lack a closed form.

\item {\em Information-theoretic lower bound.}
To understand the optimality of our algorithm we establish a lower bound for solving RMGs, showing that no algorithm can learn any of $\varepsilon$-approximate robust-$\{$NE, CCE, CE$\}$ with fewer samples than
\begin{align}
       \widetilde{O} \bigg( \frac{SH^3 \max_{1\leq i \leq n} A_i   }{   \varepsilon^2} \min \Big\{H,  ~\frac{1}{\min_{1\leq i\leq n} \ror_i}\Big\} \bigg). \label{eq:lower-intro}
       \end{align}  
To the best of our knowledge, this is the first information-theoretic lower bound for RMGs, regardless of the distance metric in use.
We construct new hard scenarios for tightness, differing from existing ones in both robust single-agent RL and standard MGs, which may be of independent interest.   This in turn establishes that the sample complexity of \DRNVI is optimal for all RMGs with respect to many critical problem-dependent parameters such as $S, H, \{\ror_i\}_{1\leq i\leq n}$, making \DRNVI the first near-optimal finite-sample guarantee for robust MGs, regardless of the divergence metric in use.

\end{itemize}

 \paragraph{Notation.}
 Throughout this paper, we introduce the notation $[T]\coloneqq \{1,\cdots,T\}$ for any positive integer $T>0$. We denote by $\Delta(\cS)$ the probability simplex over a set $\cS$ and $x = \big[x(s,a)\big]_{(s,a)\in\cS\times\cA}\in \mathbb{R}^{SA}$ (resp.~$x = \big[x(s)\big]_{s\in\cS}\in \mathbb{R}^{S}$) as any vector that constitutes certain values for each state-action pair (resp.~state).

%% file: background.tex
\section{Background: Standard Markov Games}\label{sec:background}

We begin by covering the foundational aspects of multi-agent general-sum standard Markov games in a finite-horizon setting. 

\paragraph{Standard Markov games.} A finite-horizon {\em multi-agent general-sum Markov game} can be represented as $\mathcal{MG}=\big\{ \cS,  \{\cA_i\}_{1 \le i \le n}, P,  \rew, H \big\}$. This game involves $n$ agents who optimizes their own benefits in a shared environment, consisting of the following key components. 
\begin{itemize}
	\item State space $\cS= \{1,\cdots,S\}$ of the shared environment with $S$ different states. 
	\item Joint action space $\ac$: for each $1\leq i\leq n$, we represent $\mathcal{A}_i=\{1,\cdots, A_i\}$ as the action space of the $i$-th agent that contains $A_i$ different actions. In addition, we denote the joint action space for all agents (or a subset of agents) as	$\ac \coloneqq \mathcal{A}_1 \times \cdots \times \mathcal{A}_m$ (or 
			$\mathcal{A}_{-i} \coloneqq \prod_{j:j\neq i}\mathcal{A}_j$ for all $1\leq i\leq n$). 	%
		%
	For convenience, we denote the boldface letter $\bm{a} \in \ac$ (resp.~$\bm{a}_{-i} \in \cA_{-i}$) as  a joint action profile for all agents (resp.~all agents excluding the $i$-th agent).  
		

	\item  Probability transition kernel $P = \{P_h\}_{1\leq h\leq H}$ with $P_h: \cS \times \ac \mapsto \Delta(\cS)$. Specifically, $P_h(s'\mymid s,\bm{a})$ represents the probability of $\mathcal{MG}$ transitioning from current state $s \in \cS$ to the next state $s'\in \cS$ at time step $h$, given the agents choose the joint action profile $\bm{a}\in \ac$. 

	\item Reward function $\rew=\{r_{i,h}\}_{1\leq i\leq n, 1\leq h\leq H}$ with $r_{i,h}: \cS\times \ac \mapsto [0,1]$. Specifically, for any $(i,h,s,\bm{a})\in [n] \times [H] \times \cS\times \cA$, let $r_{i,h }(s,\bm{a})$ be the immediate (deterministic) reward received by the $i$-th agent in  state $s$ when the joint action profile is $\bm{a}$, which is normalized to $[0,1]$ without loss of generality. 
\item $H$ is the horizon length of the standard MG.

\end{itemize}


\paragraph{Markov policies and value functions.} Throughout the paper, we focus on the class of Markov policies, namely, the action selection rule is solely determined by the current state $s$, 
independent from previous trajectories (including visited states, executed actions, and received rewards) of all agents. Specifically, for any $1\leq i\leq n$, the $i$-th agent executes actions according to a policy $ \pi_i = \{\pi_{i,h} :\cS   \mapsto \Delta(\ac_i)  \}_{1\leq h\leq H}$, with $\pi_{i,h}(a \mymid s)$ the probability of selecting action $a$ in state $s$ at time step $h$.
The joint Markov policy of all agents can be defined as $\pi= (\pi_1,\ldots, \pi_n): \cS \times [H] \mapsto \Delta(\ac)$, namely, the joint action profile $\bm{a}$ of all agents is chosen according to the distribution specified by $\pi_h(\cdot \mymid s) = (\pi_{1,h}, \pi_{2,h}\ldots, \pi_{n,h})(\cdot\mymid s) \in \Delta(\ac)$ conditioned on  state
$s$ at time step $h$. 

With the above notation in mind, for any given joint policy $\pi$ and transition kernel $P$ of the $\mathcal{MG}$,
we characterize the long-term cumulative reward by defining the value function $V_{i,h}^{\pi, P}: \cS \mapsto \mathbb{R}$ (resp.~Q-function $Q_{i,h}^{\pi, P}: \cS \times \ac \mapsto \mathbb{R}$) of the $i$-th agent as follows: for all $(h,s,a)\in [H]\times \cS \times \cA$,
\begin{align}
	 V_{i,h}^{\pi, P}(s) &\coloneqq\mathbb{E}_{\pi,
P}\left[\sum_{t=h}^{H} \gamma^t r_{i,t}\big(s_{t}, \bm{a}_{t}\big)\mid s_{h}=s\right], \nonumber \\
  Q_{i,h}^{\pi, P}(s, \ba)&\coloneqq\mathbb{E}_{\pi,P}\left[\sum_{t=h}^{H} \gamma^t r_{i,t}\big(s_{t}, \bm{a}_{t}\big)\mid s_{h}=s, \ba_h = \ba\right], \label{eq:value-function-defn}
\end{align}
where the expectation is taken over the Markovian trajectory $\{(s_t,\bm{a}_t)\}_{h\leq t\leq H}$ by executing the joint policy $\pi$ under the transition kernel $P$, i.e., $\bm{a}_t\sim \pi_t(\cdot \mymid s_t)$ and $s_{t+1}\sim P(\cdot \mymid s_t, \bm{a}_t)$. 

\paragraph{Best-response policy.} For any given joint policy $\pi$, we employ $\pi_{-i}$ to represent the policies of all agents excluding the $i$-th agent. We define the maximum value function of the $i$-th agent at time step $h$ against the joint policy $\pi_{-i}$ of the other agents as
\begin{equation}
	\label{eq:defn-optimal-V}
	 V_{i,h}^{\star,\pi_{-i}, P}(s)\coloneqq\max_{\pi'_i: \cS \times [H]\rightarrow \Delta(\mathcal{A}_i)} V_{i,h}^{\pi'_i \times \pi_{-i}, P}(s),
\end{equation}
where $\pi'_i \times \pi_{-i}$ represents the joint policy of all agents when the $i$-th agent executes policy $\pi'_i$. It is well-known \citep{filar2012competitive} that there exists at least one Markovian policy,  the {\em best-response policy}, that achieves $V_{i,h}^{\star,\pi_{-i}, P}(s)$ for all $s\in \cS$ and all $h\in[H]$ simultaneously. We denote the best-response policy using
 $\pi_i^{\star, P}\big( \pi_{-i} \big): \cS \times [H] \mapsto \Delta(\mathcal{A}_i)$. 

\paragraph{Solution concepts: equilibria.} In MGs, strategic agents are modeled in a possibly competitive framework and focus on finding some sort of equilibrium strategies. Here, we consider three common types of equilibria --- NE, CE, and CCE for MGs. 
\begin{itemize}
	\item {\em Nash equilibrium (NE).} A product policy $\pi=\pi_1\times \cdots \times \pi_n \in \Delta(\cA_1) \times \Delta(\cA_2) \times \cdots \times \Delta(\cA_n)$ is 
said to be a {\em (mixed-strategy Markov) NE} if 
\begin{equation}
	\text{for all }(s,i)\in\cS\times [n]: ~ V_{i,1}^{\pi, P}(s)=V_{i,1}^{\star,\pi_{-i}, P}(s).
	\label{eq:defn-Nash-E}
\end{equation}
Namely, as long as all players act independently, no player can benefit by unilaterally diverging from its present policy, given the current policies of the opponents.

\item {\em Coarse correlated equilibrium (CCE).}
A joint policy $\pi \in \Delta(\ac)$ is said to be a CCE \citep{moulin1978strategically,aumann1987correlated} if it holds that
\begin{equation}
	\text{for all }(s,i)\in\cS\times [n]: ~ V_{i,1}^{\pi, P}(s) \geq V_{i,1}^{\star,\pi_{-i}, P }(s).
	\label{eq:defn-CCE}
\end{equation}
As a relaxation of NE, CCE also guarantees that no player has incentive to unilaterally deviated from the current policy. The key difference from the NE definition is that it permits policies to be interrelated among players. 

\item {\em Correlated equilibrium (CE).} 
Before proceeding, for each $1 \leq i \leq n$, we define a set of function $f_i\defn \{f_{i,h,s}\}_{h\in[H],s\in\cS}$ with $f_{i,h,s}: \cA_i \mapsto \cA_i$, and denoting $\cF_i$ as the set of possible $f_i$. Armed with this, we can combine such $f_i$ with any joint policy $\pi$ to reach a new policy $f_i \diamond \pi$, where $f_i \diamond \pi$ will choose $(a_1, \ldots, a_{i-1}, f_i(a_i), a_{i+1}, \ldots, a_n)$ when policy $\pi$ selects $(a_1, \ldots, a_{i-1}, a_i, a_{i+1}, \ldots, a_n)$.
With these in place, a joint policy $\pi \in \Delta(\ac)$ is said to be a CE \citep{moulin1978strategically,aumann1987correlated}   if it holds that
\begin{equation}
	\text{for all }(s,i)\in\cS\times [n]: ~V_{i,1}^{\pi, P}(s) \geq \max_{f_i \in \cF_i} V_{i,1}^{f_i \diamond \pi, P }(s).
	\label{eq:defn-CE}
\end{equation}
CE is a also a relaxation of NE, which does not require the joint policy $\pi$ to be a product policy.
\end{itemize}

%% file: general-formulation.tex
\section{Distributionally Robust Markov Games}\label{sec:framework-robust-marl}

We consider a robust variant of standard MGs  incorporating environmental uncertainties --- termed {\em distributionally robust Markov games} (\rmg{s}). \rmgs represent a richer class than standard MGs, allowing for different prescribed environmental uncertainty sets as long as they meet a  {\em rectangularity} condition, detailed below.

\subsection{Distributionally robust Markov games}

A {\em distributionally robust multi-agent general-sum Markov game} (\rmg) in the finite-horizon setting  is defined by 
$$\mathcal{MG}_{\mathsf{rob}} = \big\{ \cS, \{\cA_i\}_{1 \le i \le n},\{\cU^{\ror_i}_\rho(P^\no)\}_{1 \le i \le n}, \rew, H \big\}, $$ 
where $\cS, \{\cA_i\}, \rew$, and $H$ are identical to those of standard MGs (see Section~\ref{sec:background}). A notable deviation from standard MGs is that: for $1\leq i\leq n$, instead of assuming a fixed transition kernel, each $i$-th agent anticipates that the transition kernel is allowed to be chosen arbitrarily from a prescribed uncertainty set $\cU^{\ror_i}_\rho(P^\no)$. Here, the uncertainty set $\cU^{\ror_i}_\rho(P^\no)$ is constructed centered on a {\em nominal} kernel $P^\no: \cS\times \ac \mapsto \Delta(\cS)$,  with its size and shape defined by a certain distance metric $\rho$ and a radius parameter $\ror_i>0$. Note that, for generality, to accommodate individual robustness preferences, each agent is permitted to tailor its own uncertainty set $\cU^{\ror_i}_\rho(P^\no)$ by choosing different size $\ror_i$ and even the shape determined by different divergence function $\rho$. Here, we consider the same divergence function for all agents for simplicity. And we focus on the discussion of the transition kernel's uncertainty in this work, it's worth noting that similar uncertainty can also be considered for each agent's reward function.

\paragraph{Uncertainty set with {\em agent-wise $(s,a)$-rectangularity}.}
In the following, we specify the construction of the transition kernel uncertainty sets $\cU_\rho(P^\no) = \{\cU_\rho^{\ror_i}(P^\no)\}_{1 \le i \le n}$ for RMGs.  Drawing inspiration from the {\em rectangularity} condition advocated in robust single-agent RL \citep{iyengar2005robust,wiesemann2013robust,zhou2021finite,shi2023curious}, we consider a multi-agent variant of rectangularity in RMGs --- {\em agent-wise $(s,a)$-rectangularity}.
This condition enables the robust counterpart of Bellman recursions and computational tractability of the problems.
It allows for each agent to independently choose its own uncertainty set that can be decomposed into a product of subsets over each state-action pair.

 In particular, we assume all agents use the same distance metric $\rho$ for their uncertainty sets.\footnote{Generally, each agent can decide their own (possibly different) distance metric for the uncertainty set. We consider the same $\rho$ for simplicity.} Each $i$-th agent can choose their own uncertainty level $\ror_i >0$ independently. With $\rho$ and  $\{\ror_i\}_{1\leq i \leq n}$ in hand, the uncertainty set $\cU_\rho(P^\no) $ of all agents obeying {\em agent-wise $(s,a)$-rectangularity} is mathematically specified as:
\begin{align}
	 \forall i\in[n]:~\cU^{\ror_i}_\rho(P^\no)  & \defn \otimes \; \cU^{\ror_i}(P^\no_{h,s,\ba}) \quad \text{with}\label{eq:sa-rec-defn} \\
	\cU^{\ror_i}_\rho(P^\no_{h,s,\ba})& \defn \left\{ P_{h,s,\ba} \in \Delta (\cS): \rho \left(P_{h,s,\ba}, P^0_{h,s,\ba}\right) \leq \ror_i \right\},  \notag
\end{align}
where $\otimes$ represents the Cartesian product and we denote a vector of the transition kernel $P$ or $P^{\no}$ at any state-action pair $(s,\ba) \in \cS \times \cA$ respectively as
\begin{align}\label{eq:defn-P-sa}
	&P_{h,s,\ba} \defn P_h(\cdot \mymid s, \ba) \in \mathbb{R}^{1\times S}, \qquad P_{h,s, \ba }^\no \defn P^\no_h(\cdot \mymid s,\ba) \in \mathbb{R}^{1\times S}.
\end{align}
Here, the `distance' function $\rho$ for each agent's uncertainty set can be chosen from many candidate functions that measure the difference between two probability vectors,  such as $f$-divergence (including total variation (TV), chi-square, and Kullback-Leibler (KL) divergence) \citep{yang2021towards}, $\ell_q$ norm \citep{clavier2023towards}, and Wasserstein distance \citep{xu2023improved}. In this work, we focus on the uncertainty sets that are constructed using TV distance: 
\begin{align}\label{eq:general-infinite-P}
	\rho_{\mathsf{TV}}\left(P_{h,s,\ba}, P^\no_{h,s,\ba}\right) \defn  \frac{1}{2}\left\|P_{h,s, \ba}- P^0_{h,s, \ba}\right\|_1.
\end{align}

\paragraph{Robust value functions.}
For a \rmg, each agent aims to maximize its own worst-case performance over all possible transition kernels in its own (possibly different) prescribed uncertainty set $\cU^{\ror_i}_\rho(P^\no)$. For any joint policy $\pi \in \Delta(\cA)$, the worst-case performance of the $i$-th agent at time step $h$ can be measured by the {\em robust value function} $V_{i,h}^{\pi,\ror_i}$ and the {\em robust Q-function} $Q_{i,h}^{\pi,\ror_i}$, defined as
\begin{align}
	V_{i,h}^{\pi,\ror_i}(s)& \coloneqq \inf_{P\in \unb^{\ror_i}_\rho(P^{\no})} V_{i,h}^{\pi,P} (s)  \qquad\mbox{and}\qquad  Q_{i,h}^{\pi,\ror_i}(s, \ba) \coloneqq  \inf_{P\in \unb^{\ror_i}_\rho(P^{\no})} Q_{i,h}^{\pi,P} (s, \ba)
\label{eq:value-function-defn-robust}
\end{align}
for all $ (i,h,s,\ba)\in [n] \times [H] \times \cS \times \cA $.
Similar to standard MGs, given a fixed joint policy $\pi_{-i}$ for all agents but the $i$-th agent, by optimizing over  $\pi'_i: \cS \times [H] \rightarrow \Delta(\mathcal{A}_i)$ that is executed independently from $\pi_{-i}$, we can further define the maximum of the robust value function for each agent as follows: for all $(i,h,s) \in [n] \times [H] \times  \cS:$
\begin{align}
	\label{eq:defn-optimal-V}
	   V_{i,h}^{\star,\pi_{-i}, \ror_i}(s) & \coloneqq \max_{\pi'_i: \cS \times [H] \mapsto \Delta(\mathcal{A}_i)} V_{i,h}^{\pi'_i \times \pi_{-i}, \ror_i }(s) \notag  \\
	 & = \max_{\pi'_i: \cS \times [H] \mapsto \Delta(\mathcal{A}_i)} \inf_{P\in \unb^{\ror_i}_\rho(P^{\no})} V_{i,h}^{\pi_i' \times \pi_{-i},P} (s). 
\end{align}
Similar to standard MGs, it can be easily verified that there exists at least one policy \citep[Section A.2]{blanchet2024double}, denoted by $\pi_i^{\star, \ror_i}\big( \pi_{-i} \big): \cS \times [H] \rightarrow \Delta(\mathcal{A}_i)$ and referred to as the {\em robust best-response policy} for the $i$-th agent, 
that can simultaneously attain $V_{i,h}^{\star,\pi_{-i}, \ror_i}(s)$ for all $s\in \cS$ and $h\in[H]$.

\paragraph{Robust Bellman equations.} 
Analogous to standard MGs, \rmgs feature a robust counterpart of the Bellman equation ---  {\em robust Bellman  equation}. 
In particular, the robust value functions $\{V_{i,h}^{\pi, \ror_i}\}$ of \rmgs associated with any joint policy $\pi$  
obey: for all $(i,h,s)\in  [n] \times [H] \times \cS$,
\begin{align}
V_{i,h}^{\pi, \ror_i}(s)  &= \mathbb{E}_{\ba \sim \pi_h(s)} \left[ r_{i,h}(s, \ba) +   \inf_{P\in \unb^{\ror_i}_\rho(P^{\no}_{h,s,\ba})} P V_{i,h+1}^{\pi, \ror_i} \right].\label{eq:bellman-consistency-sa}
\end{align}

We emphasize that the above robust Bellman equation is fundamentally linked to the {\em agent-wise $(s,a)$-rectangularity}  condition (cf.~\eqref{eq:sa-rec-defn}) imposed on the designed uncertainty set. Specifically, this condition decouples the dependency of uncertainty subsets across different agents, each state-action pair, and different time steps, leading to  the Bellman recursive equation.

\subsection{Solution concepts for robust Markov games}

For \rmgs, the games are no longer  $n$-agent games, but become $2n$-agent games between agents and $n$ natural adversaries to choose the worst-case transitions.  
Given the possibly conflicting objectives, finding an equilibrium becomes a core goal for \rmgs.
Below, we introduce three robust variants of widely considered standard solution concepts --- robust NE, robust CE, and robust CCE for any \rmg.
\begin{itemize}
	\item {\em Robust NE.} A product policy $\pi = \pi_1 \times \pi_2 \times \cdots \times \pi_{n}$ is 
said to be a {\em robust NE}  if (cf.~\eqref{eq:defn-Nash-E})
\begin{equation}
	\forall (i,s)\in[n] \times \cS:\quad V_{i,1}^{\pi, \ror_i}(s)=V_{i,1}^{\star,\pi_{-i}, \ror_i}(s).
	\label{eq:defn-robust-Nash-E}
\end{equation}
Robust NE indicates that given the current strategy of the opponents $\pi_{-i}$, when each agent considers the worst-case performance over its own uncertainty set $\unb^{\ror_i}_\rho(P^{\no})$, no player can benefit by unilaterally diverging from its present strategy.

\item {\em Robust CCE.}
  A (possibly correlated) joint policy $\pi \in \cS\times [H] \mapsto \Delta(\ac)$ is said to be a {\em robust CCE}  if it holds that (cf.~\eqref{eq:defn-CCE})
\begin{equation}
	\forall (i,s)\in[n] \times \cS:~ V_{i,1}^{\pi,\ror_i}(s) \geq V_{i,1}^{\star,\pi_{-i}, \ror_i }(s).
	\label{eq:defn-robust-CCE}
\end{equation}

As a relaxation of robust NE, robust CCE also guarantees that no player has incentive to unilaterally deviate from the current policy, where the policies are not necessarily independent among players. 
\item {\em Robust CE.} 
A joint policy $\pi \in \Delta(\ac)$ is said to be a robust CE  if it holds that (cf.~\eqref{eq:defn-CE})
\begin{equation}
	\forall (s,i)\in\cS\times [n]:~ V_{i,1}^{\pi, \ror_i}(s) \geq \max_{f_i \in \cF_i} V_{i,1}^{f_i \diamond \pi, \ror_i }(s).
	\label{eq:defn-robust-CE}
\end{equation}
\end{itemize}

It is known that computing exact robust equilibria is challenging and may not be necessary in practice.  As a result, people usually search for approximate equilibria. Toward this, as a slightly relaxation from \eqref{eq:defn-robust-Nash-E}, a product policy $\pi \in \Delta(\cA_1)  \times \cdots \times \Delta(\cA_n)$ is said to be an {\em $\varepsilon$-robust NE}  if 
\begin{equation}
	\mathsf{gap}_{\mathsf{NE}}(\pi) \defn \max_{s\in \cS, 1 \leq i \leq n} \left\{ V_{i,1}^{\star,\pi_{-i}, \ror_i }(s) - V_{i,1}^{\pi, \ror_i }(s)  \right\} \leq \varepsilon.
	\label{eq:defn-Nash-E-epsilon}
\end{equation}

Similarly, relaxing \eqref{eq:defn-robust-CCE} or \eqref{eq:defn-robust-CE}, a (possibly correlated) joint policy $\pi \in \Delta(\cA) $ is said to be an {\em $\varepsilon$-robust CCE}  if 
\begin{equation}
	\mathsf{gap}_{\mathsf{CCE}}(\pi) \defn \max_{s\in \cS, 1 \leq i \leq n} \left\{ V_{i,1}^{\star,\pi_{-i}, \ror_i }(s) - V_{i,1}^{\pi, \ror_i }(s)  \right\} \leq \varepsilon,
	\label{eq:defn-CCE-epsilon}
\end{equation}
or an {\em $\varepsilon$-robust CE} if 
\begin{align}
	\mathsf{gap}_{\mathsf{CE}}(\pi) & \defn \max_{s\in \cS, 1 \leq i \leq n} \left\{  \max_{f_i \in \cF_i} V_{i,1}^{f_i \diamond \pi, \ror_i }(s) - V_{i,1}^{\pi, \ror_i}(s)  \right\}  \leq \varepsilon.
	\label{eq:defn-Nash-CCE-epsilon}
\end{align}

The existence of robust NE has been verified \citep{blanchet2023double} under general divergence functions for the uncertainty set. Indeed, the robust equilibria defined here can be reduced to the standard equilibria associated with the robust variant of standard payoffs (robust Q-functions), which have been verified obeying $\{\text{NE}\} \subseteq \{\text{CE}\} \subseteq\{\text{CCE}\}$ \citep{roughgarden2010algorithmic}. Therefore, the existence of robust NE directly indicates the existence of robust CE and robust CCE.

\subsection{Non-adaptive sampling from a generative model}

Given the formulation of distributionally robust Markov games, a question of prime interest is how to  learn the robust equilibria without knowing the model exactly in a sample-efficient manner. 
 
\paragraph{Sampling mechanism: a generative model.}
As a widely used sampling mechanism in standard MARL \citep{zhang2020model-based,li2022minimax}, in this paper, we assume  access to a generative model (simulator) \citep{kearns1999finite} and collect samples in a non-adaptive manner. Specifically, for each tuple $(s,\ba,h)\in \cS\times \cA\times [H]$, we collect $N$ independent samples generated based on the true {\em nominal} transition kernel $P^{\no}$:
\begin{align}
	 s_{i,h,s, \ba} \overset{i.i.d}{\sim} P^\no_h(\cdot \mymid s, \ba), \qquad i = 1, 2,\ldots, N.
\end{align}
The total number of samples is thus $N_{\mathsf{all}} = NS\prod_{i=1}^n A_i$.

Armed with the collected dataset from the nominal environment, the goal is to learn a solution among $\varepsilon$-robust-$\{$NE, CCE, CE$\}$ for the game $\mathcal{MG}_{\mathsf{rob}}$ --- w.r.t.~some prescribed uncertainty set $\cU(P^\no)$ around the nominal kernel --- using as few samples as possible.

%% file: results.tex
\section{Algorithm and Theory}
In this and the following sections, we focus on the class of robust MGs with uncertainty set measured by TV distance, namely, the uncertainty set $\cU^{\ror_i}_\rho(\cdot) =\cU^{\ror_i}_{\rho_{\mathsf{TV}}}(\cdot)$ w.r.t the TV distance $\rho = \rho_{\mathsf{TV}}$ defined in \eqref{eq:general-infinite-P}. For convenience, we abbreviate $\cU^{\ror_i}(\cdot) \defn \cU^{\ror_i}_{\rho_{\mathsf{TV}}}(\cdot)$.

\subsection{Distributionally robust Nash value iteration}
We develop a model-based approach tailored to solve robust Markov games, which involves two separate steps. First, we construct an empirical nominal transition kernel $\widehat{P}^\no$ using the collected samples from the generative model. Then armed with $\widehat{P}^\no$, we propose to apply distributionally robust Nash value iteration (\DRNVI) to compute a robust equilibrium solution for all agents.

\paragraph{Nominal model estimation.}
Based on the empirical frequency of state transitions, we estimate the empirical nominal transition kernel $\widehat{P}^\no = \{\widehat{P}_h^\no\}_{h\in[H]}$, where the entries of $\widehat{P}_h^\no \in \mathbb{R}^{S \prod_{i=1}^n A_i \times S}$ at each time step $h$ is constructed as follows: for all $ (h,s, \ba)\in \cS\times \cA$,
\begin{align}
	 \widehat{P}^0_h(s'\mymid s, \ba) \defn  \frac{1}{N} \sum\limits_{i=1}^N \mathds{1} \big\{ s_{i,h,s, \ba} = s' \big\}.
	\label{eq:empirical-P-infinite}
\end{align}
\paragraph{Distributionally robust Nash value iteration (\DRNVI).}  With the empirical nominal kernel $\widehat{P}^\no$ in hand, to compute a robust equilibrium solution, we propose \DRNVI by adapting a model-based algorithm for standard Markov games --- Nash value iteration \citep{liu2021sharp}, summarized in Algorithm~\ref{alg:nash-dro-vi-finite}.

 The process starts from the last time step $h=H$ and proceeds with $h=H-1, H-2,\cdots, 1$. At each time step $h\in[H]$, the robust Q-function can be estimated as  $\widehat{Q}_{i,h}$ (see line~\ref{eq:robust-q-estimate}) as: for all $ (i,h,s,\ba)\in  [n] \times [H] \times \cS\times \cA$,
\begin{align}
 \widehat{Q}_{i,h}(s, \ba) = r_{i,h}(s, \ba) + \inf_{ \cP \in \unb^{\ror_i}(\widehat{P}^{\no}_{h,s,\ba})} \cP \widehat{V}_{i,h+1}. \label{eq:nvi-iteration}
\end{align}
Directly solving \eqref{eq:nvi-iteration} presents significant computational challenges due to the need to optimize over an $S$-dimensional probability simplex, a task whose complexity increases exponentially with the state space size $S$. Fortunately, leveraging strong duality  enables us to solve \eqref{eq:nvi-iteration} equivalently via its dual problem \citep{iyengar2005robust}:  
\begin{align}
&\widehat{Q}_{i,h}(s, \ba) = r_{i,h}(s, \ba) +   \max_{\alpha\in [\min_s \widehat{V}_{i,h+1}(s), \max_s \widehat{V}_{i,h+1}(s)]}  \Big\{ \widehat{P}^{\no}_{h,s,\ba} \left[\widehat{V}_{i,h+1}\right]_{\alpha} - \ror_i \left(\alpha - \min_{s'}\left[\widehat{V}_{i,h+1}\right]_{\alpha}(s') \right) \Big\}  ,  \label{eq:nvi-iteration-dual} 
\end{align} 
where $[V]_{\alpha}$ denotes the clipped version of any  vector $V\in\mathbb{R}^S$ determined by some level $\alpha\geq 0$, namely,
\begin{align}
	[V]_{\alpha}(s) \defn \begin{cases} \alpha, & \text{if } V(s) > \alpha, \\
V(s), & \text{otherwise.}
\end{cases} \label{eq:V-alpha-defn}
\end{align}
With robust Q-function estimates $\{\widehat{Q}_{i,h}\}_{i\in[n]}$ available for all agents at time step $h$, the sub-routine in line~\ref{eq:nash-subroutine} $\mathsf{Equilibrium} \in \mathsf{Compute-}\{ \mathsf{Nash}, \mathsf{CE}, \mathsf{CCE}\}$ represents the algorithm for computing the corresponding  robust-$\{$NE, CE, CCE$\}$, respectively. Note that for the studied \rmgs, a robust-NE/CE/CCE is equivalent to a corresponding NE/CE/CCE associated with the payoff matrices $\{\widehat{Q}_{i,h}\}_{i\in[n]}$. On the computing and learning front of the sub-routine $\mathsf{Equilibrium}(\cdot)$, for a general standard MG, the NE has been proved PPAD-hard to compute \citep{daskalakis2013complexity}, even for two-player matrix games (except for two-player zero-sum games). Notably, even when the non-robust standard MG associated with the nominal transition kernel is a two-player zero-sum game, the corresponding robust MG is  generally not because agents may select different worst-case transition kernels. Conversely, computing CE/CCE is computationally tractable  within polynomial time through linear programming \citep{liu2021sharp}. 

\begin{algorithm}[t]
\begin{algorithmic}[1]
	\STATE \textbf{input:} empirical nominal transition kernel $\widehat{P}^{\no}$; reward function $r$; uncertainty levels $\{\ror_i\}_{i\in[n]}$. \\ 
	\STATE  \textbf{initialization:} $\widehat{Q}_{i,h}(s,a)= 0$, $\widehat{V}_{i,h}(s)=0$ for all $(s,\ba,h) \in \cS\times \cA  \times [H+1]$. \\
   \FOR{$h = H, H-1, \cdots,1$}
		
		\FOR{$i= 1,2,\cdots,n$ and $s\in \cS, \ba\in \cA $}
			
		\STATE	Set $\widehat{Q}_{i,h}(s, \ba)$ according to \eqref{eq:nvi-iteration}. \label{eq:robust-q-estimate}

			\ENDFOR

		\FOR{$s\in \cS$}
		\STATE Get $\pi_{h}(s)=\{\pi_{i,h}(s)\}_{1 \leq i \leq n} \leftarrow  \mathsf{Equilibrium} \left( \{\widehat{Q}_{i,h}(s,\cdot) \}_{1 \leq i \leq n}\right)$. \label{eq:nash-subroutine}
		 
		\STATE	Set $\widehat{V}_{i,h}(s) = \mathbb{E}_{\ba \sim \pi_h}[\widehat{Q}_{i,h}(s,\ba)]$.
		\ENDFOR
\ENDFOR
	
	\STATE \textbf{output:} $\{\widehat{Q}_{i,h}\}$, $\{\widehat{V}_{i,h}\}$, and $\widehat{\pi} = \{\pi_h\}_{1\leq h \leq H}$.
 
\end{algorithmic}
\caption{Distributionally robust equilibrium value iteration (\DRNVI).}
 \label{alg:nash-dro-vi-finite}
\end{algorithm}

\subsection{Sample complexity: upper and lower bounds}\label{sec:theoretical-results}

We now present our main theoretical results regarding the sample complexity of learning robust equilibria of robust Markov games, including an upper bound of \DRNVI (Algorithm~\ref{alg:nash-dro-vi-finite}) and an information-theoretic lower bound.
First, we introduce the finite-sample guarantee for \DRNVI, which is proven in Appendix~\ref{proof:upper-bound}.

\begin{theorem}[Upper bound for \DRNVI]\label{thm:robust-mg-upper-bound} 
	Recall the TV uncertainty set $\cU^{\ror_i}(\cdot) = \cU^{\ror_i}_{\rho_{\mathsf{TV}}}(\cdot)$ defined in \eqref{eq:defn-P-sa}. Consider any $\delta \in (0,1)$ and any \rmg $\mathcal{MG}_{\mathsf{rob}} = \big\{ \cS, \{\cA_i\}_{1 \le i \le n},\{\cU^{\ror_i}(P^\no)\}_{1 \le i \le n}, \rew,  H \big\}$ with $\ror_i \in (0,1]$ for all $i\in[n]$. For any $\varepsilon \leq \sqrt{\min \big\{H,  \frac{1}{\min_{1\leq i\leq n} \ror_i} \big\}}$, Algorithm~\ref{alg:nash-dro-vi-finite} can output any robust equilibrium among $\varepsilon$-robust $\{$NE, CCE, CE$\}$ by executing different subroutine $\mathsf{Equilibrium} \in \mathsf{Compute-}\{ \mathsf{Nash}, \mathsf{CE}, \mathsf{CCE}\}$ in line~\ref{eq:nash-subroutine}. Namely, for some constant $C_1$, we can achieve any of the following results
     \begin{align*}
 \mathsf{gap}_{\mathsf{NE}}(\widehat{\pi}) &\leq \varepsilon, \\
  \mathsf{gap}_{\mathsf{CCE}}(\widehat{\pi}) &\leq \varepsilon, \\
	\mathsf{gap}_{\mathsf{CE}}(\widehat{\pi})& \leq \varepsilon
	\end{align*}
with probability at least $1-\delta$, as long as the total number of samples obeys 
\begin{align*}
	N_{\mathsf{all}} \geq  \frac{ C_1 S H^3\prod_{1\leq i\leq n} A_i   }{  \varepsilon^2}\min \Big\{H,  \frac{1}{\min_{1\leq i\leq n} \ror_i} \Big\} \log\left(\frac{18S \allA nHN}{\delta} \right).
\end{align*}

\end{theorem}
Before delving into the implications of  Theorem~\ref{thm:robust-mg-upper-bound}, we provide a lower bound for solving robust Markov games. The proof is provided in Appendix~\ref{proof:thm:robust-mg-lower-bound}.
\begin{theorem}[Lower bound for solving robust MGs]\label{thm:robust-mg-lower-bound}
Consider any tuple $\big\{ S, \{A_i\}_{1 \le i \le n},\{\ror_i\}_{1 \le i \le n}, H \big\}$ obeying $\ror_i \in (0, 1 - c_0]$ with $0 <c_0 \leq \frac{1}{4}$ being any small enough positive constant, and $H > 16 \log2$. Let
\begin{align}
 \varepsilon \leq \begin{cases} \frac{c_2}{H}, &\text{if } \ror_1\leq \frac{c_2}{2H}, \\
    1 & \text{otherwise}
    \end{cases}
\end{align}
for any $c_2 \leq \frac{1}{4}$.
 We can construct a set of \rmgs  --- denoted as $\mathcal{M} = \{\mathcal{MG}_i\}_{i\in [I]}$, such that for any dataset with in total $N_{\mathsf{all}}$ independent samples over all state-action pairs generated from the nominal environment (for any game $\mathcal{MG}_i \in \cM$): one has
	\begin{align}
\inf_{\widehat{\pi}}\max_{\mathcal{MG}_i\in \cM} \left\{ \mathbb{P}_{\mathcal{MG}_i}\big(  \mathsf{gap}_{\mathsf{NE}}(\widehat{\pi}) >\varepsilon\big) \right\} &\geq\frac{1}{8}, \notag \\
	\inf_{\widehat{\pi}}\max_{\mathcal{MG}_i\in \cM} \left\{ \mathbb{P}_{\mathcal{MG}_i}\big(  \mathsf{gap}_{\mathsf{CCE}}(\widehat{\pi}) >\varepsilon\big) \right\} &\geq\frac{1}{8},  \\
	\inf_{\widehat{\pi}}\max_{\mathcal{MG}_i\in \cM} \left\{ \mathbb{P}_{\mathcal{MG}_i}\big(  \mathsf{gap}_{\mathsf{CE}}(\widehat{\pi}) >\varepsilon\big) \right\} &\geq\frac{1}{8}, \notag
\end{align}
provided that
\begin{align}\label{eq:lower-bound}
N_{\mathsf{all}} \leq \frac{C_2  S H^3  \max_{1\leq i\leq n} A_i    }{ \varepsilon^2} \min \Big\{H,  \frac{1}{\min_{1 \leq i \leq n} \ror_i  } \Big\}.
\end{align}
Here, $C_2$ is some small enough constant, the infimum is taken over all estimators $\widehat{\pi}$, and $\mathbb{P}_{\mathcal{MG}_i}$ denotes the probability
when the game is $\mathcal{MG}_i$ for all $\mathcal{MG}_i \in \cM$.
\end{theorem}

We now highlight several key implications and comparisons that follow from the above results.

\paragraph{Near-optimal sample complexity for RMGs.}
Theorem~\ref{thm:robust-mg-upper-bound} shows that the proposed model-based algorithm \DRNVI can achieve any robust solution among $\varepsilon$-robust $\{$NE, CCE, CE$\}$ when the total number of samples exceeds the order of 
\begin{align}
\widetilde{O}\bigg( \frac{ S H^3\prod_{1\leq i\leq n} A_i    }{  \varepsilon^2} \min \Big\{H,   \frac{1}{\min_{1 \leq i \leq n} \ror_i  } \Big\} \bigg). \label{eq:upper-bound-order}
\end{align}
Combining this with the lower bound in \eqref{eq:lower-bound} of Theorem~\ref{thm:robust-mg-lower-bound} confirms that the sample complexity of \DRNVI is optimal with respect to many salient factors, including $\varepsilon, S,H, \{\ror_i\}_{1\leq i \leq n}$. 
To the best of our knowledge, this is the first near-optimal sample complexity upper bound for solving robust MGs. As illustrated in Figure~\ref{fig:compare-results}, it uncovers that the sample requirement of \DRNVI depends on all agents' uncertainty levels $\{\ror_i\}$ and is inversely proportional to $\min_{i\in[n]} \ror_i$ when $\min_{i\in[n]} \ror_i \gtrsim 1/H$. Furthermore, the sample complexity of  \DRNVI (Theorem~\ref{thm:robust-mg-upper-bound}) significantly improve upon the prior art  $ \widetilde{O} \big(S^4 \left(\prod_{i=1}^n A_i\right)^3 H^4  /  \varepsilon^2 \big)$ \citep{blanchet2023double}.

 \begin{figure}[t]
\centering
	\includegraphics[width=0.6\linewidth]{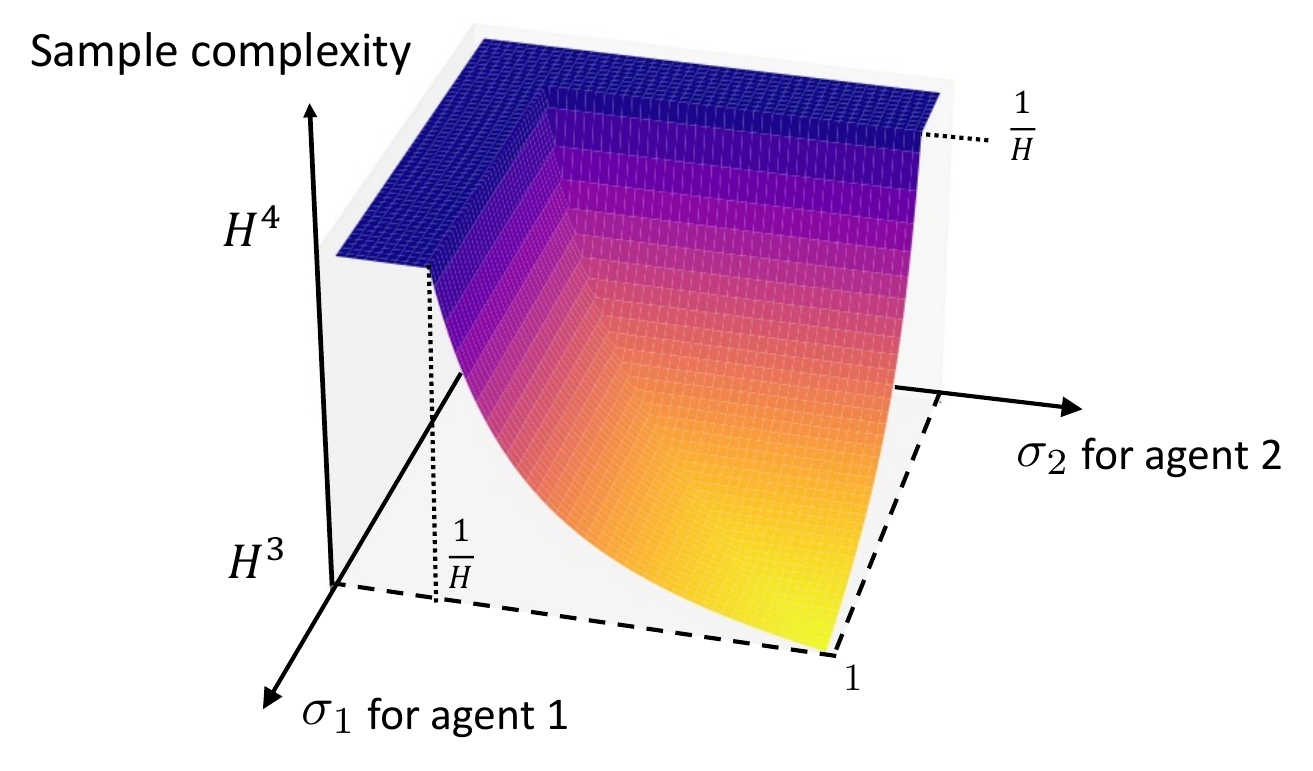} 
 \caption{Illustration of the sample complexity of \DRNVI with respect to the uncertainty levels $\ror_1$ and $\ror_2$ for two-player RMGs, where we only highlight the dependency with respect to the horizon length $H$.}
 \label{fig:compare-results}
 \end{figure}
 
\paragraph{Minimax-optimal sample complexity for single-agent RMDP.} We observe that when the size of the action space reduces to one except one agent, i.e. $A_2 = A_3 = \cdots = A_n = 1$, the robust MG simplifies to a single-agent robust Markov decision process (known as RMDP) \citep{iyengar2005robust}. Consequently, the upper bound of (cf.~\eqref{eq:upper-bound-order}) indicates that a simplified \DRNVI learns an $\varepsilon$-optimal policy for the RMDP associated with the first agent as soon as the sample complexity is on the order of 
\begin{align}
    \widetilde{O}\bigg( \frac{ S  A_1 H^3   }{  \varepsilon^2} \min \Big\{H,   \frac{1}{\ror_1  } \Big\} \bigg),
\end{align}
which is minimax-optimal in view of the lower bound (cf. \eqref{eq:lower-bound} of Theorem~\ref{thm:robust-mg-lower-bound}). To the best of our knowledge, these findings introduce the first minimax-optimal sample complexity for RMDPs in the finite-horizon setting, complementary to the infinite-horizon result established in \citet{shi2023curious}.

\paragraph{Benchmarking with standard MGs under non-adaptive sampling.}
Note that \DRNVI is based on a non-adaptive sampling mechanism from the generative model. Focusing on the same sampling mechanism, we compare the sample complexity of \DRNVI for solving robust MGs with the state-of-the-art approach (model-based \textsf{NVI}) \citep{zhang2020marl,liu2021sharp} for solving standard MGs as below\footnote{\citet{zhang2020marl} considered a two-player zero-sum standard MGs in the infinite-horizon setting. \citet{liu2021sharp} considered both two-player zero-sum and multi-player general sum standard MGs in online setting. We show the best possible outcomes after transferring into our settings}:
\begin{align}
  & \text{Standard MGs (by \textsf{NVI}):  } \qquad \qquad \text{Robust MGs (by our \DRNVI in Theorem~\ref{thm:robust-mg-upper-bound}): } \notag \\
   &\quad \widetilde{O} \bigg(\frac{S\prod_{i=1}^n A_i H^4}{  \varepsilon^2} \bigg)   \qquad \quad \qquad \qquad \begin{cases}
\widetilde{O} \Big(\frac{S\prod_{i=1}^n A_i H^4}{  \varepsilon^2} \Big) & \text{ if } 0 < \min\limits_{1\leq i\leq n} \ror_i \lesssim \frac{1}{H}\\
\widetilde{O}  \Big(\frac{S\prod_{i=1}^n A_i H^3}{ 
  \varepsilon^2 \min_{1 \leq i \leq n} \ror_i }  \Big) & \text{ if } \frac{1}{H} \lesssim \min\limits_{1 \leq i \leq n} \ror_i  <1 
  \end{cases}. \label{eq:compare-standard-robust-2}
\end{align}
It shows that  \DRNVI achieves enhanced robustness against model uncertainty in comparison to the prior art \textsf{NVI} for standard MGs, using the same or even sometimes fewer number of samples ($ \min_{1 \leq i \leq n} \ror_i \gtrsim 1/H$). In particular, as illustrated in Figure~\ref{fig:compare-results},
 \begin{itemize}
 	\item {\em When $ 0< \min_{1 \leq i \leq n} \ror_i  \lesssim \frac{1}{H}$}: the sample complexity dependency of \DRNVI on $H$ matches that of \textsf{NVI} in the order of $H^4$.
 	\item {\em When $\min_{1 \leq i \leq n} \ror_i  \gtrsim \frac{1}{H}$}: \DRNVI's sample complexity decreases towards $H^3$ as $\min_{1\leq i \leq n} \ror_i$ increases, which improves upon the sample complexity of \textsf{NVI} for standard MGs by a factor of $H \min_{1\leq i \leq n} \ror_i $ that goes to $H$ when $\min_{1\leq i \leq n} \ror_i = O(1)$. 
 	\end{itemize}

\paragraph{Technical challenges and insights.}
Compared to robust single-agent RL, robust MARL introduces complex statistical dependencies due to game-theoretical interactions between multiple agents and their natural adversaries to choose the worst-case transitions for each agent. Additionally, robust MGs are more intricate than standard MGs since the agents' payoffs become highly nonlinear without closed form, in contrast to being  linear  in standard MGs. To mitigate these challenges, we carefully control the statistical errors and exploit technical tools from distributionally robust optimization to achieve a near-optimal upper bound. Additionally, note that the established lower bound (Theorem~\ref{thm:robust-mg-lower-bound}) is the first information-theoretic lower bound for solving robust MGs, which is achieved by creating a new class of hard instances for the tightness with respect to $H$ and uncertainty levels $\{\ror_i\}_{1\leq i\leq n}$.

%% file: related-work.tex
\section{Related Works}\label{sec:relatd-work}
In this section, we discuss a non-exhaustive set of related works, limiting our discussions primarily to provable RL algorithms in the tabular setting, which are most related to this paper.

\paragraph{Finite-sample studies of standard Markov games.}
Multi-agent reinforcement learning (MARL), originated from the seminal work \citep{littman1994markov}, has been widely studied under the framework of standard Markov games \citep{shapley1953stochastic}; see \citet{busoniu2008comprehensive,zhang2021multi,oroojlooy2023review} for detailed reviews. There has been no shortage of provably convergent MARL algorithms with asymptotic guarantees \citep{littman1996generalized,littman2001friend}.

A line of recent efforts have concentrated on understanding and developing algorithms for standard MGs with non-asymptotic guarantees (finite-sample analysis). Within this field, Nash equilibrium (NE) is arguably one of the most compelling solution concepts for standard MGs. Research on calculating NE primarily focuses on an important basic class: standard two-player zero-sum MGs \citep{bai2020provable,chen2022almost,mao2022provably,wei2017online,tian2021online,cui2022offline,cui2022provably,zhong2022pessimistic,jia2019feature,yang2022t,yan2022model,dou2022gap,wei2021last}. This focus arises because computing NEs in scenarios beyond the standard two-player zero-sum MGs is generally computationally intractable (i.e., PPAD-complete) \citep{daskalakis2013complexity,daskalakis2009complexity}.

For discounted infinite-horizon two-player zero-sum Markov games, the state-of-the-art sample complexity for learning NE \citep{zhang2020model} remains suboptimal due to the "curse of multiple agents" issue \citep{zhang2020model}. In contrast, for episodic finite-horizon two-player zero-sum Markov games standard MGs, \citet{bai2020near,jin2021v,li2022minimax} have overcome this curse, progressively achieving minimax-optimal sample complexity in the order of $O(S \max_{1\leq i \leq n} A_i H^4/\varepsilon^2)$. Besides NE, \citet{jin2021v,daskalakis2022complexity,mao2022provably,song2021can,li2022minimax,liu2021sharp} have extended this achievement to other computationally tractable solution concepts (e.g., CE/CCE) in general-sum multi-player MGs.
Focusing on the same non-adaptive sampling mechanism considered in this work, the sample complexity for learning NE/CE/CCE in standard MGs with the state-of-the-art approaches \citep{zhang2020model,liu2021sharp} still suffers from the curse of multiple agents, calculated as $O(S \prod_{1\leq i \leq n} A_i H^4/\varepsilon^2)$.

\paragraph{Robustness in MARL.}
Despite significant advances in standard MARL, current algorithms may fail dramatically due to perturbations or uncertainties in game components, resulting in significantly deviated equilibrium, as illustrated in Figure~\ref{fig:intro}. 
A growing body of research is now addressing the robustness of MARL algorithms against uncertainties in various components of Markov games, such as state \citep{han2022solution,he2023robust,zhou2023robustness,zhang2023safe}, environment (reward and transition kernel), the type of agents \citep{zhang2021robust}, or other agents' policies  \citep{li2019robust,kannan2023smart}; see \citet{vial2022robust} for a recent review. 

This work considers the robustness against environmental uncertainty, adopting  distributionally robust optimization (DRO) that has primarily been investigated in the context of supervised learning \citep{rahimian2019distributionally,gao2020finite,bertsimas2018data,duchi2018learning,blanchet2019quantifying}. Applying DRO for single-agent RL \citep{iyengar2005robust} to handle model uncertainty has garnered significant attention. When turning to MARL, the problem is conceptualized as robust Markov games within the DRO framework, an area that remains relatively underexplored with only a few provable algorithms developed \citep{zhang2020robust,kardecs2011discounted,ma2023decentralized,blanchet2023double}. Notably, \citet{kardecs2011discounted} verifies the existence of Nash equilibrium for robust Markov games under mild assumptions; \citet{zhang2020robust} derives asymptotic convergence for a Q-learning type algorithm under certain conditions;  \citet{ma2023decentralized,blanchet2023double} are the most related works that provide algorithms with finite-sample guarantees for various types of uncertainty set. Especially, \citet{ma2023decentralized} considers a restricted uncertainty level that could fail to bring robustness to MARL in certain scenarios. In particular, as the required accuracy level ($\varepsilon$ goes to zero or the robust MGs has a small minimal positive transition probabilities ($p_\text{min} \rightarrow 0$), the required uncertainty level becomes quite restrictive (obeying $\ror_i \leq \max\{\frac{\varepsilon}{SH^2}, \frac{p_\text{min}}{H}\}$ for all $i\in[n]$) --- potentially reducing robust MARL to standard MARL and failing to maintain desired robustness.

\paragraph{Single-agent distributionally robust RL (robust MDPs).}
For single-agent RL, considering robustness to model uncertainty using DRO framework --- i.e., distributionally robust dynamic programming and robust MDPs --- has gained significant attention across both theoretical and practical domains \citep{iyengar2005robust,xu2012distributionally,wolff2012robust,kaufman2013robust,ho2018fast,smirnova2019distributionally,ho2021partial,goyal2022robust,derman2020distributional,tamar2014scaling,badrinath2021robust,roy2017reinforcement,derman2018soft,mankowitz2019robust}.  Recently, a substantial body of work has been dedicated to exploring the finite-sample performance of provable robust single-agent RL algorithms,  where different sampling mechanisms, diverse divergence function of the uncertainty set, and other related problems/issues has been investigated a lot \citep{yang2021towards,panaganti2021sample,zhou2021finite,shi2022distributionally,wang2023finite,blanchet2023double,liu2022distributionally,wang2023sample,liang2023single,shi2023curious,wang2021online,xu2023improved,dong2022online,badrinath2021robust,ramesh2023distributionally,panaganti2022robust,ma2022distributionally,wang2023foundation,li2022first,kumar2023policy,clavier2023towards,yang2023avoiding,zhang2023regularized,li2023first,he2024sample}.

Among the studies of robust MDPs, those particularly relevant to this paper employ the uncertainty set using total variation (TV) distance in a tabular setting \citep{yang2021towards,panaganti2021sample,xu2023improved,dong2022online,liu2024distributionally}. It has been established that solving robust MDPs requires no more samples than solving standard MDPs in terms of the sample requirement \citep{shi2023curious} with a generative model. However, robust MARL involves additional complexities compared to robust single-agent RL. It remains an open question whether the findings from robust MDPs can be generalized to robust MARL, which includes more technical challenges and strategic interactions.  Our work takes a step towards the question, confirming that similar phenomena apply in robust MARL, albeit with increased difficulties due to the multi-agent dynamics.

\paragraph{RL with a generative model.}
Access to a generative model (or simulator) serves as a fundamental and idealistic sampling protocol that has been widely used to study finite-sample guarantees for diverse types of RL algorithms, such as various model-based, model-free, and policy-based algorithms
\citep{kearns2002sparse, agarwal2020model,azar2013minimax,li2020breaking, sidford2018near, wainwright2019stochastic, li2023q, kakade2003sample,pananjady2020instance, khamaru2020temporal,even2003learning,beck2012error,zanette2019almost,yang2019sample,woo2023blessing}. This work follows this fundamental protocol with a non-adaptive sampling mechanism to understand and design algorithms for robust Markov games. 
Besides generative model, there also exist other sampling protocols that involve more realistic scenarios such as online exploration setting \citep{dong2019q,zhang2020reinforcement,zhang2020model,jafarnia2020model,liu2020gamma,yang2021q,zhang2023settling,li2021breaking} or offline setting \citep{xie2021policy,rashidinejad2021bridging,jin2021pessimism,yin2021towards,yan2022efficacy,uehara2021pessimistic,woo2024federated,shi2022pessimistic,li2024settling}, which are interesting directions in the future.

%% file: conclusion.tex
\section{Conclusion}
Providing robustness guarantees is a pressing need for RL, one that is especially crucial in multi-agent RL (MARL) since game-theoretical interactions between agents bring in extra instability.
We address the vulnerability of MARL to environmental uncertainty by focusing on robust Markov games (RMGs) that consider robustness against worst-case distribution shifts of the shared environment. We design a provable sample-efficient model-based algorithm (\DRNVI) with a finite-sample complexity guarantee. In addition, we provide a lower bound for solving RMGs, which highlights that \DRNVI has near-optimal sample complexity with respect to the size of the state space, the target accuracy, and the horizon length. To the best of our knowledge, this is the first algorithm with near-optimal sample complexity for RMGs. Our work opens up interesting future directions for robust MARL including but not limited to taming the curse of multi-agents and studying other divergence functions for the uncertainty set.


%% file: appendix.tex
\section{Preliminaries}

\subsection{Details of the example shown in Figure~\ref{fig:intro}}\label{proof:solution-for-example}

\paragraph{The standard Markov game for fishing protection.}
To simulate a scenario of defense against illegal fishing, we can formulate a two-player general sum finite-horizon standard Markov game between a fisher (the first player) and a police officer (the second player). This MG can be represented as $\mathcal{MG}^e = \big\{ \cS, \{\cA_i\}_{1 \le i \le 2}, p, r, H \big\}$. Here, $\cS \defn \{0,1,\cdots,100\}$ is the state space, where each state $s\in \cS$ represents the number of punishments received by the fisherman, with the license being revoked at $s=100$; $\cA_1 = \cA_2 = \{0,1\}$ is the action space. At each time step (round), the fisher chooses $a_1$ among space $\cA_1 = \{$legal fishing ($0$), illegal fishing ($1$)$\}$, while the officer chooses $a_2$ among $\cA_2 = \{$no patrols ($ 0$), go patrols ($1$)$\}$; $H$ is the horizon-length; the transition kernel is governed by a model parameter $p\in[0,1]$, shown in Figure~\ref{fig:reward}(a) (a detailed version of Figure~\ref{fig:intro}(a)), specified as
\begin{align}
\forall h\in[H]: \quad P_h(s^{\prime} \mymid s, a_1, a_2) = \left\{ \begin{array}{lll}
         p\mathds{1}(s^{\prime} = s+1) + (1-p)\mathds{1}(s^{\prime} = s)  & \text{if} \quad  s \in \cS\setminus \{100\}, a_1 = 1,\\
         \mathds{1}(s^{\prime} = s)   & \text{otherwise} .
                \end{array}\right.
        \label{eq:P-example}
\end{align}
In words, the state $s$ transit to $s'=s+1$ with probability $p$ when $a_1 = 1$, otherwise staying in $s'=s$, i.e.,
 In addition, $r=\{r_{i,h}\}_{i\in\{1,2\}, h\in[H]}$ represents the immediate reward (benefit) function of two players at each time step $h\in[H]$. Here, we consider time-invariant reward function $r_{i,h} =r_i$ for all $h\in[H]$. In particular, at any time step $h\in[H]$, $r_1(s,a_1,a_2,s')$ (resp.~$r_2(s,a_1,a_2,s')$) denotes the immediate benefit that the first agent (resp.~the second player) receives conditioned on the current state $s$, the actions of two players $(a_1, a_2)$, and the next state $s'$. The reward function for any state $s\in \cS\setminus \{100\}$ is defined in Figure~\ref{fig:reward}(b). And the reward function at state $s=100$ for two players is specified as below:
\begin{align}
\forall a_2 \in \{0,1\}: \quad &r_1(100,0, a_2, 100) = -1 \quad \text{and} \quad r_1(100,1, a_2, 100) = -20p  \notag \\
&r_2(100,0,0,100) = 1 \quad \text{and} \quad r_2(100,0,1,100) = 0 \notag \\
 &r_2(100,1,0,100) = 1 \quad \text{and} \quad r_2(100,1,1,100) = 3-2p. \label{eq:example-reward-100}
\end{align}

\begin{figure}[t]
\centering
	\includegraphics[width=1.0\linewidth]{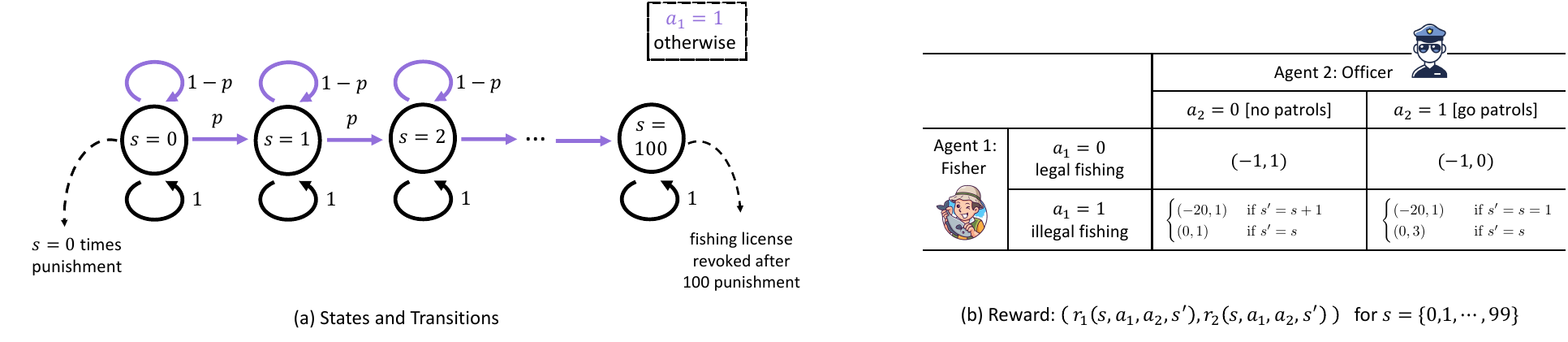} 
	\caption{(a) shows the transition kernels of the game at each time step $h$. (b) illustrates the immediate reward function of two agents.}
	\label{fig:reward}
\end{figure}

\paragraph{Computing the Nash equilibrium (NE).}
Notice that the NE of a standard Markov game is indeed a series of NE of the matrix games associated with the Q-function at each time step $h\in[H]$. We denote the NE of $\mathcal{MG}^e$ as $\pi^\star =(\mu^\star,\nu^\star) =  \{\mu^\star_h, \nu^\star_h\}_{h\in[H]}$ with $\mu^\star_h: \cS   \mapsto \Delta(\ac_1), \nu^\star_h: \cS   \mapsto \Delta(\ac_2)$ for all $h\in[H]$. To proceed, we start from characterizing the Q-function and Bellman consistency equation of the fishing protection game.

It is easily verified that for time step $H+1$, one has for any joint policy $\pi = (\mu, \nu)$,
\begin{align} \label{eq:example-Q-same-H+1}
\forall (i,a_1, a_2)\in\{1,2\} \times \cA_1 \times \cA_2: \quad  Q_{i,H+1}^{\pi,P}(s, a_1, a_2) = Q_{i,H+1}^{\pi,P}(s', a_1, a_2) =0.
\end{align}

Then, we characterize the Bellman consistency equation at time step $h=H, H-1, \cdots, 1$ for the optimal policy $\pi^\star$.  Notice that the rewards and the transition kernels have similar structures for all states except $s=100$. So we start from the cases when  $s\in \cS \setminus  \{100\}$.
  Recalling the definition of Q-function in \eqref{eq:value-function-defn}, the reward function $r$ (defined in Figure~\ref{fig:reward}(b)) and the transition kernel in \eqref{eq:P-example}, we have for any state $s\in \cS \setminus  \{100\}$ and any time step $h\in[H]$, the Q-function of the fisher (the first player) obeys
\begin{align}
    Q_{1,h}^{\pi^\star, P}(s,0,0) &= -1 +  V_{1,h+1}^{\pi^\star, P}(s) \notag \\
    Q_{1,h}^{\pi^\star, P}(s,0,1) &= -1 + V_{1,h+1}^{\pi^\star, P}(s) \notag \\
    Q_{1,h}^{\pi^\star, P}(s,1,0) &= -20p + pV_{1,h+1}^{\pi^\star, P}(s+1) + (1-p)V_{1,h+1}^{\pi^\star, P}(s)\notag \\
    Q_{1,h}^{\pi^\star, P}(s,1,1) &= -20p + pV_{1,h+1}^{\pi^\star, P}(s+1) + (1-p)V_{1,h+1}^{\pi^\star, P}(s). \label{eq:intro-agent-1-payoff}
\end{align}

Similarly, for the officer (the second player), we observe that for any state $s\in \cS \setminus  \{100\}$ and any time step $h\in[H]$:
\begin{align}
      Q_{2,h}^{\pi^\star, P}(s,0,0) &= 1 + V_{2,h+1}^{\pi^\star, P}(s) \notag \\
      Q_{2,h}^{\pi^\star, P}(s,0,1) &= 0 + V_{2,h+1}^{\pi^\star, P}(s), \notag \\
      Q_{2,h}^{\pi^\star, P}(s,1,0) &= 1 + pV_{2,h+1}^{\pi^\star, P}(s+1) + (1-p)V_{2,h+1}^{\pi^\star, P}(s), \notag \\
      Q_{2,h}^{\pi^\star, P}(s,1,1) &= 3-2p + pV_{2,h+1}^{\pi^\star, P}(s+1) + (1-p)V_{2,h+1}^{\pi^\star, P}(s).\label{eq:intro-agent-2-payoff}
\end{align}
Armed with above results, we are now ready to show that the NE for all $(h,s)\in[H] \times \cS$ are the same, which determined by the model parameter $p$ as below:
\begin{align}
  \forall (h,s)\in[H] \times \cS: \quad \begin{cases} \pi^\star_h(s) = \pi_B = (0,0) & \text{if } p >0.05 \\
 \pi^\star_h(s) = \pi_A = (1,1)& \text{if } p \leq 0.05. 
 \end{cases} \label{eq:example-NE-solution}
\end{align}
We will verify it by induction as below:
\begin{itemize}
	\item {\em Base case: when $h=H$.}
Applying \eqref{eq:intro-agent-1-payoff} and \eqref{eq:intro-agent-2-payoff} for $h=H$ with the fact in \eqref{eq:example-Q-same-H+1}, we arrive at for any state $s\in \cS \setminus  \{100\}$:
\begin{align}
 \forall a_2\in \{0,1\}: \quad Q_{1,H}^{\pi^\star, P}(s,0,a_2) &= -1 \quad \text{and} \quad  Q_{1,H}^{\pi^\star, P}(s,1,a_2) = -20p \notag \\
Q_{2,H}^{\pi^\star, P}(s,0,0) &= 1 \quad \text{and} \quad Q_{2,H}^{\pi^\star, P}(s,0,1) = 0 \notag \\
Q_{2,H}^{\pi^\star, P}(s,1,0) &= 1 \quad \text{and} \quad Q_{2,H}^{\pi^\star, P}(s,1,1) = 3-2p \label{eq:intro-agent-2-payoff-H}
\end{align}

Similarly, when state $s= 100$, recalling the reward function in \eqref{eq:example-reward-100}, we achieve the same Q-function on state $s=100$. Therefore, one has for all $s\in\cS$:
\begin{align}
  \forall a_2\in \{0,1\}: \quad Q_{1,H}^{\pi^\star, P}(s,0,a_2) &= -1 \quad \text{and} \quad  Q_{1,H}^{\pi^\star, P}(s,1,a_2) = -20p \notag \\
Q_{2,H}^{\pi^\star, P}(s,0,0) &= 1 \quad \text{and} \quad Q_{2,H}^{\pi^\star, P}(s,0,1) = 0 \notag \\
Q_{2,H}^{\pi^\star, P}(s,1,0) &= 1 \quad \text{and} \quad Q_{2,H}^{\pi^\star, P}(s,1,1) = 3-2p.  \label{eq:intro-agent-2-payoff-H-alls}
\end{align}

Consequently, in view of \eqref{eq:intro-agent-2-payoff-H-alls}, it can be verified that if $p<0.05$ (resp.~$p>0.05$), the unique NE of two agents on any state $s\in\cS$ at time step $H$ is the policy pair $\pi^\star_H(s) = (\mu^\star_H(s), \nu^\star_H(s)) = (1,1)$ (resp.~$\pi^\star_H(s) =(\mu^\star_H(s), \nu^\star_H(s)) =(0,0)$), leading to Nash $\pi_A \defn (1,1)$ when $p=p_A = 0.049$ (resp.~Nash $\pi_B \defn (0,0)$ when $p=p_B = 0.051$).

In addition, we observe the optimal value function satisfies that:
\begin{align}
	\forall s\in\cS: \begin{cases} V_{1,H}^{\pi^\star, P}(s) = -1 \quad \text{and} \quad  V_{2,H}^{\pi^\star, P}(s) = 1 & \quad \text{ if } p >0.05  \\
 V_{1,H}^{\pi^\star, P}(s) = -20p \quad \text{and} \quad  V_{2,H}^{\pi^\star, P}(s) = 3-2p  & \quad  \text{ if  } p \leq 0.05
 \end{cases}. \label{eq:example-value-H}
\end{align}

\item  {\em Induction.} The rest of this paragraph is to verify \eqref{eq:example-NE-solution} for all $(h,s)\in [H-1] \times \cS$ by induction. 
So suppose \eqref{eq:example-NE-solution} holds for time step $h+1$, then we will show that it also holds for time step $h$.

To begin with, we introduce the following claim which will be verified in Appendix~\ref{proof:eq:example-Q-same}: for any policy $\pi = (\mu,\nu)$ and any $s, s' \in\cS$:
\begin{align} \label{eq:example-Q-same}
\forall (i,h)\in\{1,2\} \times [H]: \quad  V_{i,h}^{\pi,P}(s) = V_{i,h}^{\pi,P}(s').
\end{align}

To proceed, armed with the fact in \eqref{eq:example-Q-same}, invoking the results in \eqref{eq:intro-agent-1-payoff} and \eqref{eq:intro-agent-2-payoff} yields that for all $s\in\cS$:
\begin{align}
Q_{1,h}^{\pi^\star, P}(s,0,0) &= -1 +  V_{1,h+1}^{\pi^\star, P}(s) \quad \text{and} \quad  Q_{2,h}^{\pi^\star, P}(s,0,0) = 1 + V_{2,h+1}^{\pi^\star, P}(s) \notag \\
    Q_{1,h}^{\pi^\star, P}(s,0,1) &= -1 + V_{1,h+1}^{\pi^\star, P}(s)  \quad \text{and} \quad Q_{2,h}^{\pi^\star, P}(s,0,1) = 0 + V_{2,h+1}^{\pi^\star, P}(s)\notag \\
    Q_{1,h}^{\pi^\star, P}(s,1,0) &=  -20p + V_{1,h+1}^{\pi^\star, P}(s)  \quad \text{and} \quad Q_{2,h}^{\pi^\star, P}(s,1,0) = 1 +V_{2,h+1}^{\pi^\star, P}(s)\notag \\
    Q_{1,h}^{\pi^\star, P}(s,1,1) &= -20p + V_{1,h+1}^{\pi^\star, P}(s)  \quad \text{and} \quad Q_{2,h}^{\pi^\star, P}(s,1,1) = 3-2p + V_{2,h+1}^{\pi^\star, P}(s). \label{eq:optimal-bellman-standard}
\end{align}

The above fact directly indicates that at time step $h$, the NE of the matrix games associated with the payoff $Q_{1,h}^{\pi^\star, P}(s)$ and  $Q_{2,h}^{\pi^\star, P}(s)$ satisfies
\begin{align}
		\forall s\in\cS: \quad \begin{cases} \pi^\star_h(s) = (0,0) & \text{if } p >0.05 \\
 \pi^\star_h(s) = (1,1)& \text{if } p \leq 0.05.
 \end{cases}
\end{align}

\end{itemize}

Summing up the base case and the induction results, we complete the proof for \eqref{eq:example-NE-solution}.

\paragraph{The robust MG and computing the robust Nash equilibrium (robust NE).}

When turns to the robust formulation of the fishing protection game, we construct a robust Markov game represented as $\mathcal{MG}^e_{\mathsf{rob}} = \big\{ \cS, \{\cA_i\}_{1 \le i \le 2}, p^\no, \ror, r, H \big\}$, where $\cS, \{\cA_i\}_{1 \le i \le 2}, r, H $ are the same as those defined in the standard MG $\mathcal{MG}^e$. Note that this example is designed to illustrate general environmental uncertainty (includes both the reward and transition kernel uncertainty) and is not tailored to the specific class of robust MGs defined in Section~\ref{sec:framework-robust-marl}. For simplicity, let each agent consider that the model parameter $p$ can perturb around some nominal one $p^0$ with uncertainty level $\ror =0.005$, i,e., $p\in [p^0 -\ror, p^0+\ror]$. Other components of the transition kernel is not allowed to perturb. With abuse of notation, for any joint policy $\pi$, we still denote the robust value function (resp.~robust Q-function) for $i$-th agent at time step $h$ as $V_{i,h}^{\pi,\ror}$ (resp.~$Q_{i,h}^{\pi,\ror}$). In addition, we denote the robust NE of $\mathcal{MG}_{\mathsf{rob}}^e$ as $\pi^{\star,\ror} =(\mu^{\star,\ror},\nu^{\star,\ror}) =  \{\mu^{\star,\ror}_h, \nu^{\star,\ror}_h\}_{h\in[H]}$, where $\mu^{\star,\ror}_h: \cS   \mapsto \Delta(\ac_1), \nu^{\star,\ror}_h: \cS \mapsto \Delta(\ac_2)$. 

Observe that in city A (resp.~city B), the nominal model parameter $p^0 = 0.049$ (resp.~$p^0 = 0.051$). Without loss of generality, we first focus on city A. To proceed, we shall verify the following claim using the same routine for computing NE of the standard MG $\mathcal{MG}^e$ (cf.~\eqref{eq:example-NE-solution}):
\begin{align}
  \text{In city A}: \quad (\mu^{\star,\ror}_h(s), \nu^{\star,\ror}_h(s)) = (0,0), \quad \forall (h,s)\in[H] \times \cS.
\label{eq:example-robust-NE-solution}
\end{align}

\begin{itemize}

  \item {\em Base case: when $h=H$.}
Recall the definitions of robust value/Q-function (cf.~\eqref{eq:value-function-defn-robust}), one has at time step $H$: for all $s\in\cS$,
\begin{align}
  \forall a_2\in \{0,1\}: \quad Q_{1,H}^{\pi^\star, \ror}(s,0,a_2) &= -1 \quad \text{and} \quad  Q_{1,H}^{\pi^\star, \ror}(s,1,a_2) = -20(p^0+\ror) = -1.08 \notag \\
Q_{2,H}^{\pi^\star, \ror}(s,0,0) &= 1 \quad \text{and} \quad Q_{2,H}^{\pi^\star, \ror}(s,0,1) = 0 \notag \\
Q_{2,H}^{\pi^\star, \ror}(s,1,0) &= 1 \quad \text{and} \quad Q_{2,H}^{\pi^\star, \ror}(s,0,1) = 3-2 (p^0+\ror)= 2.892.  \label{eq:intro-agent-2-payoff-H-alls}
\end{align}
As a result, it is easily verified that the unique robust NE of two agents on any state $s \in \cS$ at time step $H$ is the policy pair $(\mu^{\star,\ror}_H(s), \nu^{\star,\ror}_H(s)) = (0,0)$.

 \item {\em Induction.}
First of all, for any policy $\pi = (\mu,\nu)$ and $s,s'\in\cS$, similar to \eqref{eq:example-Q-same}
\begin{align} \label{eq:example-Q-same-robust}
\forall (i,h)\in\{1,2\} \times [H] : \quad  V_{i,h}^{\pi,\ror}(s) = V_{i,h}^{\pi,\ror}(s').
\end{align}
which indicates that the worst-case performance are indeed influenced by the uncertainty of the reward function but not the transition kernel perturbation.
Armed with above fact, invoking the robust Bellman consistency equation, similar to \eqref{eq:optimal-bellman-standard}, we can achieve that for all $h\in 1,2,\cdots, H-1$,
\begin{align}
Q_{1,h}^{\pi^{\star,\ror}, \ror}(s,0,0) &= -1 +  V_{1,h+1}^{\pi^{\star,\ror}, \ror}(s)  \quad \text{and} \quad  Q_{2,h}^{\pi^{\star,\ror}, \ror}(s,0,0) = 1 + V_{2,h+1}^{\pi^{\star,\ror}, \ror}(s) \notag \\
    Q_{1,h}^{\pi^{\star,\ror}, \ror}(s,0,1) &= -1 + V_{1,h+1}^{\pi^{\star,\ror}, \ror}(s)  \quad \text{and} \quad Q_{2,h}^{\pi^{\star,\ror}, \ror}(s,0,1) = 0 + V_{2,h+1}^{\pi^{\star,\ror}, \ror}(s)\notag \\
    Q_{1,h}^{\pi^{\star,\ror}, \ror}(s,1,0) &=  -1.08 + V_{1,h+1}^{\pi^{\star,\ror}, \ror}(s)  \quad \text{and} \quad Q_{2,h}^{\pi^{\star,\ror}, \ror}(s,1,0) = 1 +V_{2,h+1}^{\pi^{\star,\ror}, \ror}(s)\notag \\
    Q_{1,h}^{\pi^{\star,\ror}, \ror}(s,1,1) &= -1.08 + V_{1,h+1}^{\pi^{\star,\ror}, \ror}(s)  \quad \text{and} \quad Q_{2,h}^{\pi^{\star,\ror}, \ror}(s,1,1) = 2.892 + V_{2,h+1}^{\pi^{\star,\ror}, \ror}(s).
\end{align}

As a consequence,  the robust NE of the matrix games associated with the payoff $Q_{1,h}^{\pi^{\star,\ror}, \ror}(s)$ and  $Q_{2,h}^{\pi^{\star,\ror}, \ror}(s)$ satisfies  $(\mu^{\star,\ror}_h(s), \nu^{\star,\ror}_h(s)) = (0,0)$ for all $h\in 1,2,\cdots, H-1$.

\end{itemize}

Summing up the results in the base case and the induction, we verify the unique robust NE for $\mathcal{MG}^e_{\mathsf{rob}}$ in city A as \eqref{eq:example-robust-NE-solution}. 
The same unique robust NE can be verified in city B by following the same routine, which we omit for brevity. Thus, we show the unique robust NE in two slightly different environments (city A and city B) are identical.

\paragraph{Deriving the states of executing different equilibrium solutions.}
In view of \eqref{eq:example-NE-solution}, we know  that the NE of the standard MG $\mathcal{MG}^e$ in city A when $p= p_A = 0.049$ (resp.~city B when $p=p_B = 0.051$) is $\pi_A = (1,1)$ (resp.~$\pi_B = (0,0)$) for all $(h,s)\in [H] \times \cS$.
And the MG $\mathcal{MG}^e$ has some one-way transition structure, namely state $s$ can only transit to itself or a larger state $s+1$, while not any states $s' <s$. So as long as $H$ is large enough, the final state of executing $\pi_A = (1,1)$ will be state $s=100$ with the fishing license revoked since the fisher will always do illegal fishing $(a_1 = 1)$. The agents who execute the joint policy $\pi_B = (0,0)$ or the robust NE $(\mu^{\star,\ror}_h(s), \nu^{\star,\ror}_h(s)) = (0,0)$ will stay in $s=0$ with no punishment since the fisher will never choose illegal fishing ($a_1 = 1$).
 
\subsubsection{Proof of claim  \eqref{eq:example-Q-same}}\label{proof:eq:example-Q-same}
We will proof \eqref{eq:example-Q-same} by induction.
Note that the base case when $h=H$ has already been verified in \eqref{eq:example-value-H}.

Then suppose the claim holds at time step $h+1$, i.e., 
\begin{align} \label{eq:example-Q-same-h+1}
\forall (i,s,s') \in\{1,2\}\times \cS \times \cS: \quad  V_{i,h+1}^{\pi,P}(s) = V_{i,h+1}^{\pi,P}(s'),
\end{align}
it remains to show that the claim holds at time step $h$ as well.

Towards this, we first consider the  cases when state $s\in \cS \setminus  \{100\}$. Recall the recursion in  \eqref{eq:intro-agent-1-payoff}, we arrive at
	\begin{align}
	 Q_{1,h}^{\pi, P}(s,0,0) &= -1 +  V_{1,h+1}^{\pi, P}(s) \notag \\
    Q_{1,h}^{\pi, P}(s,0,1) &= -1 + V_{1,h+1}^{\pi, P}(s) \notag \\
    Q_{1,h}^{\pi, P}(s,1,0) &= -20p + pV_{1,h+1}^{\pi, P}(s+1) + (1-p)V_{1,h+1}^{\pi^\star, P}(s) \overset{\mathrm{(i)}}{=} -20p + V_{1,h+1}^{\pi, P}(s) \notag \\
    Q_{1,h}^{\pi, P}(s,1,1) &= -20p + pV_{1,h+1}^{\pi, P}(s+1) + (1-p)V_{1,h+1}^{\pi^\star, P}(s) \overset{\mathrm{(ii)}}{=} -20p + V_{1,h+1}^{\pi, P}(s), \label{eq:intro-agent-1-payoff-h}
    \end{align}
    where (i) and (ii) holds by the induction assumption in \eqref{eq:example-Q-same-h+1}.

Analogously, recalling \eqref{eq:intro-agent-2-payoff} for the second player (protector), we arrive at for any state $s\in \cS \setminus  \{100\}$ and time step $h\in[H]$,
\begin{align}
      Q_{2,h}^{\pi, P}(s,0,0) &= 1 + V_{2,h+1}^{\pi, P}(s) \notag \\
      Q_{2,h}^{\pi, P}(s,0,1) &= 0 + V_{2,h+1}^{\pi, P}(s), \notag \\
      Q_{2,h}^{\pi, P}(s,1,0) &= 1 +V_{2,h+1}^{\pi, P}(s), \notag \\
      Q_{2,h}^{\pi, P}(s,1,1) &= 3-2p + V_{2,h+1}^{\pi, P}(s). \label{eq:intro-agent-2-payoff-h}
\end{align}
Combining \eqref{eq:intro-agent-1-payoff-h} and \eqref{eq:intro-agent-2-payoff-h} gives that for any $s,s'\in\cS \setminus \{100\}$,
\begin{align} \label{eq:example-Q-same-h+1}
\forall (i,a_1, a_2) \in\{1,2\}\times \Delta(\cA_1) \times \Delta(\cA_2): \quad  Q_{i,h}^{\pi,P}(s,a_1, a_2) = Q_{i,h}^{\pi,P}(s', a_1, a_2),
\end{align}
which indicates
\begin{align}
 V_{i,h}^{\pi,P}(s) = \mathbb{E}_{(a_1,a_2) \in \mu(s) \times \mu(s)}[Q_{i,h}^{\pi,P}(s,a_1, a_2) ] = \mathbb{E}_{(a_1,a_2) \in \mu(s) \times \mu(s)}[Q_{i,h}^{\pi,P}(s',a_1, a_2) ] =  V_{i,h}^{\pi,P}(s').
\end{align}

Similarly, when $s=100$, it can be verified that \eqref{eq:intro-agent-1-payoff-h} and \eqref{eq:intro-agent-2-payoff-h} also hold. Therefore, we complete the induction argument by observing that for all $s,s'\in\cS$,  $V_{i,h}^{\pi,P}(s) = V_{i,h}^{\pi,P}(s')$ is satisfied.

\subsection{Additional notation and basic facts}\label{sec:notation}

For convenience, for any two vectors $x=[x_i]_{1\leq i\leq n}$ and $y=[y_i]_{1\leq i\leq n}$, the notation $ {x}\leq {y}$ (resp.~$ {x}\geq {y}$) means $x_{i}\leq y_{i}$ (resp.~$x_{i}\geq y_{i}$) for all $1\leq i\leq n$. 
We denote by $x \circ y=\big[x(s) \cdot y(s) \big]_{s\in\cS}$ the Hadamard product of any two vectors $x, y\in\mathbb{R}^S$.  And for any vecvor $x$, we let $x^{\circ 2} = \big[x(s,a)^2\big]_{(s,a)\in\cS\times\cA}$ (resp.~$x^{\circ 2} = \big[x(s)^2\big]_{s\in\cS}$). 
With slight abuse of notation, we denote ${0}$ (resp.~${1}$) as the all-zero (resp.~all-one) vector, and $e_i \in\mathbb{R}^S$ as a $S$-dimensional basis vector with the $i$-th entry being $1$ and others being $0$.
Recall that we abbreviate the subscript $\rho_{\mathsf{TV}}$ when the divergence function is specified to TV distance to write $ \cU^\ror(\cdot) = \cU^\ror_{\rho_{\mathsf{TV}}}(\cdot)$.

\paragraph{Additional matrix notation.}
For any $(i,h)\in [n] \times [H]$, we recall or introduce some additional notation and matrix notation that is useful throughout the analysis
\begin{itemize}

	\item $r_{i,h} = [r_{i,h}(s,\ba)]_{(s,\ba)\in \cS\times \cA} \in \mathbb{R}^{S \prod_{i=1}^n A_i}$: a reward vector that represents the reward function for the $i$-th player at  time step $h$.
	
	\item $\Pi^{\pi}_h \in \mathbb{R}^{S \times S \prod_{i=1}^n A_i}$: a projection matrix associated with time step $h$ and a given joint policy $\pi = \{\pi_{h}\}_{h\in[H]}$ in the  following form
\begin{align}
\label{eqn:bigpi}
	\Pi^{\pi}_h = {\scriptsize
	\begin{pmatrix}
		\pi_{h}(1)^{\top} &        {0}^{\top}     &  \cdots &  {0}^{\top} \\
		        {0}^{\top}     & \pi_h(2)^{\top} &  \cdots &  {0}^{\top} \\
			  \vdots  &        \vdots    & \ddots & \vdots \\	
		         {0}^{\top}    &      {0}^{\top}        &    \cdots     &  \pi_h(S)^{\top}
	\end{pmatrix}  },
\end{align}
where we recall $\pi_h(s) = \left[\pi_h(s,\ba)\right]_{\ba\in \cA}  \in \Delta(\cA)$ for all $s\in\cS$ denote the joint policy vectors from all agents.

 \item $r_{i,h}^{\pi} \in \mathbb{R}^{S}$: a reward vector associated with the distribution of actions chosen by any joint policy $\pi = \{\pi_h\}_{h\in[H]}$ at time step $h$. Here,  $r_{i,h}^{\pi}(s) = \mathbb{E}_{\ba \sim\pi_h(s)} [r_{i,h}(s,\ba)]$ for all $s\in \cS$, or equivalently $r_{i,h}^{\pi}=\Pi^{\pi}_h r_{i,h}$ (see \eqref{eqn:bigpi}).
\item $P_h^\no \in \mathbb{R}^{S\allA \times S}$: the matrix of the nominal transition kernel at time step $h$, with $P^\no_{h,s, \ba} \in \mathbb{R}^{1\times S}$ serves as the $(s,\ba)$-th row for any $(s,\ba) \in \cS\times \cA$.
\item $\widehat{P}_h^\no \in \mathbb{R}^{S\allA\times S}$: the matrix of the estimated nomimal transition kernel at time step $h$, with $\widehat{P}^\no_{h,s,\ba} \in \mathbb{R}^{1\times S}$ serves as the $(s,\ba)$-th row for any $(s,\ba) \in \cS\times \cA$.

\item $\pmin_{i,h}^{V} \in \mathbb{R}^{S \allA \times S}$, $\pmhat_{i,h}^{V} \in \mathbb{R}^{S\allA\times S}$: at time step $h$, those matrices represent the worst-case probability transition kernel within the $i$-th agent's uncertainty set around the nominal/estimated nominal transition kernel, associated with any vector $V\in\mathbb{R}^S$. As a result, we denote $\pmin_{i,h,s,\ba}^{V}$ (resp.~$\pmhat_{i,h,s,\ba}^{V}$) as the $(s,\ba)$-th row of the transition matrix $\pmin_{i,h}^{V}$ (resp.~$\widehat{\pmin}_{i,h}^{V}$), defined by
	\begin{subequations}\label{eq:inf-p-special-marl}
\begin{align}
	\pmin_{i,h,s,\ba}^{V} &= \mathrm{argmin}_{\cP\in \unb^{\ror_i}_\rho(P^{\no}_{h,s,\ba})} \cP V, \qquad \text{and} \qquad  \pmhat_{i,h,s,\ba}^{V} = \mathrm{argmin}_{\cP\in \unb^{\ror_i}_\rho(\widehat{P}^{\no}_{h,s,\ba})} \cP V.
\end{align}

Similarly, we define the corresponding probability transition matrices for some special value vectors that are useful: $\pmin_{i,h}^{\pi, V} \in \mathbb{R}^{S\allA\times S}$, $\pmin_{i,h}^{\pi, \widehat{V}} \in \mathbb{R}^{S\allA\times S}$,  $\pmhat_{i,h}^{\pi, V} \in \mathbb{R}^{S\allA\times S}$ and $\pmhat_{i,h}^{\pi, \widehat{V}} \in \mathbb{R}^{S\allA\times S}$.
Here, we already use the following short-hand notation:
\begin{align}
	\pmin_{i,h}^{\pi, V} &\defn \pmin_{i,h}^{V_{i,h+1}^{\pi,\ror_i}} \quad \text{and} \quad \pmin_{i,h,s,\ba}^{\pi, V} := \pmin_{i,h,s,\ba}^{V_{i,h+1}^{\pi,\ror_i}}= \mathrm{argmin}_{\cP\in \unb^{\ror_i}_\rho(P^{\no}_{h,s,\ba})} \cP V_{i,h+1}^{\pi,\ror_i}, \notag \\
	 \pmin_{i,h}^{\pi, \widehat{V}} &\defn \pmin_{i,h}^{\widehat{V}_{i,h+1}^{\pi, \ror_i} } \quad \text{and} \quad  \pmin_{h,s,\ba}^{\pi, \widehat{V}}:=\pmin_{h,s,\ba}^{\widehat{V}_{i,h+1}^{\pi, \ror_i}}  = \mathrm{argmin}_{\cP\in \unb^{\ror_i}_\rho(P^{\no}_{h,s,\ba})} \cP \widehat{V}_{i,h+1}^{\pi,\ror_i},  \notag \\
	 \pmhat_{i,h}^{\pi, V} &:= \pmhat_{i,h}^{V_{i,h+1}^{\pi, \ror_i}}  \quad \text{and} \quad \pmhat_{h,s,\ba}^{\pi, V} := \pmhat_{h,s,\ba}^{V_{i,h+1}^{\pi, \ror_i}} = \mathrm{argmin}_{P\in \unb^{\ror_i}_\rho(\widehat{P}^{\no}_{h,s,\ba})} P V_{i,h+1}^{\pi,\ror_i}, \notag \\
	  \pmhat_{i,h}^{\pi, \widehat{V}}&:=\pmhat_{i,h}^{\widehat{V}_{i,h+1}^{\pi, \ror_i}}  \quad \text{and} \quad \pmhat_{h,s,\ba}^{\pi, \widehat{V}}:=\pmhat_{h,s,\ba}^{\widehat{V}_{i,h+1}^{\pi, \ror_i}}  = \mathrm{argmin}_{P\in \unb^{\ror_i}_\rho(\widehat{P}^{\no}_{h,s,\ba})} P \widehat{V}_{i,h+1}^{\pi, \ror_i}. 
\end{align}
\end{subequations}

\item $P_{h}^{\pi} \in \mathbb{R}^{S\times S}$, $\widehat{P}_{h}^{\pi} \in \mathbb{R}^{S\times S}$, $\Pv_{i,h}^{\pi,V}\in \mathbb{R}^{S\times S}$, $\Pv_{i,h}^{\pi, \widehat{V}}\in \mathbb{R}^{S\times S}$, $\Phatv_{i,h}^{\pi, V} \in \mathbb{R}^{S\times S}$ and $\Phatv_{i,h}^{\pi, \widehat{V}} \in \mathbb{R}^{S\times S}$: at time step $h$, those six {\em square} probability  transition matrices w.r.t. a given joint policy $\pi$ are defined by multiplying the projection matrix in \eqref{eqn:bigpi} as below, resepctively:  
	\begin{align}
	\label{eqn:ppivq-marl}
		&\Pv_{h}^{\pi} \defn \Pi^{\pi}_h P_{h}^\no, \qquad \Phatv_{h}^{\pi} \defn  \Pi^{\pi}_h \widehat{P}_{h}^\no, \qquad  \Pv_{i,h}^{\pi,V} \defn \Pi^{\pi}_h\pmin_{i,h}^{\pi, V}, \qquad \Pv_{i,h}^{\pi, \widehat{V}} \defn \Pi^{\pi}_h\pmin_{i,h}^{\pi, \widehat{V}}, \nonumber \\
		&\Phatv_{i,h}^{\pi, V}  \defn \Pi_h^{\pi} \pmhat_{i,h}^{\pi, V} , \qquad \text{and} \qquad \Phatv_{i,h}^{\pi, \widehat{V}} \defn \Pi_h^{\pi} \pmhat_{i,h}^{\pi, \widehat{V}}.
	\end{align}

\end{itemize}

We then introduce two notations of the variance.
First, for any probability vector $P \in \mathbb{R}^{1 \times S}$ and vector $V \in \mathbb{R}^S$, we denote the variance
\begin{align}\label{eq:defn-variance}
  \mathrm{Var}_{P}(V) \defn P (V \circ V)-  (P V  ) \circ  (P V  ).
\end{align} 
Then in addition, for any transition kernel $P \in \mathbb{R}^{S\allA \times S}$ and vector $V \in \mathbb{R}^S$, we denote $\mathsf{Var}_{P}(V) \in \mathbb{R}^{S\allA}$ as a vector of variance whose $(s,\ba)$-th row of $\mathsf{Var}_{P}(V)$ is taken as
\begin{align}\label{eq:defn-variance-vector-marl}
	\mathsf{Var}_{P}(s,\ba) \defn \mathrm{Var}_{P_{s,\ba}}(V).
\end{align}

\subsection{Preliminary facts of \rmgs and empirical \rmgs}

\paragraph{Dual equivalence of robust Bellman operator with TV uncertainty set.} Opportunely, when the prescribed uncertainty set is in a benign form (such as using TV distance as the divergence function), the robust Bellman operator can be computed efficiently by solving its dual formulation instead \citep{iyengar2005robust,clavier2023towards,shi2023curious}.
In particular, the following lemma describes the equivalence between the robust Bellman operator and its dual form due to strong duality in the case of TV distance.
\begin{lemma}[Lemma~4, \citet{shi2023curious}]\label{lemma:tv-dual-form}
Consider any TV uncertainty set $\unb^{\ror}(P) = \unb^{\ror}_{\rho_{\mathsf{TV}}}(P)$ associated with any probability vector  $P\in\Delta(\cS)$, fixed uncertainty level $\ror\in(0,1]$. For any vector $V\in \mathbb{R}^S$ obeying $  V \geq  {0}$, recalling the definition of $[V]_\alpha$ in \eqref{eq:V-alpha-defn}, one has
\begin{align}
	\inf_{ \cP \in \unb^{\ror}(P)} \cP V &= \max_{\alpha\in [\min_s V(s), \max_s V(s)]} \left\{P \left[V\right]_{\alpha} - \ror~ \left(\alpha - \min_{s'}\left[V\right]_{\alpha}(s') \right)\right\} \label{eq:vi-l1norm}. 
\end{align}
\end{lemma}
The above lemma ensures that the computation cost of applying robust Bellman operator is relatively the same as applying standard Bellman operator \citep{iyengar2005robust} up to some logarithmic factors.

\paragraph{Notations and facts of \rmgs and empirical \rmgs.}
First, recall that for any robust Markov game  $\mathcal{MG}_{\mathsf{rob}} = \big\{ \cS, \{\cA_i\}_{1 \le i \le n},\{\cU^{\ror_i}_\rho(P^\no)\}_{1 \le i \leq n}, \rew,  H \big\}$, according to robust Bellman equations in \eqref{eq:bellman-consistency-sa}, one has for any joint policy $\pi: \cS\times [H] \rightarrow \Delta(\cA)$ and any $(h,i,s,\ba) \in [H] \times [n] \times \cS\times \cA$:
\begin{align}
	Q_{i,h}^{\pi, \ror_i}(s,\ba) &= r_{i,h}(s, \ba) +  \inf_{P\in \unb^{\ror_i}_\rho(P^{\no}_{h,s,\ba})} P V_{i,h+1}^{\pi,\ror_i}, 
	\qquad \text{where } V_{i,h}^{\pi,\ror_i}(s) = \mathbb{E}_{a\sim \pi_h(s)}[Q^{\pi.\ror_i}_{i,h}(s,\ba)]. \label{eq:Q-value-mg}
\end{align}

Combined with the matrix notation in Appendix~\ref{sec:notation}, we arrive at
\begin{align}
V_{i,h}^{\pi,\ror_i} = r_{i,h}^{\pi} + \Pi^\pi_h \inf_{P\in \unb^{\ror_i}_\rho(P^{\no}_{h})} P V_{i,h+1}^{\pi,\ror_i} = r_{i,h}^{\pi} +\Pv_{i,h}^{\pi, V} V_{i,h+1}^{\pi,\ror_i}. \label{eq:Q-value-mg-matrix}
\end{align}

Then we denote the empirical robust Markov games based on the estimated nominal distribution $\widehat{P}^\no$ constructed in \eqref{eq:empirical-P-infinite} as $\widehat{\mathcal{MG}}_{\mathsf{rob}} = \big\{ \cS, \{\cA_i\}_{1 \le i \le n},\{\cU^{\ror_i}_\rho(\widehat{P}^\no)\}_{1 \le i \le n}, \rew,  H \big\}$. Analogous to \eqref{eq:value-function-defn-robust}, we can define the corresponding robust value function (resp.~robust Q-function) of any joint policy $\pi$ in $\widehat{\mathcal{MG}}_{\mathsf{rob}}$ as
$\big\{\widehat{V}_{i,h}^{\pi,\ror_i} \big\}_{1\leq i\leq n}$ (resp.~$\big\{\widehat{Q}_{i,h}^{\pi,\ror_i} \big\}_{1\leq i\leq n}$). In addition, similar to \eqref{eq:defn-optimal-V}, we can define the maximum of the robust value function for each agent over $\widehat{\mathcal{MG}}_{\mathsf{rob}}$ as follows :
\begin{align}
	\forall s\in\cS: \quad \widehat{V}_{i,h}^{\star,\pi_{-i},\ror_i}(s)\coloneqq \max_{\pi'_i: \cS \times [H] \rightarrow \Delta(\mathcal{A}_i)} \widehat{V}_{i,h}^{\pi'_i \times \pi_{-i},\ror_i}(s) = \max_{\pi'_i: \cS \times [H] \rightarrow \Delta(\mathcal{A}_i)} \inf_{P\in \unb^{\ror_i}(\widehat{P}^{\no})} \widehat{V}_{i,h}^{\pi_i' \times \pi_{-i},P} (s),
\end{align}
which can be achieved by at least one {\em robust best-response} policy for all $s\in\cS$ simultaneously \citep[Section A.2]{blanchet2024double}.

Moreover, applying the robust Bellman equation in \eqref{eq:bellman-consistency-sa} for the empirical \rmg $\widehat{\mathcal{MG}}_{\mathsf{rob}}$, for any joint policy $\pi$, 
\begin{align}
	\widehat{Q}_{i,h}^{\pi,\ror_i}(s,\ba) &= r_{i,h}(s, \ba) +  \inf_{P\in \unb^{\ror_i}_\rho(\widehat{P}^{\no}_{h,s,\ba})} P \widehat{V}_{i,h+1}^{\pi,\ror_i}, 
	\qquad \text{where } \widehat{V}_{i,h}^{\pi,\ror_i}(s) = \mathbb{E}_{a\sim \pi_h(s)}[\widehat{Q}^{\pi,\ror_i}_{i,h}(s,\ba)], \label{eq:bellman-equ-star-infinite-estimate}
\end{align}
which combined with the matrix notations in Appendix~\ref{sec:notation} leads to the matrix form of the robust Bellman equation:
\begin{align}
	\widehat{V}_{i,h}^{\pi,\ror_i} = r_{i,h}^{\pi} + \Pi^\pi_h \inf_{P\in \unb^{\ror_i}(\widehat{P}^{\no}_{h})} P \widehat{V}_{i,h+1}^{\pi,\ror_i} = r_{i,h}^{\pi} +\Phatv_{i,h}^{\pi, \widehat{V}} \widehat{V}_{i,h+1}^{\pi,\ror_i}. \label{eq:r-bellman-matrix}
\end{align}

Encouragingly, the above property of the robust Bellman equations ensure that the policy $\widehat{\pi}$ output by the proposed method \DRNVI (cf.~Algorithm~\ref{alg:nash-dro-vi-finite}) is a robust-$\{\text{NE},\text{CE}, \text{CCE}\}$ of the empirical \rmg $\widehat{\mathcal{MG}}_{\mathsf{rob}}$ when executing different corresponding subroutines, summarized in the following lemma:
\begin{lemma}\label{lem:algo-output-equilibrium}
The output policy $\widehat{\pi}$ by \DRNVI (cf.~Algorithm~\ref{alg:nash-dro-vi-finite}) is a robust-$\{\text{NE}, \text{CE}, \text{CCE}\}$ of the empirical \rmg $\widehat{\mathcal{MG}}_{\mathsf{rob}} = \big\{ \cS, \{\cA_i\}_{1 \le i \le n},\{\cU^{\ror_i}_\rho(\widehat{P}^\no)\}_{1 \le i \le n}, \rew,  H \big\}$ when executing different subroutine $\mathsf{Equilibrium} \in \mathsf{Compute-}\{ \mathsf{Nash}, \mathsf{CE}, \mathsf{CCE}\}$ accordingly, namely
	\begin{align}
	 \forall (i,h)\in [n]\times [H]: \quad \begin{cases}
  \widehat{V}_{i,h}  =\widehat{V}^{\widehat{\pi},\ror_i}_{i,h} = \widehat{V}^{\star,\widehat{\pi}_{-i},\ror_i}_{i,h} & \text{when }  \mathsf{Equilibrium} =\mathsf{Compute}-\mathsf{Nash} \\
   \widehat{V}_{i,h} =\widehat{V}^{\widehat{\pi},\ror_i}_{i,h} \geq \widehat{V}^{\star,\widehat{\pi}_{-i},\ror_i}_{i,h} & \text{when }  \mathsf{Equilibrium} =\mathsf{Compute}-\mathsf{CCE} \\
   \widehat{V}_{i,h} =  \widehat{V}^{\widehat{\pi},\ror_i}_{i,h} \geq  \max_{f_i \in \cF_i} V_{i,h}^{f_i \diamond \widehat{\pi}, \ror_i } & \text{when }  \mathsf{Equilibrium} =\mathsf{Compute}-\mathsf{CE}.
  \end{cases} \label{eq:iterative-solution-confirm}
	\end{align}
\end{lemma}
\begin{proof}
See Appendix~\ref{proof:lem:algo-output-equilibrium}.
\end{proof}

\section{Proof of Theorem~\ref{thm:robust-mg-upper-bound}}\label{proof:upper-bound}

Before starting, let us introduce an essential lemma that characterize the difference between robust MGs and standard MGs. For each agent, the possible range of the robust value function shrinks as the uncertainty level $\ror_i$ of its own uncertainty set increases, shown below.

\begin{lemma}\label{lemma:pnorm-key-value-range} 
Consider the uncertainty set $\cU^{\ror_i}(\cdot) = \cU^{\ror_i}_{\rho_{\mathsf{TV}}}(\cdot)$ and any robust Markov game $\mathcal{MG}_{\mathsf{rob}} = \big\{ \cS, \{\cA_i\}_{1 \le i \le n},\{\cU^{\ror_i}(P)\}_{1 \le i \leq n}, \rew,  H \big\}$. The robust value function $\{ V_{i,h}^{\pi,\ror_i}\}_{i\in[n], h\in[H]}$ associated with any joint policy $\pi$  satisfies: 
	\begin{align*}
\forall (i,h)\in [n] \times [H]: \quad 	\max_{s\in\cS}V_{i,h}^{\pi,\ror_i}(s) - \min_{s\in\cS}V_{i,h}^{\pi,\ror_i}(s)& \leq \min \left\{\frac{1}{\ror_i}, H-h+1 \right\}.
	\end{align*}
	\end{lemma}
\begin{proof}
See Appendix~\ref{proof:lemma:pnorm-key-value-range}
\end{proof}

Equipped with the preceding lemma, we are now prepared to prove Theorem~\ref{thm:robust-mg-upper-bound} for three different robust solution concepts, respectively.

\subsection{Proof of learning robust NE/robust CCE}\label{sec:proof-nash-cce}

In this subsection, we focus on the two equilibrium concepts --- robust NE and robust CCE.
The proof is separated into several key steps as below.
\paragraph{Step 1: decomposing the error.}
Before proceeding, recall the goal is to prove that the output policy $\widehat{\pi}$ from Algorithm~\ref{alg:nash-dro-vi-finite} is an $\varepsilon$-robust NE/CCE with corresponding subroutine (cf.~line~\ref{eq:nash-subroutine}). Namely,  $\widehat{\pi} \in \Delta(\cA_1) \times \Delta(\cA_2) \times \Delta(\cA_n) $ is a product policy satisfies
\begin{align}
	\mathsf{gap}_{\mathsf{NE}}(\widehat{\pi}) \defn \max_{s\in\cS, i\in[n]} \left\{ V_{i,1}^{\star,\widehat{\pi}_{-i},\ror_i}(s) - V_{i,1}^{\widehat{\pi},\ror_i}(s)  \right\} \leq \varepsilon
\end{align}
or $\widehat{\pi} \in \Delta(\cA)$ is a (possibly correlated) policy obeys 
\begin{align}
	\mathsf{gap}_{\mathsf{CCE}}(\widehat{\pi}) \defn \max_{s\in\cS i\in[n] } \left\{ V_{i,1}^{\star,\widehat{\pi}_{-i},\ror_i}(s) - V_{i,1}^{\widehat{\pi},\ror_i}(s)  \right\} \leq \varepsilon.
\end{align}
We note that $\mathsf{gap}_{\mathsf{NE}}$ and $\mathsf{gap}_{\mathsf{CCE}}$ exhibit similar properties, differing only in the feasible set of policy $\widehat{\pi}$. So we consider them together.

To continue, we introduce the following best-response policy of the $i$-th player given other players policy $\widehat{\pi}_{-i}$:
\begin{align}
	\widetilde{\pi}_i^\star = \{\widetilde{\pi}_{i,h}^\star \}_{1\leq h\leq H} = \textrm{argmax}_{\pi_i'\in \cS \times [H] \rightarrow \Delta(\cA_i)} V_{i,1}^{\pi_i'\times \widehat{\pi}_{-i},\ror_i},
\end{align}
which indicates that 
\begin{align}
V_{i,1}^{\widetilde{\pi}_i^\star \times \widehat{\pi}_{-i},\ror_i} = V_{i,1}^{\star,\widehat{\pi}_{-i},\ror_i}.
\end{align}
Armed with above notations and facts, the term of interest $ V_{i,1}^{\star,\widehat{\pi}_{-i}, \ror_i} - V_{i,1}^{\widehat{\pi}, \ror_i} $ for any $i\in [n]$ can be decomposed as
\begin{align}
	 V_{i,1}^{\star,\widehat{\pi}_{-i},\ror_i} - V_{i,1}^{\widehat{\pi},\ror_i} 
	& = \left(V_{i,1}^{\star,\widehat{\pi}_{-i},\ror_i} - \widehat{V}_{i,1}^{\widetilde{\pi}_i^\star \times\widehat{\pi}_{-i},\ror_i}\right) + \left(\widehat{V}_{i,1}^{\widetilde{\pi}_i^\star\times\widehat{\pi}_{-i},\ror_i} - \widehat{V}_{i,1}^{\widehat{\pi},\ror_i} \right) + \left(\widehat{V}_{i,1}^{\widehat{\pi},\ror_i} - V_{i,1}^{\widehat{\pi},\ror_i} \right) \notag \\
	&  \overset{\mathrm{(i)}}{\leq}  \left(V_{i,1}^{\star,\widehat{\pi}_{-i},\ror_i} - \widehat{V}_{i,1}^{\widetilde{\pi}_i^\star\times\widehat{\pi}_{-i},\ror_i}\right) + \left(\widehat{V}_{i,1}^{\widetilde{\pi}_i^\star\times\widehat{\pi}_{-i},\ror_i} - \widehat{V}_{i,1}^{\star, \widehat{\pi}_{-i},\ror_i} \right) + \left(\widehat{V}_{i,1}^{\widehat{\pi},\ror_i} - V_{i,1}^{\widehat{\pi},\ror_i} \right) \notag \\
	&\leq \left(V_{i,1}^{\star,\widehat{\pi}_{-i},\ror_i} - \widehat{V}_{i,1}^{\widetilde{\pi}_i^\star\times\widehat{\pi}_{-i},\ror_i}\right) + \left(\widehat{V}_{i,1}^{\widehat{\pi},\ror_i} - V_{i,1}^{\widehat{\pi},\ror_i} \right)  \label{eq:l1-decompose}
\end{align}
where (i) holds by $\widehat{V}_{i,1}^{\widehat{\pi},\ror_i} = \widehat{V}_{i,1}^{\star, \widehat{\pi}_{-i},\ror_i}$ (resp.~$\widehat{V}_{i,1}^{\widehat{\pi},\ror_i}  \geq \widehat{V}_{i,1}^{\star, \widehat{\pi}_{-i},\ror_i}$) when the subroutine in line~\ref{eq:nash-subroutine} is $\mathsf{Compute}-\mathsf{Nash}$ (resp.~$\mathsf{Compute}-\mathsf{CCE}$) implied by Lemma~\ref{lem:algo-output-equilibrium}, and the last inequality follows from $\widehat{V}_{i,1}^{\widetilde{\pi}_i^\star\times\widehat{\pi}_{-i},\ror_i} \leq \max_{\pi_i'\in \cS \times [H] \rightarrow \Delta(\cA_i)} \widehat{V}_{i,1}^{\pi_i'\times \widehat{\pi}_{-i},\ror_i} = \widehat{V}_{i,1}^{\star, \widehat{\pi}_{-i},\ror_i}$ by definition.

\paragraph{Step 2: developing the recursion. }
We consider a more general form for any time step $h\in[H]$ and any joint policy $\pi$. Towards this, one has 
\begin{align}
V_{i,h}^{\pi,\ror_i} - \widehat{V}_{i,h}^{\pi, \ror_i} & \overset{\mathrm{(i)}}{=}  r_{i,h}^{\pi} + \Pi^\pi_h \inf_{P\in \unb^{\ror_i}(P^{\no}_{h,s,\ba})} P V_{i,h+1}^{\pi,\ror_i} - \Big(r_{i,h}^{\pi} + \Pi^\pi_h \inf_{P\in \unb^{\ror_i}(\widehat{P}^{\no}_{h,s,\ba})} P \widehat{V}_{i,h+1}^{\pi, \ror_i} \Big) \notag \\
& \overset{\mathrm{(ii)}}{=}  \Pv_{i,h}^{\pi, V} V_{i,h+1}^{\pi,\ror_i} - \Phatv_{i,h}^{\pi, \widehat{V}}\widehat{V}_{i,h+1}^{\pi, \ror_i} \label{eq:recursive-dro0} \\
& = \left( \Pv_{i,h}^{\pi, V} V_{i,h+1}^{\pi,\ror_i} -  \Pv_{i,h}^{\pi, \widehat{V}} \widehat{V}_{i,h+1}^{\pi, \ror_i} \right) + \left(\Pv_{i,h}^{\pi, \widehat{V}} \widehat{V}_{i,h+1}^{\pi, \ror_i} - \Phatv_{i,h}^{\pi, \widehat{V}}\widehat{V}_{i,h+1}^{\pi, \ror_i } \right ) \notag \\
& \overset{\mathrm{(iii)}}{\leq} \Pv_{i,h}^{\pi, \widehat{V}} \left(V_{i,h+1}^{\pi,\ror_i} - \widehat{V}_{i,h+1}^{\pi, \ror_i} \right) + \underbrace{\left|\Pv_{i,h}^{\pi, \widehat{V}} \widehat{V}_{i,h+1}^{\pi, \ror_i} - \Phatv_{i,h}^{\pi, \widehat{V}}\widehat{V}_{i,h+1}^{\pi, \ror_i} \right|}_{=: a^{\pi}_{i,h}} \label{eq:recursive-dro}
\end{align}
where (i) and (ii) hold by the matrix version of robust Bellman consistency equations in \eqref{eq:Q-value-mg-matrix} and  \eqref{eq:r-bellman-matrix}, and (iii) follows from the observation
\begin{align*}
	  \Pv_{i,h}^{\pi, V} V_{i,h+1}^{\pi,\ror_i}
	&\leq \Pv_{i,h}^{\pi, \widehat{V}} V_{i,h+1}^{\pi,\ror_i}
\end{align*}
due to the definition of $  \Pv_{i,h}^{\pi, V}  = \Pi^\pi_h \arg\min_{P\in \unb^{\ror_i}(P^{\no}_{h,s,\ba})} P V_{i,h+1}^{\pi,\ror_i} \leq \Pi^\pi_h \arg\min_{P\in \unb^{\ror_i}(P^{\no}_{h,s,\ba})} P \widehat{V}_{i,h+1}^{\pi,\ror_i}$ (cf.~\eqref{eq:inf-p-special-marl} and \eqref{eqn:ppivq-marl}). 

Recursively applying \eqref{eq:recursive-dro} leads to
\begin{align}
&V_{i,h}^{\pi,\ror_i} - \widehat{V}_{i,h}^{\pi,\ror_i} \notag \\
&  \leq \Pv_{i,h}^{\pi, \widehat{V}} \Pv_{i,h+1}^{\pi, \widehat{V}}  \left(V_{i,h+2}^{\pi,\ror_i} - \widehat{V}_{i,h+2}^{\pi, \ror_i} \right)  + \Pv_{i,h}^{\pi, \widehat{V}} \left|\Pv_{i,h+1}^{\pi, \widehat{V}} \widehat{V}_{i,h+2}^{\pi, \ror_i} - \Phatv_{i,h+1}^{\pi, \widehat{V}}\widehat{V}_{i,h+2}^{\pi,\ror_i} \right| + \left|\Pv_{i,h}^{\pi, \widehat{V}} \widehat{V}_{i,h+1}^{\pi, \ror_i} - \Phatv_{i,h}^{\pi, \widehat{V}}\widehat{V}_{i,h+1}^{\pi, \ror_i} \right|\notag \\
& \leq \cdots \leq \sum_{j=h}^H \left(\prod_{k=h}^{j-1}\Pv_{i,k}^{\pi, \widehat{V}}\right) a^{\pi}_{i,j}, \label{eq:recursive-equation}
\end{align}
where the last inequality holds by adopting the following notations
\begin{align}
	\left(\prod_{k=h}^{h-1}\Pv_{i,k}^{\pi, \widehat{V}}\right) = I \quad \text{and} \quad \left(\prod_{k=h}^{j-1}\Pv_{i,k}^{\pi, \widehat{V}}\right) = \Pv_{i,h}^{\pi, \widehat{V}} \cdot \Pv_{i,h+1}^{\pi, \widehat{V}} \cdots \Pv_{i,j-1}^{\pi, \widehat{V}}.
\end{align}

Next, similar to \eqref{eq:recursive-dro}, we can achieve
\begin{align}
\widehat{V}_{i,h}^{\pi, \ror_i} - V_{i,h}^{\pi,\ror_i} & \overset{\mathrm{(i)}}{=}   \Phatv_{i,h}^{\pi, \widehat{V}}\widehat{V}_{i,h+1}^{\pi, \ror_i} - \Pv_{i,h}^{\pi, V} V_{i,h+1}^{\pi,\ror_i}\notag \\
& = \left(  \Phatv_{i,h}^{\pi, \widehat{V}}\widehat{V}_{i,h+1}^{\pi, \ror_i} -  \Pv_{i,h}^{\pi, \widehat{V}}\widehat{V}_{i,h+1}^{\pi, \ror_i}  \right) + \left( \Pv_{i,h}^{\pi, \widehat{V}}\widehat{V}_{i,h+1}^{\pi, \ror_i} - \Pv_{i,h}^{\pi, V} V_{i,h+1}^{\pi,\ror_i} \right ) \notag \\
& \leq \Pv_{i,h}^{\pi, V} \left(\widehat{V}_{i,h+1}^{\pi, \ror_i} - V_{i,h+1}^{\pi,\ror_i} \right) + \left|\Pv_{i,h}^{\pi, \widehat{V}} \widehat{V}_{i,h+1}^{\pi, \ror_i} - \Phatv_{i,h}^{\pi, \widehat{V}}\widehat{V}_{i,h+1}^{\pi, \ror_i} \right| \label{eq:recursive-dro2}
\end{align}
where (i) holds by \eqref{eq:recursive-dro0}, and the last inequality follows from the fact $\Pv_{i,h}^{\pi, \widehat{V}} \widehat{V}_{i,h+1}^{\pi}
	\leq \Pv_{i,h}^{\pi, V}\widehat{V}_{i,h+1}^{\pi}$ (see the definition of $\Pv_{i,h}^{\pi, \widehat{V}}$, i.e., \eqref{eq:inf-p-special-marl} and \eqref{eqn:ppivq-marl}).

Then following the routine of achieving \eqref{eq:recursive-equation}, we arrive at
\begin{align}
&\widehat{V}_{i,h}^{\pi, \ror_i} - V_{i,h}^{\pi,\ror_i} \leq \sum_{j=h}^H \left(\prod_{k=h}^{j-1}\Pv_{i,k}^{\pi, V}\right) a^{\pi}_{i,j}. \label{eq:recursive-equation2}
\end{align}
Summing up  \eqref{eq:recursive-equation} and \eqref{eq:recursive-equation2}, one has for any joint policy $\pi$,
\begin{align}
\left|\widehat{V}_{i,h}^{\pi, \ror_i} - V_{i,h}^{\pi,\ror_i} \right| &\leq \max\{V_{i,h}^{\pi,\ror_i} - \widehat{V}_{i,h}^{\pi, \ror_i} , \widehat{V}_{i,h}^{\pi, \ror_i} - V_{i,h}^{\pi,\ror_i} \} \notag \\
&\leq  \max \left\{ \sum_{j=h}^H \left(\prod_{k=h}^{j-1}\Pv_{i,k}^{\pi, \widehat{V}}\right) a^{\pi}_{i,j} , \sum_{j=h}^H \left(\prod_{k=h}^{j-1}\Pv_{i,k}^{\pi, V}\right) a^{\pi}_{i,j}\right\}, \label{eq:summary-control-term}
\end{align}
where the $\max$ operator is taken entry-wise for the vectors.

To continue, we introduce an important concentration result about the value estimation error as follows:
\begin{lemma}\label{lemma:tv-dro-b-bound-star-marl}
Consider any $\delta \in (0,1)$. With probability at least $1- \delta$, one has for any joint policy $\pi$,
\begin{align}\label{eq:tv-dro-b-bound-star}
	\forall (h,i)\in [H] \times [n]: \quad a_{i,h}^\pi &=\left|\Pv_{i,h}^{\pi, \widehat{V}} \widehat{V}_{i,h+1}^{\pi, \ror_i} - \Phatv_{i,h}^{\pi, \widehat{V}}\widehat{V}_{i,h+1}^{\pi, \ror_i} \right| \notag \\
  &\leq  2\sqrt{\frac{\log(\frac{18S \allA nHN}{\delta})}{N}}  \sqrt{\mathsf{Var}_{\Pv_{h}^{\pi}}(\widehat{V}_{i,h+1}^{\pi})} +  \frac{\log(\frac{18S\allA nHN}{\delta}) H}{N} 1  \notag \\
  &\leq 3\sqrt{\frac{H^2\log(\frac{18S\allA nHN}{\delta})}{N}}1
\end{align}
where $\mathsf{Var}_{\Pv_{h}^{\pi}}(\cdot)$ is defined in \eqref{eq:defn-variance-vector-marl}.
\end{lemma}
\begin{proof}
See Appendix~\ref{proof:lemma:tv-dro-b-bound-star-marl}.
\end{proof}

\paragraph{Step 3: controlling the first term in \eqref{eq:summary-control-term}.}
Let us introduce some additional notations for convenience. Recall $e_s$ denote a $S$-dimensional standard basis supported on the $s$-th element. We denote
\begin{align}
	d_{h}^h = e_s \quad \text{and} \quad d_h^j = e_s^\top \left(\prod_{k=h}^{j-1}\Pv_{i,k}^{\pi, \widehat{V}}\right) \quad \forall j =h+1, \cdots, H. \label{eq:defn-of-d}
\end{align}

Armed with above notations and facts, for any $s\in\cS$, we have
\begin{align}
	V_{i,h}^{\pi,\ror_i}(s) - \widehat{V}_{i,h}^{\pi,\ror_i}(s) &= \left< e_s, V_{i,h}^{\pi,\ror_i} - \widehat{V}_{i,h}^{\pi, \ror_i} \right> = \sum_{j=h}^H \left< d_h^j, a^{\pi}_{i,j} \right> \notag \\
	& \leq \sum_{j=h}^H \left< d_h^j, \left(2\sqrt{\frac{\log(\frac{18S \allA nHN}{\delta})}{N}} \sqrt{\mathsf{Var}_{\Pv_{j}^{\pi}}(\widehat{V}_{i,j+1}^{\pi, \ror_i})} +  \frac{\log(\frac{18S\allA nHN}{\delta}) H}{N} 1\right) \right> \notag \\
	&\leq \frac{\log(\frac{18S\allA nHN}{\delta}) H^2}{N} + 2\sqrt{\frac{\log(\frac{18S \allA nHN}{\delta})}{N}} \sum_{j=h}^H \left< d_h^j, \sqrt{\mathsf{Var}_{\Pv_{j}^{\pi}}(\widehat{V}_{i,j+1}^{\pi, \ror_i})} \right> \notag \\
	&\overset{\mathsf{(i)}}{\leq}  \frac{\log(\frac{18S\allA nHN}{\delta}) H^2}{N} + 2\sqrt{\frac{\log(\frac{18S \allA nHN}{\delta})}{N}}  \sqrt{H\sum_{j=h}^H \left< d_h^j, \mathsf{Var}_{\Pv_{j}^{\pi}}(\widehat{V}_{i,j+1}^{\pi, \ror_i}) \right> } \notag \\
	& \leq \frac{\log(\frac{18S\allA nHN}{\delta}) H^2}{N} + \underbrace{2\sqrt{\frac{H\log(\frac{18S \allA nHN}{\delta})}{N}}  \sqrt{\sum_{j=h}^H \left< d_h^j, \mathsf{Var}_{\Pv_{i,j}^{\pi, \widehat{V}}}(\widehat{V}_{i,j+1}^{\pi, \ror_i}) \right> } }_{=: \cB_1}  \notag \\
	& + \underbrace{2\sqrt{\frac{\log(\frac{18S \allA nHN}{\delta})}{N}} \sqrt{H\sum_{j=h}^H \left< d_h^j, \left| \mathsf{Var}_{\Pv_{j}^{\pi}}(\widehat{V}_{i,j+1}^{\pi, \ror_i}) - \mathsf{Var}_{\Pv_{i,j}^{\pi, \widehat{V}}}(\widehat{V}_{i,j+1}^{\pi, \ror_i})\right| \right>} }_{=: \cB_2} \label{eq:key-concentration-bound2}
\end{align}
where (i) holds by the Cauchy-Schwarz inequality.

Then we control the two main terms in \eqref{eq:key-concentration-bound2} separately.
\begin{itemize}
	\item {\bf Controlling $\cB_1$.}
To begin with, we introduce the following lemma about $\sum_{j=h}^H \left< d_h^j, \mathsf{Var}_{\Pv_{i,j}^{\pi, \widehat{V}}}(\widehat{V}_{i,j+1}^{\pi, \ror_i})  \right>$ whose proof is postponed to Appendix~\ref{proof:lem:key-lemma-reduce-H}.

\begin{lemma}\label{lem:key-lemma-reduce-H}
Consider any $\delta\in(0,1)$. With probability at least $1-\delta$, one has for any joint policy $\pi$,
\begin{align}
 \forall (h,i)\in [H] \times [n]:\quad  &\sum_{j=h}^H \left< d_h^j, \mathsf{Var}_{\Pv_{i,j}^{\pi, \widehat{V}}}(\widehat{V}_{i,j+1}^{\pi, \ror_i}) \right>  \notag\\
 & \leq 3H \left(\max_{s\in\cS}\widehat{V}_{i,j+1}^{\pi, \ror_i}(s) - \min_{s\in\cS}\widehat{V}_{i,j+1}^{\pi, \ror_i}(s)\right) \left(1 + 2H\sqrt{\frac{\log(\frac{18S\allA nHN}{\delta})}{N}} \right).
\end{align}
\end{lemma}

Applying Lemma~\ref{lem:key-lemma-reduce-H} to $\cB_1$ in \eqref{eq:key-concentration-bound2}, we arrive at
\begin{align}
\cB_1 & =2\sqrt{\frac{H\log \left(\frac{18S \allA nHN}{\delta} \right)}{N}}  \sqrt{\sum_{j=h}^H \left< d_h^j, \mathsf{Var}_{\Pv_{i,j}^{\pi, \widehat{V}}}(\widehat{V}_{i,j+1}^{\pi, \ror_i}) \right> } \notag \\
& \leq 2\sqrt{\frac{H\log \left(\frac{18S \allA nHN}{\delta} \right)}{N}} \sqrt{3H \left(\max_{s\in\cS}\widehat{V}_{i,j+1}^{\pi, \ror_i}(s) - \min_{s\in\cS}\widehat{V}_{i,j+1}^{\pi, \ror_i}(s)\right) \Bigg(1 + 2H\sqrt{\frac{\log \left(\frac{18S\allA nHN}{\delta} \right)}{N}} \Bigg)} \notag \\
&  \overset{\mathrm{(i)}}{\leq} 2 \sqrt{  \frac{3H^2\log \left(\frac{18S \allA nHN}{\delta} \right)}{N} \min \left\{\frac{1}{\ror_i}, H-h+1 \right\} \Bigg(1 + 2H\sqrt{\frac{\log \left(\frac{18S\allA nHN}{\delta} \right)}{N}} \Bigg) } \notag \\
& \leq 6\sqrt{  \frac{H^2\min \left\{1/\ror_i,H\right\}  \log \left(\frac{18S \allA nHN}{\delta} \right)}{N}  }, \label{eq:key-concentration-bound2-solve1}
\end{align}
where (i) holds by applying Lemma~\ref{proof:lemma:pnorm-key-value-range}, and the last inequality follows by taking $N \geq 4H^2\log \big(\frac{18S \allA nHN}{\delta} \big)$.

\item {\bf Controlling  $\cB_2$.}
We introduce another lemma; refer to the proof in Appendix~\ref{proof:eq:extra-lemma1}.
\begin{lemma}\label{eq:extra-lemma1}
Consider the standard \rmg $\mathcal{MG} = \big\{ \cS, \{\cA_i\}_{1 \le i \le n},\{\cU^{\ror_i}(P^\no)\}_{1 \le i \leq n}, \rew,  H \big\}$ and empirical \rmg $\mathcal{MG}_{\mathsf{rob}} = \big\{ \cS, \{\cA_i\}_{1 \le i \le n},\{\cU^{\ror_i}(\widehat{P}^\no)\}_{1 \le i \leq n}, \rew,  H \big\}$. Considering any joint policy $\pi$,  any transition kernel $P'\in\mathbb{R}^S$ and any $\widetilde{P} \in\mathbb{R}^S$ obeying $\widetilde{P}\in \cU^{\ror_i}(P)$, one has
\begin{subequations}
\begin{align}
\forall (i,j)\in[n]\times [H]: \quad &\left| \mathsf{Var}_{P'}(\widehat{V}_{i,j+1}^{\pi, \ror_i}) - \mathsf{Var}_{\widetilde{P}}(\widehat{V}_{i,j+1}^{\pi, \ror_i})\right| \leq    \min \left\{\frac{1}{\ror_i}, H-h+1 \right\}, \label{eq:lemma-6-line1} \\
& \left| \mathsf{Var}_{P'}(V_{i,j+1}^{\pi, \ror_i}) - \mathsf{Var}_{\widetilde{P}}(V_{i,j+1}^{\pi, \ror_i})\right| \leq    \min \left\{\frac{1}{\ror_i}, H-h+1 \right\}. \label{eq:lemma-6-line2}
\end{align}
\end{subequations}
\end{lemma}

Armed with above lemma, we observe that
\begin{align}
	\left| \mathsf{Var}_{\Pv_{j}^{\pi}}(\widehat{V}_{i,j+1}^{\pi, \ror_i}) - \mathsf{Var}_{\Pv_{i,j}^{\pi, \widehat{V}}}(\widehat{V}_{i,j+1}^{\pi, \ror_i})\right| & \overset{\mathrm{(i)}}{=} \left| \Pi_j^{\pi} \left(\mathsf{Var}_{P^\no_j}(\widehat{V}_{i,j+1}^{\pi, \ror_i}) - \mathsf{Var}_{P_{i,j}^{\pi, \widehat{V}}}(\widehat{V}_{i,j+1}^{\pi, \ror_i}) \right)\right|    \notag \\
	& \overset{\mathrm{(ii)}}{\leq} \left\| \mathsf{Var}_{P^\no_j}(\widehat{V}_{i,j+1}^{\pi, \ror_i}) - \mathsf{Var}_{P_{i,j}^{\pi, \widehat{V}}}(\widehat{V}_{i,j+1}^{\pi, \ror_i}) \right\|_\infty 1    \notag \\
	& \leq \min \left\{\frac{1}{\ror_i}, H-h+1 \right\} 1, \label{eq:tv-first-C2}
\end{align}
where (i) and (ii) follows from the matrix notations $\Pi_j^\pi$ (cf~\eqref{eqn:bigpi}) and $\Pv_{j}^{\pi}, \Pv_{i,j}^{\pi, \widehat{V}}$ (cf~\eqref{eqn:ppivq-marl}), and the last inequality holds by applying Lemma~\ref{eq:extra-lemma1} with $P' = P^\no_{j,s,\ba}, \widetilde{P} = P_{i,j,s,\ba}^{\pi, \widehat{V}}$ for all $(s,\ba)\in \cS \times \cA$.

Plugging back \eqref{eq:tv-first-C2} to \eqref{eq:key-concentration-bound2}, it can be verified that
\begin{align}
\cB_2 &= 2\sqrt{\frac{\log(\frac{18S \allA nHN}{\delta})}{N}} \sqrt{H\sum_{j=h}^H \left< d_h^j, \left| \mathsf{Var}_{\Pv_{j}^{\pi}}(\widehat{V}_{i,j+1}^{\pi, \ror_i}) - \mathsf{Var}_{\Pv_{i,j}^{\pi, \widehat{V}}}(\widehat{V}_{i,j+1}^{\pi, \ror_i})\right| \right>} \notag \\
& \leq 2\sqrt{\frac{H\log(\frac{18S \allA nHN}{\delta})}{N}} \sqrt{ \sum_{j=h}^H \left< d_h^j, \min \left\{\frac{1}{\ror_i}, H-h+1 \right\} 1 \right> } \notag \\
& \leq 2\sqrt{\frac{H^2\min \left\{1/\ror_i, H \right\}\log(\frac{18S \allA nHN}{\delta})}{N}}. \label{eq:key-concentration-bound2-solve2}
\end{align}

\end{itemize}

Consequently, combining \eqref{eq:key-concentration-bound2-solve1} and \eqref{eq:key-concentration-bound2-solve2}, \eqref{eq:key-concentration-bound2} can be bounded by
\begin{align}
V_{i,h}^{\pi,\ror_i}(s) - \widehat{V}_{i,h}^{\pi, \ror_i}(s) &\leq \frac{\log(\frac{18S\allA nHN}{\delta}) H^2}{N} + 6\sqrt{  \frac{H^2\min \left\{1/\ror_i,H\right\}  \log \left(\frac{18S \allA nHN}{\delta} \right)}{N}  } \notag \\
	& + 2\sqrt{\frac{H^2\min \left\{1/\ror_i, H \right\}\log(\frac{18S \allA nHN}{\delta})}{N}} \notag \\
	&\leq 9\sqrt{\frac{H^2\min \left\{1/\ror_i, H \right\}\log(\frac{18S \allA nHN}{\delta})}{N}}, \label{eq:upper-final1}
\end{align}
where the last inequality holds by taking $N \geq 4H^2\log(\frac{18S \allA nHN}{\delta})$.

\paragraph{Step 4: controlling the second term in \eqref{eq:summary-control-term}.}

To do so, similar to \eqref{eq:defn-of-d}, we define 
\begin{align}
	w_{h}^h = e_s \quad \text{and} \quad w_h^j = e_s^\top \left(\prod_{k=h}^{j-1}\Pv_{i,k}^{\pi, V}\right) \quad \forall j =h+1, \cdots, H. \label{eq:defn-of-d-2}
\end{align}

With the above notations in mind, following the routine of \eqref{eq:key-concentration-bound2} gives: for any $s\in\cS$,
\begin{align}
  &\widehat{V}_{i,h}^{\pi, \ror_i}(s) - V_{i,h}^{\pi,\ror_i}(s)  \notag \\
  &\leq  \frac{\log(\frac{18S\allA nHN}{\delta}) H^2}{N} + 2\sqrt{\frac{\log(\frac{18S \allA nHN}{\delta})}{N}} \sum_{j=h}^H \left< w_h^j, \sqrt{\mathsf{Var}_{\Pv_{j}^{\pi}}(\widehat{V}_{i,j+1}^{\pi, \ror_i})} \right> \notag \\
  &  \overset{\mathrm{(i)}}{\leq}   \frac{\log(\frac{18S\allA nHN}{\delta}) H^2}{N} + 2\sqrt{\frac{\log(\frac{18S \allA nHN}{\delta})}{N}} \sum_{j=h}^H \Big< w_h^j, \notag \\
  & \quad \Big( \sqrt{ \big| \mathsf{Var}_{\Pv_{j}^{\pi}}(\widehat{V}_{i,j+1}^{\pi, \ror_i} - V_{i,j+1}^{\pi, \ror_i})\big|} + \sqrt{\big| \mathsf{Var}_{\Pv_{j}^{\pi}}(V_{i,j+1}^{\pi, \ror_i}) - \mathsf{Var}_{\Pv_{i,j}^{\pi, V}}(V_{i,j+1}^{\pi, \ror_i})\big|} + \sqrt{\mathsf{Var}_{\Pv_{i,j}^{\pi, V}}(V_{i,j+1}^{\pi, \ror_i})} \Big) \Big> \notag \\
  &\leq  \frac{ H^2\log \left(\frac{18S\allA nHN}{\delta} \right)}{N} + \underbrace{2\sqrt{\frac{H\log(\frac{18S \allA nHN}{\delta})}{N}}  \sqrt{\sum_{j=h}^H \left< w_h^j, \mathsf{Var}_{\Pv_{i,j}^{\pi, V}}(V_{i,j+1}^{\pi, \ror_i}) \right> } }_{=: \cB_3}  \notag \\
  & \quad + \underbrace{2\sqrt{\frac{H \log(\frac{18S \allA nHN}{\delta})}{N}} \sqrt{\sum_{j=h}^H \left< w_h^j, \left| \mathsf{Var}_{\Pv_{j}^{\pi}}(V_{i,j+1}^{\pi, \ror_i}) - \mathsf{Var}_{\Pv_{i,j}^{\pi, V}}(V_{i,j+1}^{\pi, \ror_i})\right| \right>} }_{=: \cB_4} \notag \\
  & \quad + \underbrace{2\sqrt{\frac{H\log(\frac{18S \allA nHN}{\delta})}{N}} \sqrt{\sum_{j=h}^H \left< w_h^j, \left| \mathsf{Var}_{\Pv_{j}^{\pi}}(\widehat{V}_{i,j+1}^{\pi, \ror_i} - V_{i,j+1}^{\pi, \ror_i})\right| \right>} }_{=: \cB_5}, \label{eq:key-concentration-bound3}
\end{align}
where (i) holds by the triangle inequality and the elementary inequality $\sqrt{\mathsf{Var}_P(V+V')} \leq \sqrt{\mathsf{Var}_P(V)} + \sqrt{\mathsf{Var}_P(V')}$ for any transition kernel $P\in\mathbb{R}^S$ and vectors $V,V' \in \mathbb{R}^S$, and the last inequality follows from applying the Cauchy-Schwarz inequality to those terms.

We can control the three main terms in \eqref{eq:key-concentration-bound3} separately as below:

\begin{itemize}
	\item {\bf Controlling $\cB_3$.} First, we introduce the following lemma for $\sum_{j=h}^H \left< w_h^j, \mathsf{Var}_{\Pv_{i,j}^{\pi, V}}(V_{i,j+1}^{\pi, \ror_i})  \right>$.

\begin{lemma}\label{lem:key-lemma-reduce-H-2}
Consider any $\delta\in(0,1)$. For any joint policy $\pi$, with probability at least $1-\delta$,
\begin{align}
 \forall (h,i)\in [H] \times [n]:\quad  &\sum_{j=h}^H \left< w_h^j, \mathsf{Var}_{\Pv_{i,j}^{\pi, \widehat{V}}}(V_{i,j+1}^{\pi, \ror_i}) \right>  \leq 3H \left(\max_{s\in\cS}V_{i,h}^{\pi, \ror_i}(s) - \min_{s\in\cS}V_{i,h}^{\pi, \ror_i}(s)\right).
\end{align}
\end{lemma}
\begin{proof}
See Appendix~\ref{proof:lem:key-lemma-reduce-H-2}.
\end{proof}

Then applying Lemma~\ref{lem:key-lemma-reduce-H-2} yields
\begin{align}
\cB_3 & =2\sqrt{\frac{H\log \left(\frac{18S \allA nHN}{\delta} \right)}{N}}  \sqrt{\sum_{j=h}^H \left< w_h^j, \mathsf{Var}_{\Pv_{i,j}^{\pi, V}}(V_{i,j+1}^{\pi, \ror_i}) \right> } \notag \\
& \leq 2\sqrt{\frac{H\log \left(\frac{18S \allA nHN}{\delta} \right)}{N}} \sqrt{3H \left(\max_{s\in\cS}\widehat{V}_{i,h}^{\pi, \ror_i}(s) - \min_{s\in\cS}\widehat{V}_{i,h}^{\pi, \ror_i}(s)\right) } \notag \\
& \leq 4\sqrt{  \frac{H^2\min \left\{1/\ror_i,H\right\}  \log \left(\frac{18S \allA nHN}{\delta} \right)}{N}  }, \label{eq:key-concentration-bound2-solve2-B3}
\end{align}
where the last inequality follows from Lemma~\ref{lemma:pnorm-key-value-range}.

\item {\bf Controlling $\cB_4$ and $\cB_5$ .}
First, it is easily verified that $\cB_4$ can be controlled as the same as that for $\cB_2$ (see \eqref{eq:key-concentration-bound2-solve2}) by applying Lemma~\eqref{eq:lemma-6-line2}, namely
\begin{align}
	\cB_4 \leq 2\sqrt{\frac{H^2\min \left\{1/\ror_i, H \right\}\log(\frac{18S \allA nHN}{\delta})}{N}} \label{eq:key-concentration-bound2-solve2-B4}.
\end{align}

Then the remainder of the proof shall focus on $\cB_5$. Recalling the definition in \eqref{eq:key-concentration-bound3}, one has
\begin{align}
\cB_5 &= 2\sqrt{\frac{H\log(\frac{18S \allA nHN}{\delta})}{N}} \sqrt{ \sum_{j=h}^H \left< w_h^j, \left| \mathsf{Var}_{\Pv_{j}^{\pi}}(\widehat{V}_{i,j+1}^{\pi, \ror_i} - V_{i,j+1}^{\pi,\ror_i})\right| \right>} \notag \\
& \leq 2\sqrt{\frac{H^2\log(\frac{18S \allA nHN}{\delta})}{N}} \sqrt{ \max_{h\leq j\leq H} \left\| \mathsf{Var}_{\Pv_{j}^{\pi}}(\widehat{V}_{i,j+1}^{\pi, \ror_i} - V_{i,j+1}^{\pi,\ror_i})\right\|_\infty } \notag \\
& \leq 2\sqrt{\frac{H^2\log(\frac{18S \allA nHN}{\delta})}{N}}  \max_{h\leq j\leq H} \left\|\widehat{V}_{i,j+1}^{\pi, \ror_i} - V_{i,j+1}^{\pi,\ror_i} \right\|_\infty \label{eq:key-concentration-bound2-solve2-B5}.
\end{align}
\end{itemize}

Summing up \eqref{eq:key-concentration-bound2-solve2-B3}, \eqref{eq:key-concentration-bound2-solve2-B4}, and \eqref{eq:key-concentration-bound2-solve2-B5} and inserting back to \eqref{eq:key-concentration-bound3}, we conclude
\begin{align}
	&\widehat{V}_{i,h}^{\pi, \ror_i}(s) - V_{i,h}^{\pi,\ror_i}(s)  \notag \\
	& \leq \frac{\log(\frac{18S\allA nHN}{\delta}) H^2}{N} + 4\sqrt{  \frac{H^2\min \left\{1/\ror_i,H\right\}  \log(\frac{18S \allA nHN}{\delta})}{N}  } \notag \\
	& \quad + 2\sqrt{\frac{H^2\min \left\{1/\ror_i, H \right\}\log(\frac{18S \allA nHN}{\delta})}{N}} + 2\sqrt{\frac{H^2\log(\frac{18S \allA nHN}{\delta})}{N}} \max_{h\leq j\leq H}  \left\|\widehat{V}_{i,j+1}^{\pi, \ror_i} - V_{i,j+1}^{\pi,\ror_i} \right\|_\infty \notag \\
	& \leq 7\sqrt{\frac{H^2\min \left\{1/\ror_i, H \right\}\log(\frac{18S \allA nHN}{\delta})}{N}} 1 \notag \\
  & \quad + 2\sqrt{\frac{H^2\log(\frac{18S \allA nHN}{\delta})}{N}} \max_{h\leq j\leq H}   \left\|\widehat{V}_{i,j+1}^{\pi, \ror_i} - V_{i,j+1}^{\pi,\ror_i} \right\|_\infty 1, \label{eq:upper-final2}
\end{align} 
as long as $N \geq H^2\log \big(\frac{18S \allA nHN}{\delta} \big)$.

\paragraph{Step 5: summing up the results.}

Inserting \eqref{eq:upper-final1} and \eqref{eq:upper-final2} back into \eqref{eq:summary-control-term}, we observe that
\begin{align}
\left|\widehat{V}_{i,h}^{\pi,\ror_i} - V_{i,h}^{\pi,\ror_i} \right| &\leq \max \left\{V_{i,h}^{\pi,\ror_i} - \widehat{V}_{i,h}^{\pi, \ror_i} , \widehat{V}_{i,h}^{\pi, \ror_i} - V_{i,h}^{\pi,\ror_i} \right\} \notag \\
& \leq \max \Big\{ 9\sqrt{\frac{H^2\min \left\{1/\ror_i, H \right\}\log(\frac{18S \allA nHN}{\delta})}{N}} 1, \notag \\
& \quad 7\sqrt{\frac{H^2\min \left\{1/\ror_i, H \right\}\log(\frac{18S \allA nHN}{\delta})}{N}}  1 + 2\sqrt{\frac{H^2\log(\frac{18S \allA nHN}{\delta})}{N}} \max_{h\leq j\leq H}    \left\|\widehat{V}_{i,j+1}^{\pi, \ror_i} - V_{i,j+1}^{\pi,\ror_i} \right\|_\infty  1 \Big\},
\end{align}
which indicates
\begin{align}
	& \max_{h\in[H]}\left\|\widehat{V}_{i,h}^{\pi, \ror_i} - V_{i,h}^{\pi,\ror_i} \right\|_\infty \notag \\
	& \leq 9\sqrt{\frac{H^2\min \left\{1/\ror_i, H \right\}\log(\frac{18S \allA nHN}{\delta})}{N}} 1 + 2\sqrt{\frac{H^2\log(\frac{18S \allA nHN}{\delta})}{N}} \max_{h\in[H]} \left\|\widehat{V}_{i,h+1}^{\pi, \ror_i} - V_{i,h+1}^{\pi,\ror_i} \right\|_\infty \notag \\
	& \overset{\mathrm{(i)}}{\leq} 9\sqrt{\frac{H^2\min \left\{1/\ror_i, H \right\}\log(\frac{18S \allA nHN}{\delta})}{N}} 1 + \frac{1}{2} \max_{h\in[H]} \left\|\widehat{V}_{i,h}^{\pi, \ror_i} - V_{i,h}^{\pi,\ror_i} \right\|_\infty \notag \\
	& \leq 18\sqrt{\frac{H^2\min \left\{1/\ror_i, H \right\}\log(\frac{18S \allA nHN}{\delta})}{N}}, \label{eq:upper-final3}
\end{align}
where (i) holds by taking  $N \geq 16H^2\log(\frac{18S \allA nHN}{\delta})$ and invoking the basic fact that $\widehat{V}_{i,H+1}^{\pi, \ror_i} = V_{i,H+1}^{\pi,\ror_i} =0$.

Finally, we complete the proof by showing that the performance gap in \eqref{eq:l1-decompose} is bounded by
\begin{align}
	 V_{i,1}^{\star,\widehat{\pi}_{-i}} - V_{i,1}^{\widehat{\pi}} 
	&\leq \left(V_{i,1}^{\star,\widehat{\pi}_{-i}} - \widehat{V}_{i,1}^{\widetilde{\pi}_i^\star\times\widehat{\pi}_{-i}}\right) + \left(\widehat{V}_{i,1}^{\widehat{\pi}} - V_{i,1}^{\widehat{\pi}} \right)  \notag \\
	& \leq \left\|V_{i,1}^{\star,\widehat{\pi}_{-i}} - \widehat{V}_{i,1}^{\widetilde{\pi}_i^\star\times\widehat{\pi}_{-i}}\right\|_\infty 1 +  \left\|\widehat{V}_{i,1}^{\widehat{\pi}} - V_{i,1}^{\widehat{\pi}} \right\|_\infty 1 \notag \\
  & \leq  \max_{h\in[H]} \left\|V_{i,h}^{\star,\widehat{\pi}_{-i}} - \widehat{V}_{i,h}^{\widetilde{\pi}_i^\star\times\widehat{\pi}_{-i}}\right\|_\infty 1 +   \max_{h\in[H]} \left\|\widehat{V}_{i,h}^{\widehat{\pi}} - V_{i,h}^{\widehat{\pi}} \right\|_\infty 1 \notag \\
	& \leq   36\sqrt{\frac{H^2\min \left\{1/\ror_i, H \right\}\log(\frac{18S \allA nHN}{\delta})}{N}} 1,
\end{align}
where the last inequality holds by applying \eqref{eq:upper-final3} to two different cases when $\pi = \widetilde{\pi}_i^\star \times \widehat{\pi}_{-i}$ or $\pi = \widehat{\pi}$, respectively.

As a result, to achieve $\max_{s\in\cS, i\in[n]} \left\{ V_{i,1}^{\star,\widehat{\pi}_{-i},\ror_i}(s) - V_{i,1}^{\widehat{\pi},\ror_i}(s)  \right\} \leq \varepsilon$ with probability at least $1-\delta$, we require the total number of samples 
\begin{align}
N_{\mathsf{all}} = HS \prod_{i\in[n]} A_i N &\geq  \frac{ C_1 S H^3\prod_{1\leq i\leq n} A_i  \log\left(\frac{18S \allA nHN}{\delta} \right) }{  \varepsilon^2}\min \Big\{H,  \frac{1}{\min_{1\leq i\leq n} \ror_i} \Big\} \notag \\
& \geq \frac{ C_0 S H^3\prod_{1\leq i\leq n} A_i  \log\left(\frac{18S \allA nHN}{\delta} \right) }{  \varepsilon^2}\min \Big\{H,  \frac{1}{\min_{1\leq i\leq n} \ror_i} \Big\} \notag \\
& \quad + 16H^3 S \prod_{i\in[n]} A_i \log(\frac{18S \allA nHN}{\delta}), \label{eq:proof-final-samples}
\end{align}
providing $C_1 > C_0$ are larger enough universal constant, and $\varepsilon \leq \sqrt{\min \Big\{H,  \frac{1}{\min_{1\leq i\leq n} \ror_i} \Big\}}$.

\subsection{Proof of learning robust CE}

This section is analogous to the proof for learning robust NE/CCE in Appendix~\ref{sec:proof-nash-cce}.

The goal is to prove that the policy $\widehat{\pi}$ output from Algorithm~\ref{alg:nash-dro-vi-finite} is an $\varepsilon$-robust CE when executing subroutine $\mathsf{Compute}-\mathsf{CE}(\cdot)$ for line~\ref{eq:nash-subroutine}, i.e.,
\begin{align}
	\mathsf{gap}_{\mathsf{CE}}(\widehat{\pi}) =  \max_{s\in \cS, 1 \leq i \leq n} \left\{  \max_{f_i \in \cF_i} V_{i,1}^{f_i \diamond \widehat{\pi}, \ror_i }(s) - V_{i,1}^{\widehat{\pi}, \ror_i}(s)  \right\}  \leq \varepsilon.
\end{align}
So we define the following best perturbation policy of the $i$-th player as
\begin{align}
	\overline{\pi}_i^\star = \{\overline{\pi}_{i,h}^\star \}_{1\leq h\leq H} = \left(\textrm{argmax}_{f_i\in \cF_i} V_{i,1}^{f_i \diamond \widehat{\pi}, \ror_i }\right) \diamond \widehat{\pi} 
\end{align}
which leads to 
\begin{align}
V_{i,1}^{\overline{\pi}_i^\star,\ror_i} = \max_{f_i\in \cF_i} V_{i,1}^{f_i \diamond \widehat{\pi}, \ror_i}.
\end{align}
With above notations in mind, for any $1\leq i\leq n$, the term of interest can be decomposed as
\begin{align}
	  \max_{f_i \in \cF_i} V_{i,1}^{f_i \diamond \widehat{\pi}, \ror_i } - V_{i,1}^{\widehat{\pi}, \ror_i}
	& = \left(V_{i,1}^{\overline{\pi}_i^\star,\ror_i} - \widehat{V}_{i,1}^{\overline{\pi}_i^\star,\ror_i}\right) + \left(\widehat{V}_{i,1}^{\overline{\pi}_i^\star,\ror_i} - \widehat{V}_{i,1}^{\widehat{\pi},\ror_i} \right) + \left(\widehat{V}_{i,1}^{\widehat{\pi},\ror_i} - V_{i,1}^{\widehat{\pi},\ror_i} \right) \notag \\
	&  \overset{\mathrm{(i)}}{\leq}  \left(V_{i,1}^{\overline{\pi}_i^\star,\ror_i} - \widehat{V}_{i,1}^{\overline{\pi}_i^\star,\ror_i}\right) + \left(\widehat{V}_{i,1}^{\overline{\pi}_i^\star,\ror_i} - \max_{f_i\in \cF_i} \widehat{V}_{i,1}^{f_i \diamond \widehat{\pi}, \ror_i}\right) + \left(\widehat{V}_{i,1}^{\widehat{\pi},\ror_i} - V_{i,1}^{\widehat{\pi},\ror_i} \right) \notag \\
	&\leq \left(V_{i,1}^{\overline{\pi}_i^\star,\ror_i} - \widehat{V}_{i,1}^{\overline{\pi}_i^\star,\ror_i}\right) + \left(\widehat{V}_{i,1}^{\widehat{\pi},\ror_i} - V_{i,1}^{\widehat{\pi},\ror_i} \right)  \label{eq:l1-decompose-CE}
\end{align}
where (i) holds by $\widehat{V}_{i,1}^{\widehat{\pi},\ror_i} \geq \max_{f_i\in \cF_i} V_{i,1}^{f_i \diamond \widehat{\pi}, \ror_i}$ when the subroutine in line~\ref{eq:nash-subroutine} is $\mathsf{Compute}-\mathsf{CE}(\cdot)$ implied by Lemma~\ref{lem:algo-output-equilibrium}, and the last inequality follows from $\widehat{V}_{i,1}^{\overline{\pi}_i^\star,\ror_i}  = \widehat{V}_{i,1}^{\overline{f}_i \diamond \widehat{\pi} ,\ror_i} \leq \max_{f_i\in \cF_i} \widehat{V}_{i,1}^{f_i \diamond \widehat{\pi}, \ror_i}$ for some $\overline{f}_i \in\cF_i$.

Observing that \eqref{eq:l1-decompose-CE} is similar to \eqref{eq:l1-decompose}, it can be  verified that following the same pipeline routine and the same facts developed from Step 2 to Step 5 in Appendix~\ref{sec:proof-nash-cce}, we can achieve similar results as below:
\begin{align}
	\forall i\in[n]:\quad \max_{f_i \in \cF_i} V_{i,1}^{f_i \diamond \widehat{\pi}, \ror_i } - V_{i,1}^{\widehat{\pi}, \ror_i}
	& \leq   36\sqrt{\frac{H^2\min \left\{1/\ror_i, H \right\}\log(\frac{18S \allA nHN}{\delta})}{N}} 1,
\end{align}
which yields \eqref{eq:proof-final-samples} and complete the proof. We omit the details here for conciseness.

\subsection{Proof of the auxiliary lemmas}

\subsubsection{Proof of Lemma~\ref{lem:algo-output-equilibrium}}\label{proof:lem:algo-output-equilibrium}

We will prove each line of \eqref{eq:iterative-solution-confirm} separately with an induction argument. Note that \citet{blanchet2023double} provides the proof of the first line of \eqref{eq:iterative-solution-confirm} for robust NE. For completeness, we offer the whole proof for all of the three robust solution concepts (including robust-NE).

\paragraph{Proof for robust NE.} First, we focus on the first line of \eqref{eq:iterative-solution-confirm} and provide the following induction argument:
\begin{itemize}
    \item {\em Base case when $h=H$.} Note that $\widehat{V}_{i,H+1}^{\pi,\ror_i} =0$ for all $i\in[n]$ are satisfied by definition. As a result, the robust Q-function for any joint policy $\pi$ and the estimate from Algorithm~\ref{alg:nash-dro-vi-finite} satisfy
    \begin{align}
      \forall (i,s,a) \in[n] \times \cS \times \cA: \quad  \widehat{Q}_{i,H}^{\pi,\ror_i}(s,\ba) = r_{i,H}(s,\ba) \quad \text{and} \quad  \widehat{Q}_{i,H}(s,\ba) = r_{i,H}(s,\ba) \label{eq:q-step-H}
    \end{align}
    which directly leads to 
    \begin{align}
        \widehat{V}_{i,H}^{\pi,\ror_i}=\widehat{V}_{i,H}
    \end{align}
    and the output $\pi_H$ obeying
    \begin{align}
  \forall s\in \cS: \quad  \widehat{\pi}_H(\cdot\mymid s) \leftarrow \mathsf{Compute}-\mathsf{Nash}\left(r_{1,H}(s,\ba), r_{2,H}(s,\ba), \cdots, r_{n,H}(s,\ba) \right). \label{eq:nash-compute-H}
    \end{align}

 Consequently, invoking line~\ref{eq:nash-subroutine} of Algorithm~\ref{alg:nash-dro-vi-finite} gives that for all $s\in\cS$,
 \begin{align}
\widehat{V}_{i,H}(s)  &= \mathbb{E}_{\ba \sim \widehat{\pi}_H(s)} \left[\widehat{Q}_{i,H}(s,\ba) \right] \overset{\mathrm{(i)}}{=}   \mathbb{E}_{\ba \sim \widehat{\pi}_H(s)} \left[\widehat{Q}^{\widehat{\pi},\ror_i}_{i,H}(s,\ba) \right] \overset{\mathrm{(ii)}}{=} \mathbb{E}_{\ba \sim \widehat{\pi}_H(s)}[r_{i,H}(s,\ba)] \label{eq:equal-Qhat-Qhatpi-H} \\
 &\overset{\mathrm{(iii)}}{=} \max_{\widetilde{\pi}_{i,H}(s) \in\Delta(\cA_i)} \mathbb{E}_{\ba \sim \widetilde{\pi}_{i,H}(s) \times \widehat{\pi}_{-i,H}(s)}[r_{i,H}(s,\ba)] \label{eq:robust-NE-middle-1} \\ 
 & \overset{\mathrm{(iv)}}{=}  \max_{\widetilde{\pi}_{i,H}(s) \in\Delta(\cA_i)} \mathbb{E}_{\ba \sim \widetilde{\pi}_{i,H}(s) \times \widehat{\pi}_{-i,H}(s)}\left[\widehat{Q}_{i,H}^{\widetilde{\pi}_i \times \widehat{\pi}_{-i},\ror_i}(s,\ba)\right] \notag \\
 & = \max_{\widetilde{\pi}_i: \cS\times [H] \rightarrow \in\Delta(\cA_i)} \widehat{V}^{\widetilde{\pi}_i \times \widehat{\pi}_{-i},\ror_i}_{i,H}(s)  = \widehat{V}^{\star,\widehat{\pi}_{-i},\ror_i}_{i,H}, \label{eq:equal-Qhat-Qhatpi-H-2} 
 \end{align}
 where (i) and (ii) hold by \eqref{eq:q-step-H}, (iii) arises from 
 the definition of robust-NE (see \eqref{eq:nash-compute-H}) associated with $\{r_{i,H}\}_{i\in[n]}$, (iv) holds by applying \eqref{eq:q-step-H} for policy $\pi = \widetilde{\pi}_i \times \widehat{\pi}_{-i}$,
 and the penultimate equality follows from the fact that only the policy of the time step $H$ will influence $\widehat{V}^{\pi,\ror_i}_{i,H}(s)$ due to Markov property. Thus we complete the proof for the base case.

 \item {\em Induction.} To continue, suppose the first line in \eqref{eq:iterative-solution-confirm} holds for step $h+1$, we shall proof that it also holds for time step $h$.
To proceed, applying the robust Bellman equation in \eqref{eq:bellman-equ-star-infinite-estimate} for the TV uncertainty set $\unb^{\ror_i} (\cdot)$, we observe that
\begin{align}
    \forall (s,a)\in \cS \times \cA: \quad \widehat{Q}_{i,h}^{\widehat{\pi},\ror_i}(s,\ba) &= r_{i,h}(s, \ba) +  \inf_{P\in \unb^{\ror_i}(\widehat{P}^{\no}_{h,s,\ba})} P  \widehat{V}_{i,h+1}^{\widehat{\pi},\ror_i}. \label{eq:q-step-h}
\end{align}
In addition, line~\ref{eq:robust-q-estimate} of Algorithm~\ref{alg:nash-dro-vi-finite} gives that for all $(s,a)\in \cS \times \cA$,
 \begin{align}
    \widehat{Q}_{i,h}(s,\ba) &= r_{i,h}(s, \ba) +  \inf_{P\in \unb^{\ror_i}(\widehat{P}^{\no}_{h,s,\ba})} P  \widehat{V}_{i,h+1} \notag \\
    & = r_{i,h}(s, \ba) +  \inf_{P\in \unb^{\ror_i}(\widehat{P}^{\no}_{h,s,\ba})} P  \widehat{V}_{i,h+1}^{\widehat{\pi},\ror_i} = \widehat{Q}_{i,h}^{\widehat{\pi},\ror_i}(s,\ba), \label{eq:Q-value-equal-h}
\end{align}
where the penultimate equality holds by the induction assumption and the final equality follows from  \eqref{eq:q-step-h}. It indicates 
\begin{align}
   \forall s\in\cS: \quad \widehat{V}_{i,h}(s)  &= \mathbb{E}_{\ba \sim \widehat{\pi}_h(s)} \left[\widehat{Q}_{i,h}(s,\ba) \right]=  \mathbb{E}_{\ba \sim \widehat{\pi}_h(s)} \left[\widehat{Q}^{\widehat{\pi},\ror_i}_{i,h}(s,\ba)\right] =  \widehat{V}^{\widehat{\pi},\ror_i}_{i,h}(s) 
\end{align}
and that 
the output policy obeys
 \begin{align}
   \forall s\in\cS: \quad  \widehat{\pi}_h(\cdot\mymid s) \leftarrow \mathsf{Compute}-\mathsf{Nash}\left( \widehat{Q}_{1,h}^{\widehat{\pi},\ror_i}(s,\cdot), \widehat{Q}_{2,h}^{\widehat{\pi},\ror_i}(s,\cdot) , \widehat{Q}_{n,h}^{\widehat{\pi},\ror_i}(s,\cdot) \right). \label{eq:nash-compute-h}
    \end{align}
Then the term of interest satisfies that for any $s\in\cS$,
 \begin{align}
& \widehat{V}_{i,h}^{\star,\widehat{\pi}_{-i},\ror_i} (s)
 = \max_{\widetilde{\pi}_i: \cS\times [H] \rightarrow \Delta(\cA_i)} \mathbb{E}_{\ba \sim \widetilde{\pi}_{i,h}(s) \times \widehat{\pi}_{-i,h}(s)} \Big[\widehat{Q}^{\widetilde{\pi}_i \times \widehat{\pi}_{-i},\ror_i}_{i,h}(s,\ba) \Big] \notag \\
 & =  \max_{\widetilde{\pi}_i: \cS\times [H] \rightarrow \Delta(\cA_i)} \mathbb{E}_{\ba \sim \widetilde{\pi}_{i,h}(s) \times \widehat{\pi}_{-i,h}(s)} \Big[r_{i,h}(s,\ba) + \inf_{P\in \unb^{\ror_i}(\widehat{P}^{\no}_{h,s,\ba})} P  \widehat{V}_{i,h+1}^{\widetilde{\pi}_i \times \widehat{\pi}_{-i},\ror_i} \Big] \notag \\
 & \overset{\mathrm{(i)}}{=} \max_{\widetilde{\pi}_{i,h} \in \Delta(\cA_i)} \mathbb{E}_{\ba \sim \widetilde{\pi}_{i,h}(s) \times \widehat{\pi}_{-i,h}(s)} \Big[r_{i,h}(s,\ba) + \max_{\widetilde{\pi}_i: \cS\times [H] \rightarrow \Delta(\cA_i)} \inf_{P\in \unb^{\ror_i}(\widehat{P}^{\no}_{h,s,\ba})} P  \widehat{V}_{i,h+1}^{\widetilde{\pi}_i \times \widehat{\pi}_{-i},\ror_i} \Big] \notag \\
 & \overset{\mathrm{(ii)}}{=}   \max_{\widetilde{\pi}_{i,h}(s) \in\Delta(\cA_i)} \mathbb{E}_{\ba \sim \widetilde{\pi}_{i,h}(s) \times \widehat{\pi}_{-i,h}(s)} \Big[r_{i,h}(s,\ba) + \inf_{P\in \unb^{\ror_i}(\widehat{P}^{\no}_{h,s,\ba})} P  \widehat{V}_{i,h+1}^{\widehat{\pi},\ror_i} \Big]
 \notag \\
& =  \max_{\widetilde{\pi}_{i,h}(s) \in\Delta(\cA_i)} \mathbb{E}_{\ba \sim \widetilde{\pi}_{i,h}(s) \times \widehat{\pi}_{-i,h}(s)}\left[  \widehat{Q}_{i,h}^{\widehat{\pi},\ror_i}(s,\ba)\right],
 \label{eq:equal-Qhat-Qhatpi-h-1-half} 
 \end{align}
 where (i) holds by $r_{i,h}(s,a)$ is independent from all other time steps $h' \neq h$, (ii) is due to the exchangability of $\max_{\widetilde{\pi}_i: \cS\times [H] \rightarrow \Delta(\cA_i)}$ and $\inf_{P\in \unb^{\ror_i}(\widehat{P}^{\no}_{h,s,\ba})} $, along with the  induction assumption $\widehat{V}_{i,h+1}^{\widehat{\pi},\ror_i} = \widehat{V}_{i,h+1}^{\star, \widehat{\pi}_{-i},\ror_i} = \max_{\widetilde{\pi}_i: \cS\times [H] \rightarrow \Delta(\cA_i)} \widehat{V}_{i,h+1}^{\widetilde{\pi}_i \times \widehat{\pi}_{-i},\ror_i}$, and the last equality can be verified by \eqref{eq:q-step-h}. To continue, applying  \eqref{eq:nash-compute-h} with the definition of robust NE, one has
 \begin{align}
  \widehat{V}_{i,h}^{\star,\widehat{\pi}_{-i},\ror_i} (s)& =  \max_{\widetilde{\pi}_{i,h}(s) \in\Delta(\cA_i)} \mathbb{E}_{\ba \sim \widetilde{\pi}_{i,h}(s) \times \widehat{\pi}_{-i,h}(s)}\left[  \widehat{Q}_{i,h}^{\widehat{\pi},\ror_i}(s,\ba)\right] \notag \\
  & = \mathbb{E}_{\ba\in \widehat{\pi}_{h}(s)} \left[ \widehat{Q}_{i,h}^{\widehat{\pi},\ror_i}(s,\ba)\right] = \mathbb{E}_{\ba\in \widehat{\pi}_{h}(s)} \left[ \widehat{Q}_{i,h}(s,\ba)\right]= \widehat{V}_{i,h}(s), \label{eq:equal-Qhat-Qhatpi-h-2}
 \end{align} 
 where the penultimate equality follows from \eqref{eq:Q-value-equal-h}.
  Finally, it is easily observed that 
 \begin{align}
 \forall s\in\cS: \widehat{V}_{i,h}(s) = \mathbb{E}_{\ba\in \widehat{\pi}_{h}(s)} \left[ \widehat{Q}_{i,h}(s,\ba)\right] = \mathbb{E}_{\ba\in \widehat{\pi}_{h}(s)} \left[ \widehat{Q}_{i,h}^{\widehat{\pi},\ror_i}(s,\ba)\right] = \widehat{V}^{\widehat{\pi},\ror_i}_{i,h}(s).
 \end{align}
 Combined this fact with \eqref{eq:equal-Qhat-Qhatpi-h-2} shows that
 $\widehat{V}_{i,h} = \widehat{V}^{\widehat{\pi},\ror_i}_{i,h} = \widehat{V}_{i,h}^{\star,\widehat{\pi}_{-i},\ror_i} $, which complete the induction argument.
 
\end{itemize}

\paragraph{Proof for robust CCE.}
The proof is analogous to the above argument for robust NE. According to the different subroutine $\mathsf{Compute}-\mathsf{CCE}$ and the corresponding output policy $\widehat{\pi}$, the proof only differs in two steps. First, for the base case, following the same routine in \eqref{eq:equal-Qhat-Qhatpi-H-2} but replacing the robust NE property by the one of  robust CCE, one has
 \begin{align}
\widehat{V}_{i,H}(s)  &= \mathbb{E}_{\ba \sim \widehat{\pi}_H(s)}[\widehat{Q}_{i,H}(s,\ba)] =\mathbb{E}_{\ba \sim \widehat{\pi}_H(s)}[r_{i,H}(s,\ba)] \notag \\
 &\geq \max_{\widetilde{\pi}_{i,H}(s) \in\Delta(\cA_i)} \mathbb{E}_{\ba \sim \widetilde{\pi}_{i,H}(s) \times \widehat{\pi}_{-i,H}(s)}[r_{i,H}(s,\ba)] \notag \\
 &=\max_{\widetilde{\pi}_{i,H}(s) \in\Delta(\cA_i)} \mathbb{E}_{\ba \sim \widetilde{\pi}_{i,H}(s) \times \widehat{\pi}_{-i,H}(s)}\left[\widehat{Q}_{i,H}^{\widetilde{\pi}_i \times \widehat{\pi}_{-i},\ror_i}(s,\ba)\right] \notag \\
 & = \max_{\widetilde{\pi}_i: \cS\times [H] \rightarrow \in\Delta(\cA_i)} \widehat{V}^{\widetilde{\pi}_i \times \widehat{\pi}_{-i},\ror_i}_{i,H}(s)  = \widehat{V}^{\star,\widehat{\pi}_{-i},\ror_i}_{i,H}.  \label{eq:equal-Qhat-Qhatpi-H-2-CCE} 
 \end{align}
Secondly, following \eqref{eq:equal-Qhat-Qhatpi-h-2} in induction step, we can achieve 
  \begin{align}
 & \widehat{V}_{i,h}^{\star,\widehat{\pi}_{-i},\ror_i} 
 \leq \widehat{V}_{i,h} \label{eq:equal-Qhat-Qhatpi-h-2} 
 \end{align}
and  $\widehat{V}_{i,h} = \widehat{V}^{\widehat{\pi},\ror_i}_{i,h} \geq  \widehat{V}_{i,h}^{\star,\widehat{\pi}_{-i},\ror_i} $, which complete the proof.

\paragraph{Proof for robust CE.}
The proof is similar to the one of robust NE as well. According to the different subroutine $\mathsf{Compute}-\mathsf{CE}$ and the corresponding output policy $\widehat{\pi}$, the parallel claims to \eqref{eq:equal-Qhat-Qhatpi-H-2} and \eqref{eq:equal-Qhat-Qhatpi-h-2} are shown below, which we omit the process for brevity:
 \begin{align}
\widehat{V}_{i,H}(s)  &= \mathbb{E}_{\ba \sim \widehat{\pi}_H(s)}[\widehat{Q}_{i,H}(s,\ba)] =\mathbb{E}_{\ba \sim \widehat{\pi}_H(s)}[r_{i,H}(s,\ba)] \notag \\
 &\geq \max_{f_{i,H,s}:\cA_i \rightarrow  \cA_i } \mathbb{E}_{\ba \sim f_{i,H,s} \diamond \widehat{\pi}_{H}(s)}[r_{i,H}(s,\ba)] = \max_{f_i\in\cF_i} \widehat{V}^{f_i \diamond \widehat{\pi},\ror_i}_{i,H}, \label{eq:equal-Qhat-Qhatpi-H-2-CE} 
 \end{align}
 and
  \begin{align}
\max_{f_i\in\cF_i} \widehat{V}^{f_i \diamond \widehat{\pi},\ror_i}_{i,h} \leq \widehat{V}^{\widehat{\pi},\ror_i}_{i,h} = \widehat{V}_{i,h}. \label{eq:equal-Qhat-Qhatpi-h-2-CE} 
 \end{align}
Thus we complete the proof.

\subsubsection{Proof of Lemma~\ref{lemma:pnorm-key-value-range}}\label{proof:lemma:pnorm-key-value-range}

To begin with, we observe that
\begin{align}
	\min_{s\in\cS} V_{i,h}^{\pi,\ror_i}(s) &= \min_{s\in\cS}\mathbb{E}_{a\sim \pi_h(s)}[Q_{i,h}^{\pi,\ror_i}(s,\ba)] = \min_{s\in\cS}\mathbb{E}_{a\sim \pi_h(s)}[r_{i,h}(s, \ba) +  \inf_{P\in \unb^{\ror_i}(P_{h,s,\ba})} P V_{i,h+1}^{\pi,\ror_i}] \notag \\
& \geq 0 + \min_{s\in\cS} V_{i,h+1}^{\pi,\ror_i}(s), \label{eq:shrink-min}
\end{align}
where the second equality holds by the robust Bellman equation (cf.~\eqref{eq:bellman-consistency-sa}).
Similarly, one has
\begin{align}
\max_{s\in\cS} V_{i,h}^{\pi,\ror_i}(s) &= \max_{s\in\cS}\mathbb{E}_{a\sim \pi_h(s)}[Q_{i,h}^{\pi,\ror_i}(s,\ba)] = \max_{s\in\cS}\mathbb{E}_{a\sim \pi_h(s)}[r_{i,h}(s, \ba) +  \inf_{P\in \unb^{\ror_i}(P_{h,s,\ba})} P V_{i,h+1}^{\pi,\ror_i}] \notag \\
& \leq 1 + \max_{(s,\ba)\in\cS \times \cA} \inf_{P\in \unb^{\ror_i}(P_{h,s,\ba})} P V_{i,h+1}^{\pi,\ror_i}. \label{eq:shrink-max}
\end{align}
Armed with above results, we are ready to prove  Lemma~\ref{lemma:pnorm-key-value-range}. Towards this, we introduce some additional notations for convenience. Fixing any joint policy $\pi$, note that for any $(i,h)\in [n]\times [H]$, there exist at least one state $s_{i,h}^\star$ that satisfies $V_{i,h}^{\pi,\ror_i}(s_{i,h}^\star) = \min_{s\in\cS} V_{i,h}^{\pi,\ror_i}(s)$. 

Then, it is observed that for any $(s,\ba) \in\cS \times \cA$ and accessible uncertainty set $\ror_i>0$, we can construct an auxiliary vector $P'_{h,s,\ba} \in \mathbb{R}^{S}$ by strictly reducing the values of some elements of $P_{h,s,\ba}$ so that 
\begin{equation}
0 \leq P'_{h,s,\ba} \leq P_{h,s,\ba} \quad \text{and} \quad \sum_{s'\in\cS} P_{h,s,\ba}(s') - P'_{h,s,\ba}(s') = \left\|P'_{h,s,\ba} - P_{h,s,\ba} \right\|_1  = \ror_i. \label{eq:range-shrink-q-norm-0}
\end{equation}
Recalling $e_{s_{i,h}^\star}$ denote a $S$-dimensional standard basis supported on $s_{i,h}^\star$, the above fact directly indicates that
\begin{align}
 \frac{1}{2}\left\|P'_{h,s,\ba} + \ror_i \big[e_{s_{i,h}^\star}\big]^\top - P_{h,s,\ba} \right\|_1  &\leq  \frac{1}{2}\left\|P'_{h,s,\ba} - P_{h,s,\ba} \right\|_1 +  \frac{1}{2}\left\| \ror_i \big[e_{s_{i,h}^\star}\big]^\top \right\|_1 \leq \ror_i, \label{eq:range-shrink-q-norm}
\end{align}
where the first inequality holds by that TV distance enjoys the triangle inequality.

The above results in \eqref{eq:range-shrink-q-norm} imply that $P'_{h,s,\ba} + \ror_i \big[e_{s_{i,h}^\star}\big]^\top$ is a distribution vector and  $P'_{h,s,\ba} + \ror_i \big[e_{s_{i,h}^\star}\big]^\top \in \cU^{\ror_i}(P_{h,s,\ba})$, which leads to
\begin{align}
\inf_{P\in \unb^{\ror_i}(P_{h,s,\ba})} P V_{i,h+1}^{\pi,\ror_i} \leq \left( P'_{h,s,\ba} +  \ror_i \big[e_{s_{i,h}^\star}\big]^\top \right)V_{i,h+1}^{\pi,\ror_i}  &\leq \big\| P'_{h,s,\ba} \big\|_1 \big\| V_{i,h+1}^{\pi,\ror_i} \big\|_\infty +  \ror_i V_{i,h+1}^{\pi,\ror_i}  (s_{i,h+1}^\star) \notag \\
& \leq \left(1- \ror_i \right) \max_{s\in\cS} V_{i,h+1}^{\pi,\ror_i}(s) +  \ror_i \min_{s\in\cS} V_{i,h+1}^{\pi,\ror_i} (s), \label{eq:shrink-max2}
\end{align}
where the last inequality can be verified by (see \eqref{eq:range-shrink-q-norm-0})
\begin{align}
\big\| P'_{h,s,\ba}\big\|_1 =  \sum_{s'} P'_{h,s, \ba}(s') =  - \sum_{s'} \left( P_{h,s,\ba}(s') - P'_{h,s,\ba}(s') \right) + \sum_{s'} P_{h,s,\ba}(s')  = 1-\ror_i.
\end{align}

Inserting \eqref{eq:shrink-max2} back to \eqref{eq:shrink-max} yields
\begin{align}
\max_{s\in\cS} V_{i,h}^{\pi,\ror_i}(s) &\leq 1 + \max_{(s,\ba)\in\cS \times \cA} \inf_{P\in \unb^{\ror_i}(P_{h,s,\ba})} P V_{i,h+1}^{\pi,\ror_i} \notag \\
& \leq 1 + \left(1- \ror_i \right) \max_{s\in\cS} V_{i,h+1}^{\pi,\ror_i}(s) +  \ror_i \min_{s\in\cS} V_{i,h+1}^{\pi,\ror_i} (s).
\end{align}

Combined above fact with \eqref{eq:shrink-min} shows that
\begin{align}
\max_{s\in\cS} V_{i,h}^{\pi,\ror_i}(s)  - \min_{s\in\cS} V_{i,h}^{\pi,\ror_i}(s) & \leq  1 + \left(1- \ror_i \right) \max_{s\in\cS} V_{i,h+1}^{\pi,\ror_i}(s) +  \ror_i \min_{s\in\cS} V_{i,h+1}^{\pi,\ror_i} (s) - \min_{s\in\cS} V_{i,h+1}^{\pi,\ror_i}(s) \notag \\
& \leq 1 + (1-\ror_i) \left(\max_{s\in\cS} V_{i,h+1}^{\pi,\ror_i}(s)  - \min_{s\in\cS} V_{i,h+1}^{\pi,\ror_i}(s) \right) \notag \\
& \leq 1 + (1-\ror_i) \left[1 + (1-\ror_i) \left(\max_{s\in\cS} V_{i,h+2}^{\pi}(s)  - \min_{s\in\cS} V_{i,h+2}^{\pi}(s) \right)\right] \notag \\
& \leq \cdots \leq \frac{1 - (1-\ror_i)^{H-h}}{\ror_i} \leq \frac{1}{\ror_i}.
\end{align}

Combining above result with the basic fact $\max_{s\in\cS} V_{i,h}^{\pi,\ror_i}(s)  - \min_{s\in\cS} V_{i,h}^{\pi,\ror_i}(s) \leq H-h+1$, we complete the proof.
\subsubsection{Proof of Lemma~\ref{lemma:tv-dro-b-bound-star-marl}}\label{proof:lemma:tv-dro-b-bound-star-marl}

The proof is adapted from the routine for proving \citet[Lemma~9]{shi2023curious}.

\paragraph{Step 1: a point-wise bound.}
Consider any fixed (independent from $\widehat{P}^\no$) value vector $V$, combined with the definitions in \eqref{eq:inf-p-special-marl}, the $(s,\ba)$-th row of the term of interest can be written out as 
\begin{align}
 \left|\pmin_{i,h,s, \ba}^{V} V - \pmhat_{i,h,s,\ba}^{V} V \right| &= \left|\inf_{ \cP \in \unb^{\ror_i}(P^\no_{h,s,\ba})} \cP V  - \inf_{ \cP \in \unb^{\ror_i}(\widehat{P}^\no_{h,s,\ba})} \cP V \right| \notag \\
& \overset{\mathrm{(i)}}{=} \Big| \max_{\alpha\in [\min_s V(s), \max_s V(s)]} \left\{P^\no_{h,s,\ba} \left[V\right]_{\alpha} - \ror_i~ \left(\alpha - \min_{s'}\left[V\right]_{\alpha}(s') \right)\right\} \notag \\
& \quad  - \max_{\alpha\in [\min_s V(s), \max_s V(s)]} \left\{\widehat{P}^\no_{h,s,\ba} \left[V\right]_{\alpha} - \ror_i~ \left(\alpha - \min_{s'}\left[V\right]_{\alpha}(s') \right)\right\}  \Big| \notag \\
& \leq  \max_{\alpha\in [\min_s V(s), \max_s V(s)]} \left| P^\no_{h,s,\ba} \left[V\right]_{\alpha} - \widehat{P}^\no_{h,s,\ba} \left[V\right]_{\alpha}\right| \notag \\
& \leq \max_{\alpha\in [0, H]} \left| P^\no_{h,s,\ba} \left[V\right]_{\alpha} - \widehat{P}^\no_{h,s,\ba} \left[V\right]_{\alpha}\right|,  \label{eq:middle-key-lemma}
\end{align}
where (i) holds by applying Lemma~\ref{lemma:tv-dual-form}, and the last inequality can be verified by the fact that the maximum operator is $1$-Lipschitz.  

To continue, recalling the definition of variance in \eqref{eq:defn-variance} and using the Bernstein's inequality, one has for a fixed $\alpha \in [0,H]$ and $(s,\ba) \in \cS \times \cA$, with probability at least $1 - \delta$,
\begin{align} \label{eq:V-p-phat-gap-one-alpha-bernstein}
	\left| \left(P^{\no}_{h,s,\ba} - \widehat{P}^{\no}_{h,s,\ba} \right) [V]_\alpha\right|  & \leq \sqrt{\frac{2\log(\frac{2}{\delta})}{N}} \sqrt{\mathrm{Var}_{P^{\no}_{h,s,\ba}}([V]_\alpha)} +  \frac{2 H\log(\frac{2}{\delta})}{3N}  \nonumber \\
	&\leq \sqrt{\frac{2\log(\frac{2}{\delta})}{N}} \sqrt{\mathrm{Var}_{P^{\no}_{h,s,\ba}}(V)} +  \frac{2H \log(\frac{2}{\delta})}{3N},
\end{align}
where the first inequality holds by the fact that $\|V\|_\infty \leq H$, and the last inequality can be easily verified by noticing that $\mathrm{Var}_{P^{\no}_{h,s,\ba}}([V]_\alpha) \leq \mathrm{Var}_{P^{\no}_{h,s,\ba}}(V)$ for all $\alpha \in [0, \max_s V(s)]$.

\paragraph{Step 2: the union bound.}
Then to obtain the union bound, we first notice that the function $\left| \left(P^{\no}_{h,s,\ba} - \widehat{P}^{\no}_{h,s,\ba} \right) [V]_\alpha\right|$ is $1$-Lipschitz w.r.t. $\alpha$ for any V obeying $0\leq V(s) \leq H$. Therefore, we can construct an $\varepsilon_1$-net $N_{\varepsilon_1}$ for $\alpha$ over $[0, H]$ with the size up to $|N_{\varepsilon_1}| \leq \frac{3H}{\varepsilon_1}$ \citep{vershynin2018high}. So applying the uniform concentration argument combined with \eqref{eq:V-p-phat-gap-one-alpha-bernstein} yields that for all $(\alpha, s,\ba)\in N_{\varepsilon_1} \times \cS\times \cA$, with probability at least $1-\delta$,
\begin{align}
\left| \left(P^{\no}_{h,s,\ba} - \widehat{P}^{\no}_{h,s,\ba} \right) [V]_\alpha\right| & \leq \sqrt{\frac{2\log \Big( \frac{2S \allA |N_{\varepsilon_1}| }{\delta} \Big)}{N}} \sqrt{\mathrm{Var}_{P^{\no}_{h,s,\ba}}(V)} +  \frac{2H \log \Big(\frac{2S \allA |N_{\varepsilon_1}| }{\delta} \Big)}{3N}. \label{eq:donkey} 
\end{align}

Inserting the above fact back to \eqref{eq:middle-key-lemma}, we arrive at: for all $(s, \ba) \in \cS \times \cA$,
\begin{align}
  \left|\pmin_{i,h,s, \ba}^{V} V - \pmhat_{i,h,s,\ba}^{V} V \right| &\leq  \max_{\alpha\in [0,H]} \left| P^\no_{h,s,\ba} \left[V\right]_{\alpha} - \widehat{P}^\no_{h,s,\ba} \left[V\right]_{\alpha}\right| \nonumber \\
  & \overset{\mathrm{(i)}}{\leq}  \sup_{\alpha\in \cN_{\varepsilon_1}}  \left| P^\no_{h,s,\ba} \left[V\right]_{\alpha} - \widehat{P}^\no_{h,s,\ba} \left[V\right]_{\alpha}\right| + \varepsilon_1  \notag \\
  & \overset{\mathrm{(ii)}}{\leq}  \sqrt{\frac{2\log \left(\frac{2S\allA |\cN_{\varepsilon_1}|}{\delta} \right)}{N}} \sqrt{\mathrm{Var}_{P^{\no}_{h,s,\ba}}(V)} +  \frac{2\log \left(\frac{2S\allA |\cN_{\varepsilon_1}|}{\delta} \right)H}{3N} + \varepsilon_1  \label{eq:V-p-phat-gap-one-alpha-bernstein-union-N-net} \\
  &\overset{\mathrm{(iii)}}{\leq} \sqrt{\frac{2\log \left(\frac{2S\allA|\cN_{\varepsilon_1}|}{\delta} \right)}{N}} \sqrt{\mathrm{Var}_{P^{\no}_{h,s,\ba}}(V)} +  \frac{\log \left(\frac{2S\allA|\cN_{\varepsilon_1}| }{\delta} \right)H }{N} \notag \\
  & \overset{\mathrm{(iv)}}{\leq} 2\sqrt{\frac{\log \left(\frac{18S \allA N}{\delta} \right)}{N}} \sqrt{\mathrm{Var}_{P^{\no}_{h,s,\ba}}(V)} +  \frac{\log \left(\frac{18S\allA N}{\delta} \right) H}{N} \label{eq:V-p-phat-gap-one-alpha-bernstein-union} \\
  & \leq 2\sqrt{\frac{\log \left(\frac{18S\allA N}{\delta} \right)}{N}} \|V\|_\infty  +  \frac{\log \left(\frac{18S\allA N}{\delta} \right) H}{N}   \notag \\
	&\leq 3\sqrt{\frac{H^2 \log \left(\frac{18S\allA N}{\delta} \right)}{N}}  \label{eq:V-p-phat-gap-one-alpha-hoeffding-union}
\end{align}
where (i) arises from the fact that the solution $\alpha^\star = \arg\max_{\alpha\in [0, H]} \left| P^\no_{h,s,\ba} \left[V\right]_{\alpha} - \widehat{P}^\no_{h,s,\ba} \left[V\right]_{\alpha}\right|$ falls into the $\varepsilon_1$-ball centered around some point inside $N_{\varepsilon_1}$ and $\left| P^\no_{h,s,\ba} \left[V\right]_{\alpha} - \widehat{P}^\no_{h,s,\ba} \left[V\right]_{\alpha}\right|$ is $1$-Lipschitz w.r.t. $\alpha$, (ii) holds by \eqref{eq:donkey}, (iii) follows from taking $\varepsilon_1 = \frac{\log(\frac{2S\allA|\cN_{\varepsilon_1}|}{\delta})H}{3N}$, (iv) is verified by $|\cN_{\varepsilon_1}| \leq \frac{3H}{\varepsilon_1} \leq 9N$, and the last inequality is due to the fact $\|V\|_\infty \leq H$ and letting $N \geq \log(\frac{18S\allA N}{\delta})$.

Invoking the matrix form (see \eqref{eq:inf-p-special-marl} and \eqref{eqn:ppivq-marl}) and applying the above result with $V = \widehat{V}_{i,h+1}^{\pi, \ror_i}$ for a union bound over all $(h,i,s,\ba)\in [H] \times [n] \times \cS\times \cA$, we complete the proof: with probability at least $1-\delta$,
\begin{align}
\forall (h,i)\in [H] \times [n]: \quad a_{i,h}^\pi &=\left|\Pv_{i,h}^{\pi, \widehat{V}} \widehat{V}_{i,h+1}^{\pi, \ror_i} - \Phatv_{i,h}^{\pi, \widehat{V}}\widehat{V}_{i,h+1}^{\pi, \ror_i} \right| \notag \\
&= \left|\Pi_h^{\pi}\pmin_{i,h}^{\pi, \widehat{V}} \widehat{V}_{i,h+1}^{\pi, \ror_i} - \Pi_h^{\pi}\pmhat_{i,h}^{\pi, \widehat{V}}\widehat{V}_{i,h+1}^{\pi, \ror_i} \right| \notag \\
& \overset{\mathrm{(i)}}{\leq} \Pi_h^{\pi}  \left|\pmin_{i,h}^{\pi, \widehat{V}} \widehat{V}_{i,h+1}^{\pi, \ror_i} - \pmhat_{i,h}^{\pi, \widehat{V}}\widehat{V}_{i,h+1}^{\pi, \ror_i} \right| \\
  &\leq  2\sqrt{\frac{\log(\frac{18S \allA nHN}{\delta})}{N}} \Pi_h^{\pi} \sqrt{\mathsf{Var}_{P_h^\no}(\widehat{V}_{i,h+1}^{\pi})} +  \frac{\log(\frac{18S\allA nHN}{\delta}) H}{N} 1 \notag \\
  &\overset{\mathrm{(ii)}}{\leq}  2\sqrt{\frac{\log(\frac{18S \allA nHN}{\delta})}{N}}  \sqrt{\mathsf{Var}_{\Pv_{h}^{\pi}}(\widehat{V}_{i,h+1}^{\pi})} +  \frac{\log(\frac{18S\allA nHN}{\delta}) H}{N} 1  \notag \\
  &\leq 3\sqrt{\frac{H^2\log(\frac{18S\allA nHN}{\delta})}{N}}1, \label{eq:key-concentration-bound}
\end{align}
where (i) and (ii) hold by the Jensen's inequality,  $\mathsf{Var}(\cdot)$ is defined in \eqref{eq:defn-variance-vector-marl}, and $P_h^\no, \Pv_{h}^{\pi}$ are defined in \eqref{eqn:ppivq-marl}.

\subsubsection{Proof of Lemma~\ref{lem:key-lemma-reduce-H}}\label{proof:lem:key-lemma-reduce-H}

In this section, we want to take the accessible range of the robust value function $\widehat{V}_{i,j+1}^{\pi, \ror_i}$ into consideration when controlling $\sum_{j=h}^H \left< d_h^j, \mathsf{Var}_{\Pv_{i,j}^{\pi, \widehat{V}}}(\widehat{V}_{i,j+1}^{\pi, \ror_i})  \right>$. Towards this, we introduce some auxiliary values and reward functions as below. For any time step $h\in[H]$ and the $i$-th agent:
\begin{itemize}
	\item $\widehat{V}_h^{\min} \defn \min_{s\in\cS} \widehat{V}_{i,h}^{\pi, \ror_i} (s)$: $\widehat{V}_h^{\min}$ denote the minimum value of all the entries in vector $\widehat{V}_{i,h}^{\pi, \ror_i}$.
	\item $\widehat{V}_h'\defn   \widehat{V}_{i,h}^{\pi, \ror_i} - \widehat{V}_h^{\min} 1 $: truncated value function.
	\item $\widehat{r}_{i,h}^{\min} = r_{i,h}^{\pi}  + \left( \widehat{V}_{h+1}^{\min} - \widehat{V}_{h}^{\min}\right) 1 $: truncated reward function.
\end{itemize}
With above notations, we introduce the following fact of $V_h'$:
\begin{align}
\widehat{V}_h' = \widehat{V}_{i,h}^{\pi, \ror_i} - \widehat{V}_h^{\min} 1 &\overset{\mathrm{(i)}}{=} r_{i,h}^{\pi} +\Phatv_{i,h}^{\pi, \widehat{V}} \widehat{V}_{i,h+1}^{\pi, \ror_i} -  \widehat{V}_h^{\min} 1 \nonumber \\
&= r_{i,h}^{\pi} + \Pv_{i,h}^{\pi, \widehat{V}} \widehat{V}_{i,h+1}^{\pi, \ror_i} +  \Big( \Phatv_{i,h}^{\pi, \widehat{V}} - \Pv_{i,h}^{\pi, \widehat{V}}\Big) \widehat{V}_{i,h+1}^{\pi, \ror_i} - \widehat{V}_h^{\min} 1\nonumber \\
&= r_{i,h}^{\pi}  + \left( \widehat{V}_{h+1}^{\min} - \widehat{V}_{h}^{\min}\right) 1  + \Pv_{i,h}^{\pi, \widehat{V}}  \widehat{V}_{h+1}' + \Big( \Phatv_{i,h}^{\pi, \widehat{V}} - \Pv_{i,h}^{\pi, \widehat{V}}\Big) \widehat{V}_{i,h+1}^{\pi, \ror_i}  \nonumber \\
& = \widehat{r}_{i,h}^{\min} + \Pv_{i,h}^{\pi, \widehat{V}}  \widehat{V}_{h+1}' + \Big( \Phatv_{i,h}^{\pi, \widehat{V}} - \Pv_{i,h}^{\pi, \widehat{V}}\Big) \widehat{V}_{i,h+1}^{\pi, \ror_i} , \label{eq:bellman-minus-vmin-vstar}
\end{align}
where (i) holds by the robust Bellman's consistency equation in \eqref{eq:r-bellman-matrix}.

With the above fact in hand, we can verify that
\begin{align}
	\mathsf{Var}_{\Pv_{i,h}^{\pi, \widehat{V}}}(\widehat{V}_{i,h+1}^{\pi, \ror_i})  &\overset{\mathrm{(i)}}{=} \mathrm{Var}_{\Pv_{i,h}^{\pi, \widehat{V}}}(\widehat{V}_{h+1}')  = \Pv_{i,h}^{\pi, \widehat{V}} \left(\widehat{V}_{h+1}' \circ \widehat{V}_{h+1}'\right) - \big(\Pv_{i,h}^{\pi, \widehat{V}} \widehat{V}_{h+1}'\big) \circ  \big(\Pv_{i,h}^{\pi, \widehat{V}} \widehat{V}_{h+1}' \big) \nonumber \\
	& \overset{\mathrm{(ii)}}{=} \Pv_{i,h}^{\pi, \widehat{V}} \left( \widehat{V}_{h+1}' \circ  \widehat{V}_{h+1}'\right)  - \Big( \widehat{V}_h' - \widehat{r}_{i,h}^{\min} - \Big( \Phatv_{i,h}^{\pi, \widehat{V}} - \Pv_{i,h}^{\pi, \widehat{V}}\Big) \widehat{V}_{i,h+1}^{\pi, \ror_i} \Big)^{\circ 2} \nonumber \\
	& = \Pv_{i,h}^{\pi, \widehat{V}} \left(\widehat{V}_{h+1}' \circ \widehat{V}_{h+1}'\right) -  \widehat{V}_h' \circ \widehat{V}_h' + 2 \widehat{V}_h' \circ \Big(\widehat{r}_{i,h}^{\min} + \Big( \Phatv_{i,h}^{\pi, \widehat{V}} - \Pv_{i,h}^{\pi, \widehat{V}}\Big) \widehat{V}_{i,h+1}^{\pi, \ror_i}\Big) \nonumber \\
	& \quad -   \Big(\widehat{r}_{i,h}^{\min} + \Big( \Phatv_{i,h}^{\pi, \widehat{V}} - \Pv_{i,h}^{\pi, \widehat{V}}\Big) \widehat{V}_{i,h+1}^{\pi, \ror_i} \Big)^{\circ 2} \nonumber \\
	& \overset{\mathrm{(iii)}}{\leq} \Pv_{i,h}^{\pi, \widehat{V}} \left(\widehat{V}_{h+1}' \circ \widehat{V}_{h+1}'\right) -  \widehat{V}_h' \circ \widehat{V}_h' + 2 \big\|\widehat{V}_h'\big\|_\infty \Big( 1 + \Big|\Big( \Phatv_{i,h}^{\pi, \widehat{V}} - \Pv_{i,h}^{\pi, \widehat{V}}\Big) \widehat{V}_{i,h+1}^{\pi, \ror_i} \Big| \Big)  \label{eq:variance-tight-bound-vstar-repeat} \\
	& \leq \Pv_{i,h}^{\pi, \widehat{V}} \left(\widehat{V}_{h+1}' \circ \widehat{V}_{h+1}'\right) -  \widehat{V}_h' \circ \widehat{V}_h' + 2 \big\|\widehat{V}_h' \big\|_\infty 1 + 6 \|V_h'\|_\infty   \sqrt{\frac{H^2\log \left(\frac{18S\allA nHN}{\delta} \right)}{N}}1 \label{eq:variance-tight-bound-vstar}
\end{align}
holds with probability at least $1- \delta$, where (i) follows from the fact that $\mathrm{Var}_{\Pv_{i,h}^{\pi, \widehat{V}}}(V - b 1) = \mathrm{Var}_{\Pv_{i,h}^{\pi, \widehat{V}}}(V)$ for any value vector $V\in\mathbb{R}^S$ and scalar $b$, (ii) holds by \eqref{eq:bellman-minus-vmin-vstar}, (iii) arises from $\widehat{r}_{i,h}^{\min}  \leq r_{i,h}^{\pi} \leq 1 $ since $ V_{h+1}^{\min} - V_{h}^{\min} \leq 0$ by definition, and the last inequality holds by \eqref{eq:key-concentration-bound}.

Finally, combining \eqref{eq:variance-tight-bound-vstar} and the definition of $d_h^j$ in \eqref{eq:defn-of-d}, the term of interest can be controlled as

\begin{align}
&\sum_{j=h}^H \left< d_h^j, \mathsf{Var}_{\Pv_{i,j}^{\pi, \widehat{V}}}(\widehat{V}_{i,j+1}^{\pi, \ror_i})  \right> \notag \\
& = \sum_{j=h}^H (d_h^j)^\top \left( \Pv_{i,j}^{\pi, \widehat{V}} \left(\widehat{V}_{j+1}' \circ \widehat{V}_{j+1}'\right) -  \widehat{V}_j' \circ \widehat{V}_j' + 2 \|\widehat{V}_j'\|_\infty 1 + 6 \|\widehat{V}_j'\|_\infty   \sqrt{\frac{H^2\log \left(\frac{18S\allA nHN}{\delta} \right)}{N}}1\right) \notag \\
& \overset{\mathrm{(i)}}{\leq}\sum_{j=h}^H \left[(d_h^j)^\top \left( \Pv_{i,j}^{\pi, \widehat{V}} \left(\widehat{V}_{j+1}' \circ \widehat{V}_{j+1}'\right) -  \widehat{V}_j' \circ \widehat{V}_j'\right) \right]+ 2H\|\widehat{V}_h'\|_\infty + 6H^2\|\widehat{V}_h'\|_\infty\sqrt{\frac{\log \left(\frac{18S\allA nHN}{\delta} \right)}{N}} \notag \\
& = \sum_{j=h}^H \left[ (d_h^{j+1})^\top  \left(\widehat{V}_{j+1}' \circ \widehat{V}_{j+1}'\right) -  (d_h^{j})^\top \left(\widehat{V}_j' \circ \widehat{V}_j' \right) \right]+ 2H\|\widehat{V}_h'\|_\infty + 6H^2\|\widehat{V}_h'\|_\infty\sqrt{\frac{\log \left(\frac{18S\allA nHN}{\delta} \right)}{N}} \notag \\
& \leq \left\|d_h^{H+1} \right\|_1  \left\|\widehat{V}_{H+1}' \circ \widehat{V}_{H+1}'\right\|_\infty + 2H\|\widehat{V}_h'\|_\infty + 6H^2\|\widehat{V}_h'\|_\infty\sqrt{\frac{\log \left(\frac{18S\allA nHN}{\delta} \right)}{N}} \notag \\
&\leq 3H\|\widehat{V}_h'\|_\infty + 6H^2\|\widehat{V}_h'\|_\infty\sqrt{\frac{\log \left(\frac{18S\allA nHN}{\delta} \right)}{N}} \notag \\
& = 3H\|\widehat{V}_h'\|_\infty \left(1 + 2H\sqrt{\frac{\log(\frac{18S\allA nHN}{\delta})}{N}} \right),
\end{align}
where (i) holds by the fact $\|\widehat{V}_h'\|_\infty \geq \|\widehat{V}_{h+1}'\|_\infty \geq \cdots \geq \|\widehat{V}_{H}'\|_\infty $ and basic calculus.

\subsubsection{Proof of Lemma~\ref{eq:extra-lemma1}}\label{proof:eq:extra-lemma1}

We start with the proof about the empirical MG $\mathcal{MG}_{\mathsf{rob}}$.
To begin with, for any policy $\pi$ and the $i$-th agent, we define
\begin{align}
\forall h\in[H]: \quad V_{i,h}^{\mathsf{span}} \defn \widehat{V}_{i,h}^{\pi,\ror_i} - \min_{s'\in\cS} \widehat{V}_{i,h}^{\pi,\ror_i}(s') 1,
\end{align}
which leads to 
\begin{align}
 \left\|V_{i,h}^{\mathsf{span}} \right\|_\infty \leq  \min \left\{\frac{1}{\ror_i}, H-h+1 \right\}.
\end{align}
which holds by applying Lemma~\ref{lemma:pnorm-key-value-range}.

Armed with above notation and facts, considering any transition kernel $P'\in\mathbb{R}^S$ and any $\widetilde{P} \in\mathbb{R}^S$ obeying $\widetilde{P}\in \cU^{\ror_i}(P')$, we have for all $(i,h)\in [n]\times [H] $
\begin{align}
  \big|\mathsf{Var}_{P'}(\widehat{V}_{i,h}^{\pi,\ror_i}) - \mathsf{Var}_{\widetilde{P}}(\widehat{V}_{i,h}^{\pi,\ror_i}) \big|  & =   \big|\mathsf{Var}_{P'}(V_{i,h}^{\mathsf{span}}) - \mathsf{Var}_{\widetilde{P}}(V_{i,h}^{\mathsf{span}}) \big|  \notag \\
&   \leq  \big\|\widetilde{P}- P' \big\|_1\big\|V_{i,h}^{\mathsf{span}}\big\|_\infty \notag \\
& \leq    \ror_i  \left(\min \left\{\frac{1}{\ror_i}, H-h+1 \right\} \right)^2 \leq \min \left\{\frac{1}{\ror_i}, H-h+1 \right\}. \label{eq:vmax-sigma-big}
\end{align}

Similar facts can be verified for standard MG $\mathcal{MG}$ analogously.

\subsubsection{Proof of Lemma~\ref{lem:key-lemma-reduce-H-2}}\label{proof:lem:key-lemma-reduce-H-2}

Analogous to Appendix~\ref{proof:lem:key-lemma-reduce-H}, we introduce some auxiliary values and reward functions to control 
$$\sum_{j=h}^H \left< w_h^j, \mathsf{Var}_{\Pv_{i,j}^{\pi, V}}(V_{i,j+1}^{\pi,\ror_i})  \right>$$ as below: for any time step $h$ and the $i$-th agent
\begin{itemize}
	\item $V_h^{\min} \defn \min_{s\in\cS} V_{i,h}^{\pi,\ror_i} (s)$: $V_h^{\min}$ denote the minimum value of all the entries in vector $V_{i,h}^{\pi,\ror_i}$.
	\item $V_h'\defn   V_{i,h}^{\pi,\ror_i} - V_h^{\min} 1 $: truncated value function.
	\item $ r_{i,h}^{\min} = r_{i,h}^{\pi}  + \left( V_{h+1}^{\min} - V_{h}^{\min}\right) 1 $: truncated reward function.
\end{itemize}
 
Then applying the robust Bellman's consistency equation in \eqref{eq:Q-value-mg-matrix} gives
\begin{align}
V_h' = V_{i,h}^{\pi,\ror_i} - V_h^{\min} 1 &=r_{i,h}^{\pi} +\Pv_{i,h}^{\pi, V} V_{i,h+1}^{\pi,\ror_i} -  V_h^{\min} 1 \nonumber \\
&= r_{i,h}^{\pi} + \left( V_{h+1}^{\min} - V_{h}^{\min}\right) 1 + \Pv_{i,h}^{\pi, V} V_{h+1}' = r_{i,h}^{\min} + \Pv_{i,h}^{\pi, V} V_{h+1}'. \label{eq:bellman-minus-vmin-vstar2}
\end{align}

The above fact leads to
\begin{align}
	\mathsf{Var}_{\Pv_{i,h}^{\pi, V}}(V_{i,h+1}^{\pi,\ror_i})  &\overset{\mathrm{(i)}}{=} \mathrm{Var}_{\Pv_{i,h}^{\pi, V}}(V_{h+1}')  = \Pv_{i,h}^{\pi, V} \left(V_{h+1}' \circ V_{h+1}'\right) - \big(\Pv_{i,h}^{\pi, V} V_{h+1}'\big) \circ  \big(\Pv_{i,h}^{\pi, V} V_{h+1}' \big) \nonumber \\
	& \overset{\mathrm{(ii)}}{=} \Pv_{i,h}^{\pi, V} \left(V_{h+1}' \circ V_{h+1}'\right)  - \Big(V_h' - r_{i,h}^{\min} \Big)^{\circ 2} \nonumber \\
	& = \Pv_{i,h}^{\pi, V} \left(V_{h+1}' \circ V_{h+1}'\right) -  V_h' \circ V_h' + 2 V_h' \circ r_{i,h}^{\min} -  r_{i,h}^{\min} \circ r_{i,h}^{\min} \nonumber \\
	& \leq \Pv_{i,h}^{\pi, V} \left(V_{h+1}' \circ V_{h+1}'\right) -  V_h' \circ V_h' + 2 \|V_h'\|_\infty 1   \label{eq:variance-tight-bound-vstar2}
\end{align}
where (i) follows from the fact that $\mathrm{Var}_{\Pv_{i,h}^{\pi, V}}(V - b 1) = \mathrm{Var}_{\Pv_{i,h}^{\pi, \widehat{V}}}(V)$ for any value vector $V\in\mathbb{R}^S$ and scalar $b$, (ii) holds by \eqref{eq:bellman-minus-vmin-vstar2}, and the last inequality arises from $r_{i,h}^{\min}  \leq r_{i,h}^{\pi} \leq 1 $ since $ V_{h+1}^{\min} - V_{h}^{\min} \leq 0$ by definition.

Consequently, combining \eqref{eq:variance-tight-bound-vstar2} and the definition of $w_h^j$ in \eqref{eq:defn-of-d-2}, we arrive at
\begin{align}
&\sum_{j=h}^H \Big< w_h^j, \mathsf{Var}_{\Pv_{i,j}^{\pi, V}} \left(V_{i,j+1}^{\pi, \ror_i} \right)  \Big> \notag \\
& = \sum_{j=h}^H (w_h^j)^\top \left( \Pv_{i,j}^{\pi, V} \left(V_{j+1}' \circ V_{j+1}'\right) -  V_j' \circ V_j' + 2 \|V_h'\|_\infty 1 \right) \notag \\
& \overset{\mathrm{(i)}}{\leq}\sum_{j=h}^H \left[(w_h^j)^\top \left( \Pv_{i,j}^{\pi, V} \left(V_{j+1}' \circ V_{j+1}'\right) -  V_j' \circ V_j'\right) \right]+ 2H\|V_h'\|_\infty  \notag \\
& = \sum_{j=h}^H \left[ (w_h^{j+1})^\top  \left(V_{j+1}' \circ V_{j+1}'\right) -  (w_h^{j})^\top \left(V_j' \circ V_j' \right) \right]+ 2H\|V_h'\|_\infty \notag \\
& \leq \|w_h^{H+1}\|_1  \left\|V_{H+1}' \circ V_{H+1}'\right\|_\infty + 2H\|V_h'\|_\infty  \notag \\
&\leq 3H\|V_h'\|_\infty,
\end{align}
where (i) and the last inequality hold by the fact $\|V_h'\|_\infty \geq \|V_{h+1}'\|_\infty \geq \cdots \geq \|V_{H}'\|_\infty $ and basic calculus.

\section{Proof of Theorem~\ref{thm:robust-mg-lower-bound}}\label{proof:thm:robust-mg-lower-bound}

In this section, the proof will focus on a special and simpler class of \rmgs: distributionally robust Markov decision processes (RMDPs) --- single-agent \rmgs.

Before proceeding, to keep self-contained, we first briefly introduce the definition of a RMDP in finite-horizon episodic setting. Recall that a multi-agent general-sum robust Markov games (\rmg) with TV uncertainty set can be represented as $\mathcal{MG} = \big\{ \cS, \{\cA_i\}_{1 \le i \le n},\{\cU^{\ror_i}(P^\no)\}_{1 \le i \le n}, \rew,  H \big\}$. Resorting to the same notations for \rmgs, a finite-horizon episodic distributionally robust MDP (RMDP) can be represented as $\mathcal{M}_{\mathsf{rob}}= \big(\mathcal{S},\mathcal{A}_1, \unb^{\ror_1}(P^{\no}), \{r_{1,h}\}_{1\leq h\leq H}, H\big)$, i.e., let $n=1$. Then we can show an essential fact between \rmgs and RMDPs that allow us to turn to RMDPs for proving Theorem~\ref{thm:robust-mg-lower-bound}.
Without loss of generality, we consider the class of \rmgs with $n$ players that obey $|\cA_1| \geq \max\{|\cA_2|, \cdots, |\cA_m|\}$. Moreover, let $|\cA_2| = |\cA_3| = \cdots =|\cA_m| = 1$ for simplicity, which leaves those agents' ($i=2,3,\cdots, n$) choices of actions having no randomness or effects on the transitions or rewards for any agents. Consequently, it is clear that finding a robust NE/CE/CCE of such \rmgs degrades to finding the optimal policy of the first agent over a corresponding RMDP $\mathcal{M}_{\mathsf{rob}} = \big\{ \cS, \cA_1, \cU^{\ror_1}(P^\no), \{r_{1,h}\}_{1\leq h\leq H},  H \big\}$.

Therefore, in this section, we turn to construct the lower bound for finding the optimal policy over RMDPs instead, which directly imply a lower bound for finding equilibriums (robust NE/CE/CCE) of \rmgs.

Before continuing, we make note of the following useful property about the KL divergence in \citet[Lemma~2.7]{tsybakov2009introduction} which is useful in this section.
\begin{lemma}
    \label{lem:KL-key-result}
    For any $p, q \in (0,1)$, it holds that 
    \begin{align}
      \mathsf{KL}(p \parallel q) \leq \frac{(p-q)^2}{q(1-q)}. \label{eq:KL-dis-key-result} 
    \end{align}
\end{lemma}

\subsection{Constructing hard robust MDP instances}

The hard instances developed here are different from standard  MDP since we need to consider that the transition kernel can be perturbed in robust MDPs. This is the first lower bound for robust MDPs in episodic setting.

\paragraph{Step 1: constructing hard robust MDP instances.}

To begin with, we first introduce an auxiliary collection $\Theta \subseteq \{0,1\}^{H}$, consisting of $H$-dimensional vectors. In addition, resorting to the Gilbert-Varshamov lemma \citep{gilbert1952comparison}, we notice that there exists a set $\Theta \subseteq \{0, 1\}^{H}$ such that:%
\begin{equation}
      \text{for any }\theta,\widetilde{\theta}\in \Theta \text{ obeying }\theta \ne \widetilde{\theta}: \quad    \|\theta - \widetilde{\theta}\|_1 \ge \frac{H}{8}
     \qquad \text{and} \qquad 
    |\Theta| \ge e^{H/8}.
    \label{eq:property-Theta} 
\end{equation}

Without loss of generality, we denote the first component of $\Theta$ as $\theta^{\mathsf{base}}$ and denote $\Theta^\star$ as $\Theta \setminus \{\theta^{\mathsf{base}}\}$. 
With this in mind, we construct a set of RMDPs as below:
\begin{equation}
    \cM(\cW, \Theta) \defn \left\{ \mathcal{M}_{w}^{\theta} = 
    \big(\mathcal{S}, \mathcal{A},  \unb^\ror(P^{w,\theta}), \{r_h\}_{h=1}^H, H \big) 
    \mid w \in \cW =\{0,1,\cdots, SA-1\}, \theta = [\theta_h]_{1\leq h\leq H} \in \Theta^\star
    \right\}, \label{eq:theta-class}
\end{equation}
where
\begin{align*}
    \cS = \{0, 1, \ldots, S-1\}, 
    \qquad \text{and} \qquad  \mathcal{A} = \{0, 1,\cdots, A-1\},
\end{align*}
and $\ror$ will be introduced momentarily.

In words, the collection of $\cM(\cW, \Theta)$ consists of $|\cW| = SA$ subsets, with each includes $|\Theta^\star|$ different RMDPs associated with some $w \in \cW$. The state space of each RMDP $\cM_w^\theta \in \cM(\cW, \Theta)$ is denoted as $\cS_{\cM}$, includes two classes of states $\cX = \{x_{i} \mid i\in \cW\}$ and $\cY = \{y_{i} \mid i\in\cW\}$. Each state in $\cX$ and $\cY$ only have two possible actions $\cA_{\cM} = \{0,1\}$. So we have totally $2|\cW| = 2SA$ states and there is in total $|\cS_{\cM}||\cA_{\cM}| = 4SA$ state-action pairs.

We shall define the nominal transition kernels for $\cM(\cW, \Theta)$, where any state $x_{i} \in \cX$ only transits to the corresponding $y_{i}\in\cY$ or itself. For convenience, for any $s = x_{i}\in \cX$, we denote the corresponding state $y_{i} \in \cY$ as $s^{x\rightarrow y}$. 

Armed with above notations, we define a basic nominal transition kernel associated with $\theta^{\mathsf{base}}$ as below: for all $(h,s,a)\in [H] \times \cS_{\cM} \times \cA_{\cM}$,
\begin{align}
P^{\star}_h(s^{\prime} \mymid s, a) = \left\{ \begin{array}{lll}
         (p+\Delta)\mathds{1}(s^{\prime} = s^{x \rightarrow y}) + (1-p-\Delta)\mathds{1}(s^{\prime} = s)  & \text{if} & s \in \cX, a = \theta^{\mathsf{base}}_h \\
          p\mathds{1}(s^{\prime} = s^{x \rightarrow y}) + (1-p)\mathds{1}(s^{\prime} = s)  & \text{if} & s\in \cX, a = 1-\theta^{\mathsf{base}}_h \\
          \ind(s'=s) & \text{if} & s\in \cY.
                \end{array}\right.
        \label{eq:Ph-construction-lower-bound-finite-theta}
\end{align}

In addition, for any RMDP $\cM_{w}^\theta \in \cM(\cW, \Theta)$, the transition kernel $P^{w,\theta} = \{P^{w,\theta}_{h}\}_{h=1}^H$ is specified as follows: for any $(s,a,s',h)\in \cS_{\cM}\times \cA_{\cM} \times \cS_{\cM}\times [H]$,
\begin{align}
P^{w,\theta}_{h}(s^{\prime} \mymid s, a) =
\left\{ \begin{array}{lll}
         p\mathds{1}(s^{\prime} = y_{w}) + (1-p)\mathds{1}(s^{\prime} = s)  & \text{if} \quad  s= x_{w}, a = \theta_h \\
          q\mathds{1}(s^{\prime} = y_w) + (1-q)\mathds{1}(s^{\prime} = s)  & \text{if} \quad s= x_{w}, a = 1-\theta_h \\
          P^{\star}_h(s^{\prime} \mymid s, a) & \text{otherwise} 
                \end{array}\right.
        \label{eq:Ph-construction-lower-bound-finite-theta2}
\end{align}
Here,
 $p$ and $q$ are set according to
\begin{align}\label{eq:p-q-defn-infinite}
   0 \leq p \leq p +\Delta \leq 1 \quad \text{ and } \quad 0\leq q = p - \Delta
\end{align}
for some $p$ and $\Delta>0$ that will be introduced momentarily.
In words, the transition kernel of each $\cM_{w}^\theta \in \cM(\cW, \Theta)$ only differs slightly from the basic nominal transition kernel $P^{\star}_{h}$ when $s= x_w$, which makes all the components within $\cM(\cW, \Theta)$ closed to each other.

To continue, the reward function is defined as
\begin{align}
\forall (h, s, a) \in[H] \times \cS_{\cM} \times  \{0,1\}: \quad r_h(s, a) = \left\{ \begin{array}{lll}
         1 & \text{if } s \in \cY \\
         0 & \text{otherwise}.
                \end{array}\right.
        \label{eq:rh-construction-lower-bound-infinite}
\end{align}

\paragraph{Uncertainty set of the transition kernels.}
Denote the transition kernel vector as
\begin{align}
  \forall (h, s, a) \in[H] \times \cS_{\cM} \times \{0,1\}: \quad P_{h,s,a}^{w,\theta} \defn P^{w,\theta}_h(\cdot \mymid s,a) \in \Delta(\cS).
\end{align}
Recalling the uncertainty set defined in \eqref{eq:sa-rec-defn}, we know $\cU^{\ror}(P^{w,\theta})$ represents:
\begin{align}
\unb^{\ror}(P^{w,\theta}) \defn  \otimes \; \cU^{\ror}(P^{w,\theta}_{h,s,a}),\qquad &\cU^{\ror}(P^{w,\theta}_{h,s,a}) \defn \Big\{ \widetilde{P}^{w,\theta}_{h,s,a} \in \Delta (\cS): \frac{1}{2} \big\|  \widetilde{P}^{w,\theta}_{h,s,a} - P^{w,\theta}_{h,s,a}\big\|_1 \leq \ror \Big\}, \label{eq:tv-ball-infinite-P-recall1}
\end{align}
where $\otimes$ represents the Cartesian product over $(h,s,a)\in [H] \times \cS_{\cM} \times \cA_{\cM}$.

For such TV uncertainty set, without loss of generality, let the uncertainty level to be $\ror \in (0, 1-c_0]$ for some $0< c_0 < 1$. Then taking $c_2 \leq \frac{1}{4}$ amd $c_1 \defn \frac{c_0}{2} \leq \frac{1}{4}$,  $p$ and $\Delta$ are set as 
\begin{align}\label{eq:p-q-defn-infinite2}
    p =  \begin{cases} \frac{c_2}{H}, &\text{if } \ror\leq \frac{c_2}{2H} \\
    \left( 1 + \frac{c_1}{H} \right) \sigma & \text{otherwise}
    \end{cases}
      \qquad \text{and} \qquad
      \Delta \leq \begin{cases}  \frac{c_2}{2H}, &\text{if } \ror\leq \frac{c_2}{2H} \\
      \frac{c_1}{H} \sigma & \text{otherwise}
      \end{cases}
\end{align}
Combined with $H\geq 2$, it is easily verified that $ 0 \leq p + \Delta \leq 1$ as follows:
\begin{align}\label{eq:tv-1-p-bound}
&\text{when $\ror > \frac{c_2}{2H}$}: \quad \left(1 + \frac{c_1}{H} \right)\ror + \frac{c_1}{H} \ror  \leq 1- c_0 + \frac{2c_1}{H} \ror \leq 1-\frac{c_0(H-1)}{H} < 1,  \notag \\
& \text{when $\ror\leq \frac{c_2}{2H}$}: \quad
\frac{3c_2}{2H} \leq 1.
\end{align}

Then we introduce some useful notations and facts throughout this section. First, for any RMDP $\cM_{w}^\theta \in \cM(\cW, \Theta)$ and any $(h,s,a,s')\in [H] \times \cS_{\cM}\times \cA_{\cM} \times \cS_{\cM}$, we denote the minimum probability of transiting from $(s,a)$ to $s'$ determined by any perturbed transition kernel $P_{h,s,a} \in \unb^{\ror}(P^{w,\theta}_{h,s,a})$ as  
\begin{align}\label{eq:infinite-lw-def-p-q}
\underline{P}_h^{w,\theta}(s' \mymid s,a) &\defn \inf_{P_{h,s,a} \in\unb^{\ror}(P^{w,\theta}_{h,s,a}) } P_h(s'  \mymid s,a) = \max \{P_h(s'  \mymid s,a) - \ror, 0\},
\end{align}
where the last equation can be easily verified by the definition of $\cU^{\ror}(\cdot)$ in \eqref{eq:tv-ball-infinite-P-recall1} and distributing the probability on $s'$ to other states.

Especially, for convenience, we denote the transition from each $s \in \cX$ to the corresponding state $s^{x \rightarrow y} \in \cY$ of any $\cM_w^\theta$ as below, which plays an important role in the analysis: for all $h\in [H]$,
\begin{align}\label{eq:infinite-lw-p-q-perturb-inf}
\text{for } x_w: \quad  &\underline{p}_h \defn \underline{P}_h^{w,\theta}(y_w \mymid x_w, \theta_h) = p - \ror ,\qquad \underline{q}_h  \defn \underline{P}_h^{w,\theta}(y_w \mymid x_w, 1 - \theta_h)  = q - \ror, \notag \\
\text{for } s\in \cX \setminus \{x_w\}: \quad  & \underline{p}_h' \defn \underline{P}_h^{w,\theta}(s^{x\rightarrow y} \mymid s, \theta^{\mathsf{base}}_h) = p +\Delta- \ror,\qquad \underline{q}'_h  \defn \underline{P}_h^{w,\theta}(s^{x\rightarrow y} \mymid s, 1 - \theta^{\mathsf{base}}_h)  = p - \ror,
\end{align}
which follows from the following fact that is clear from \eqref{eq:p-q-defn-infinite2}
\begin{align}\label{eq:infinite-p-q-bound}
    p + \Delta \geq p \geq q = p -\Delta \geq  \max \left\{ \frac{c_2}{2H}, \sigma\right\}.
\end{align}

Then it is obvious that
\begin{align}
    \underline{p}_1 = \underline{p}_2 = \cdots \underline{p}_H, \quad \underline{q}_1 = \underline{q}_2 = \cdots \underline{q}_H, \quad \underline{p}'_1 = \underline{p}'_2 = \cdots \underline{p}'_H, \quad \underline{q}'_1 = \underline{q}'_2 = \cdots \underline{q}'_H, \label{eq:finite-lw-upper-p-q-theta-0}
\end{align}
which motivates us to abbreviate them consistently as $\underline{p} \defn \underline{p}_1$, $\underline{q} \defn \underline{q}_1$, $\underline{p}' \defn \underline{p}'_1$, and $\underline{q}' \defn \underline{q}'_1$ later.

\paragraph{Robust value functions and optimal policies.}
Now we are ready to characterize the corresponding robust value functions and  identify the optimal policies for RMDP instances. With abuse of notations, for any RMDP $\cM_w^\theta \in \cM(\cW, \Theta)$, we denote $\pi^{\star,w, \theta} = \{\pi^{\star,w, \theta}_h\}_{h=1}^H$ as the optimal policy. In addition, at each step $h$, we let $V_{h}^{\pi,\ror, w, \theta}$ (resp.~$V_{h}^{\star, \ror, w,  \theta}$) represent the robust value function of any policy $\pi$ (resp.~$\pi^{\star, w, \theta}$) with uncertainty level $\ror$. 
Armed with these notations, the following lemma shows some essential properties concerning the robust value functions and optimal policies; the proof is postponed to Appendix~\ref{proof:lem:finite-lb-value-theta}.
\begin{lemma}\label{lem:finite-lb-value-theta}
Consider any $\cM_w^\theta \in \cM(\cW, \Theta)$ and any policy $\pi$. Defining
\begin{align}
x_h^{\pi, w, \theta} = \underline{p} \pi_h(\theta_h\mymid x_w) + \underline{q} \pi_h(1-\theta_h\mymid x_w),\label{eq:finite-x-h-theta}
\end{align}
it holds that
\begin{subequations}
\begin{align}
  \forall h\in[H]: \quad& V_{h}^{\pi,\ror, w,\theta}(x_w) =x_h^{\pi, w,\theta} V_{h+1}^{\pi,\ror, w,\theta}(y_w) + (1-x_h^{\pi, w,\theta}) V_{h+1}^{\pi,\ror, w,\theta}(x_w), \label{eq:finite-lemma-value-0-pi-theta} \\
  \forall (s,h)\in \cY \times [H]: \quad&  V_{h}^{\pi,\ror, w,\theta}(s) = 1 + (1-\ror) V_{h+1}^{\pi,\ror, w,\theta}(s) + \ror  V_{h+1}^{\pi,\ror, w,\theta}(x_w).  \label{eq:finite-lemma-value-0-pi-theta-yw} 
\end{align}
\end{subequations}
In addition, for all $h\in[H]$, the optimal policy and the optimal value function obey
\begin{subequations}
    \label{eq:finite-lb-value-lemma-theta}
\begin{align}
  \pi_h^{\star,w,\theta}(\theta_h \mymid x_w) &= \pi_h^{\star,w,\theta}(\theta_h \mymid y_w) =  1, \notag \\
    \pi_h^{\star,w,\theta}(\theta^{\mathsf{base}}_h \mymid s) &= \pi_h^{\star,w,\theta}(\theta^{\mathsf{base}}_h \mymid s^{x\rightarrow y}) =  1,  \quad \forall s\in\cX \setminus \{x_w\}
\end{align}
\end{subequations}

and 
\begin{align}
 V_{h}^{\star,\ror, w,\theta}(x_w) = \underline{p} V_{h+1}^{\pi,\ror, w,\theta}(y_w) + (1-\underline{p}) V_{h+1}^{\pi,\ror, w,\theta}(x_w) . \label{eq:finite-lemma-value-0-pi-theta2} 
\end{align}

\end{lemma}

\subsection{Establishing the lower bound}

Recall our goal: for any policy estimator $\widehat{\pi}$ computed based on the dataset with $N$ samples, we plan to control the quantity
\begin{align}
\max_{(w,\theta)\in\cW\times \Theta^\star}\max_{s\in \cX\cup \cY} \left\{ V_{1}^{\star,\ror, w,\theta}(s) - V_{1}^{\widehat{\pi},\ror, w,\theta}(s) \right\} \geq \max_{(w,\theta)\in\cW\times \Theta^\star}\max_{s\in \cX} \left\{ V_{1}^{\star,\ror, w,\theta}(s) - V_{1}^{\widehat{\pi},\ror, w,\theta}(s) \right\}. \label{eq:lower-bound-goal}
\end{align}

\paragraph{Step 1: converting the goal to estimate $(w,\theta)$.}
Towards this, we make the following  essential claim which shall be verified in Appendix~\ref{proof:eq:tv-Value-0-recursive}: letting
\begin{align}
 \varepsilon \leq \begin{cases} \frac{c_2}{H}, &\text{if } \ror\leq \frac{c_2}{2H} \\
    1 & \text{otherwise}
    \end{cases}
\end{align}
and
\begin{align}\label{eq:Delta-chosen}
    \Delta  =  c_5\begin{cases} \frac{\varepsilon}{H^2}, &\text{if } \ror\leq \frac{c_2}{2H} \\
    \frac{\ror\varepsilon}{H}  & \text{otherwise}
    \end{cases} 
\end{align}
which satisfies \eqref{eq:p-q-defn-infinite2},
it leads to that for any policy $\pi$ obeying 
\begin{align}
    \sum_{h=1}^H \big\|\widehat{\pi}_h(\cdot\mymid x_w) - \pi_h^{\star, w,\theta}(\cdot\mymid x_w) \big\|_1 \ge \frac{H}{8}, \label{eq:theta-assumption}
\end{align}
 one has
\begin{align}
    V_{1}^{\star,\ror, w,\theta}(x_w) - V_{1}^{\widehat{\pi},\ror, w,\theta}(x_w)  > \varepsilon. \label{eq:lower-bound-assumption}
\end{align}

Now we are ready to convert the estimation of an optimal policy to estimate $(w,\theta)$. Towards this, we denote $\mathbb{P}_{w,\theta}$ as the probability distribution when the RMDP is $\mathcal{M}_w^\theta$ for any $(w,\theta) \in \cW\times \Theta^\star$. In addition, we represent the subset of $\cM(\cW,\Theta)$ excluding the ones associated with some $w\in\cW$ as below:
\begin{align}
\cG_{-w} \defn \cW  \setminus \{w\} \times  \Theta^\star.
\end{align}

Then, for any $(w,\theta) \in \cW \times \Theta^\star$, suppose there exists a policy $\widehat{\pi}$ that achieves
\begin{align}
    \mathbb{P}_{w,\theta} \left\{ V_{1}^{\star,\ror, w,\theta}(x_w) - V_{1}^{\widehat{\pi},\ror, w,\theta}(x_w) \leq \varepsilon\right\} \geq \frac{3}{4},
    \label{eq:assumption-theta-small-LB-finite}
\end{align}
which in view of \eqref{eq:lower-bound-assumption} indicates that we necessarily have
\begin{align}
      \mathbb{P}_{w,\theta} \left\{  \sum_{h=1}^H \big\|\widehat{\pi}_h(\cdot\mymid x_w) - \pi_h^{\star, w,\theta}(\cdot\mymid x_w) \big\|_1 < \frac{H}{8} \right\} \geq \frac{3}{4}.
\end{align}

Consequently, taking $ \widetilde{\theta} =\arg\min_{\theta\in\Theta}  \sum_{h=1}^H \big\|\widehat{\pi}_h(\cdot\mymid x_w) - \pi_h^{\star, w,\theta}(\cdot\mymid x_w) \big\|_1 $, we are motivated to construct the following estimate of $(w,\theta)$:
\begin{align}
    \big(\widehat{w},\widehat{\theta} \big)  \begin{cases} = (w, \widetilde{\theta}) & \quad \text{ if } \quad \widetilde{\theta}\in\Theta^\star \\
     \in \cG_{-w}  & \quad \text{ if } \quad \widetilde{\theta} = \Theta \setminus \Theta^\star = \theta^{\mathsf{base}}.
    \end{cases}
    \label{eq:defn-theta-hat-inf-LB}
\end{align}

Then let us focus on the first kind of scenarios in \eqref{eq:defn-theta-hat-inf-LB} when $\widetilde{\theta}\in\Theta^\star$ so that we have the hope to estimate $(w,\theta)$ correctly. Namely, if  $\sum_{h=1}^H \big\|\widehat{\pi}_h(\cdot\mymid x_w) - \pi_h^{\star, w,\theta}(\cdot\mymid x_w) \big\|_1 < \frac{H}{8}$ holds for some $\theta\in\Theta^\star$, then for any $\theta' \in \Theta^\star$ obeying $\theta' \neq \theta$, one has
\begin{align}
\sum_{h=1}^H \big\|\widehat{\pi}_h(\cdot\mymid x_w) - \pi_h^{\star, w,\theta'}(\cdot\mymid x_w) \big\|_1 & \geq \sum_{h=1}^H \big\|\pi^{\star,w,\theta}_h(\cdot\mymid x_w) - \pi_h^{\star, w,\theta'}(\cdot\mymid x_w) \big\|_1 - \sum_{h=1}^H \big\|\widehat{\pi}_h(\cdot\mymid x_w) - \pi_h^{\star, w,\theta}(\cdot\mymid x_w) \big\|_1 \notag \\
& > \frac{H}{4} - \frac{H}{8} = \frac{H}{8},
\end{align}
where the first inequality holds by the triangle inequality, and the last inequality follows from the assumption $\sum_{h=1}^H \big\|\widehat{\pi}_h(\cdot\mymid x_w) - \pi_h^{\star, w,\theta}(\cdot\mymid x_w) \big\|_1 < \frac{H}{8}$ and the separation property of $\theta\in\Theta$ (see  \eqref{eq:property-Theta}).
Similarly,
It shows that we have $(\widehat{w},\widehat{\theta}) = (w, \theta)$ if
\begin{align}
	\sum_{h=1}^H \big\|\widehat{\pi}_h(\cdot\mymid x_w) - \pi_h^{\star, w,\theta}(\cdot\mymid x_w) \big\|_1 < \frac{H}{8} < \sum_{h=1}^H \big\|\widehat{\pi}_h(\cdot\mymid x_w) - \pi_h^{\star, w,\theta'}(\cdot\mymid x_w) \big\|_1 
\end{align}
holds for all $(w', \theta')\in \cW\times \Theta $ that $(w', \theta')\neq (w, \theta)$. It is clear that the above equation can be directly achieved when $\sum_{h=1}^H \big\|\widehat{\pi}_h(\cdot\mymid x_w) - \pi_h^{\star, w,\theta}(\cdot\mymid x_w) \big\|_1 < \frac{H}{8}$, which gives
\begin{align}
	\mathbb{P}_{w,\theta} \left[(\widehat{w},\widehat{\theta}) = (w,\theta)\right] \geq  \mathbb{P}_{w,\theta} \left\{  \sum_{h=1}^H \big\|\widehat{\pi}_h(\cdot\mymid x_w) - \pi_h^{\star, w,\theta}(\cdot\mymid x_w) \big\|_1 < \frac{H}{8} \right\} \geq \frac{3}{4}. \label{eq:consequence-of-sum-H}
\end{align}

\paragraph{Step 2: developing the probability of error in testing multiple hypotheses.}

Before proceeding, we discuss the data generation choices of the dataset $\cD$.
Recall that each RMDP inside the set $\cM(\cW,\Theta)$ under testing has two classes of states $\cX$ and $\cY$, with each has $|\cW| = SA$ components. Noticing that accordingly, $\cM(\cW,\Theta)$ consists of $|\cW|$ subset, with each $\{ \cM_w^\theta \}_{\theta\in\Theta^\star}$ constructed symmetrically around one pair of state $(x_w, y_w) \in \cX \times \cY$, respectively. Therefore, at each time step $h$, it is clear that the dataset are supposed to be generated uniformly by the transition kernels on each pair of states $(x_w, y_w) \in \cX \times \cY$ to maximize the information gain. Namely, the dataset $\cD$ has in total $\frac{N}{|\cW|H} = \frac{N}{SAH}$ samples for the two states $(x_w, y_w) \in \cX \times \cY$ at each time step $h\in[H]$.

Now we turn to the hypothesis testing problem over $(w,\theta) \in \cW \times \Theta^\star$. We shall develop  the information theoretical lower bound for the probability of error. In particular, we consider the minimax probability of error defined as follows:
\begin{equation}
    p_{\mathrm{e}} \coloneqq \inf_{(\widehat{w}, \widehat{\theta})}\max_{(w,\theta)\in \cW\times \Theta^\star} \big\{ \mathbb{P}_{w,\theta} \big( (\widehat{w}, \widehat{\theta}) \neq (w,\theta) \big)\big\} , \label{eq:error-prob-two-hypotheses-finite-LB}
\end{equation}
where the infimum is taken over all possible tests $(\widehat{w}, \widehat{\theta})$ constructed from the dataset.

To continue, armed with the dataset $\cD$ with $N$ samples generated independently, we denote $\mu^{w,\theta}$ (resp.~$\mu^{w,\theta}_h(s,a)$) as the distribution vector (resp.~distribution) of each sample tuple $(s_h, a_h, s_h')$ at time step $h$ under the nominal transition kernel $P^{w,\theta}$ associated with $\mathcal{M}_w^{\theta}$. With this in mind, combined with Fano's inequality from \citet[Theorem~2.2]{tsybakov2009introduction} and  the additivity of the KL divergence (cf.~\citet[Page~85]{tsybakov2009introduction}), we obtain
\begin{align}
p_{\mathrm{e}} &  \geq 1 - N\frac{ \mathop{\max}\limits_{(w,\theta), (\widetilde{w},\widetilde{\theta}) \in \cW\times \Theta^\star, (w,\theta)\neq (\widetilde{w},\widetilde{\theta}) } \mathsf{KL} \big(\mu^{w,\theta} \mymid \mu^{w,\theta}\big)  + \log 2}{\log |\cW| |\Theta^\star|}   \nonumber\\
     &\overset{\mathrm{(i)}}{\geq} 1 - \frac{8N}{H}\mathop{\max}\limits_{(w,\theta), (\widetilde{w},\widetilde{\theta}) \in \cW\times \Theta^\star, (w,\theta)\neq (\widetilde{w},\widetilde{\theta}) } \mathsf{KL} \big(\mu^{w,\theta} \mymid \mu^{w,\theta}\big)  -  \frac{\log2}{H} \notag \\
    &\overset{\mathrm{(ii)}}{\geq} \frac{1}{2} - \frac{8N}{H}\max_{(w,\theta), (\widetilde{w},\widetilde{\theta}) \in \cW\times \Theta^\star, (w,\theta)\neq (\widetilde{w},\widetilde{\theta}) } \mathsf{KL} \big(\mu^{w,\theta} \mymid \mu^{w,\theta}\big) 
    \label{eq:finite-remainder-KL}
\end{align}
where (i) and (ii) holds by $|\cW| | \Theta^\star| \geq 2 (e^{H/8} - 1) \geq e^{H/8}$ as long as $H \geq 16 \log2$.

To continue, applying the chain rule of the KL divergence \citep[Lemma 5.2.8]{duchi2018introductory} with the dataset $\cD$ generated independently yields:
\begin{align}
 \mathsf{KL} \big(\mu^{w,\theta} \mymid \mu^{w,\theta}\big) & =\sum_{h=1}^{H} \mathop{\mathbb{E}}\limits_{(s,a) \sim \mu_h^{w,\theta}(s,a) }\left[\mathsf{KL}\big(P^{w,\theta}_h(\cdot \mymid s,a) \parallel P^{ \widetilde{w}, \widetilde{\theta}}_h(\cdot \mymid s,a) \big)\right] \notag \\
& \overset{\mathrm{(i)}}{= } \sum_{h=1}^{H} \sum_{s\in \{x_w, x_{\widetilde{w}}\}, a\in\{0,1\}} \mu_h^{w,\theta}(s,a) \left[\mathsf{KL}\big(P^{w,\theta}_h(\cdot \mymid, s,a) \parallel P^{ \widetilde{w}, \widetilde{\theta}}_h(\cdot \mymid, s,a) \big)\right] \notag \\
& \leq \frac{1}{SAH} \sum_{h=1}^{H} \sum_{s\in \{x_w, x_{\widetilde{w}}\}, a\in\{0,1\}}  \left[\mathsf{KL}\big(P^{w,\theta}_h(\cdot \mymid, s,a) \parallel P^{ \widetilde{w}, \widetilde{\theta}}_h(\cdot \mymid, s,a) \big)\right], \label{eq:KL-summary}
\end{align}
where (i) follows from the fact $P^{w,\theta}_h(\cdot \mymid, s,a)$ and $P^{ \widetilde{w}, \widetilde{\theta}}_h(\cdot \mymid, s,a)$ only differs from each other on state $x_w, x_{\widetilde{w}}$ (see the definitions in \eqref{eq:Ph-construction-lower-bound-finite-theta}), and the last inequality holds by noticing $\mu_h^{w,\theta}(s,a) \leq \sum_{a\in\{0,1\}} \mu_h^{w,\theta}(s,a) = \frac{1}{SAH}$.

Consequently, now we turn to focus on terms in \eqref{eq:KL-summary} in different cases of the uncertainty level $\ror$.

\begin{itemize}
\item When $0< \ror \leq \frac{c_2}{2H}$. When $w=\widetilde{w}$, it is clear that
\begin{align}
 \sum_{s\in \{x_w, x_{\widetilde{w}}\}, a\in\{0,1\}}  \mathsf{KL}\big(P^{w,\theta}_h(\cdot \mymid, s,a) \parallel P^{ \widetilde{w}, \widetilde{\theta}}_h(\cdot \mymid, s,a) \big) = 0
\end{align}
as long as $\theta_h = \widetilde{\theta}_h$. Then if $\theta_h \neq \widetilde{\theta}_h$, without loss of generality, we suppose $\theta_h = 0$ and $\widetilde{\theta}_h = 1$, which indicates
\begin{align}
	P^{w,\theta}_h(0 \mymid x_w, 0) = 1-p \quad \text{and} \quad P^{w,\widetilde{\theta} }_h(0 \mymid x_w, 0) = 1-q.
\end{align}

Applying Lemma~\ref{lem:KL-key-result} gives
\begin{align}
	\mathsf{KL}\big(P^{w,\theta}_h(0 \mymid x_w, 0) \parallel P^{w,\widetilde{\theta} }_h(0 \mymid x_w, 0)  \big) & \leq \frac{(p-q)^2}{q(1-q)} \overset{\mathrm{(i)}}{=} \frac{\Delta^2}{q(1-q)} \notag\\
    & \overset{\mathrm{(ii)}}{=} \frac{ (c_5)^2  \varepsilon^2 }{H^4 q(1-q)} \leq  \frac{ 4(c_5)^2  \varepsilon^2 }{c_2 H^3},
    \label{eq:chi2-finite-KL-bounded1}
\end{align}
where (i) and (ii) follows from the definitions in \eqref{eq:p-q-defn-infinite} or \eqref{eq:Delta-chosen},  and the last inequality arises from $q  = p -\Delta \geq \frac{c_2}{2H} $ (see \eqref{eq:p-q-defn-infinite2}) and $1-q \geq 1-p \geq 1-\frac{c_2}{H} \geq \frac{1}{2}$.

The same bound can be established for $\mathsf{KL}\big(P^{w,\theta}_h(0 \mymid x_w, 1) \parallel P^{w,\widetilde{\theta} }_h(0 \mymid x_w, 1)  \big)$. In addition, it is easily verified that when $w\neq \widetilde{w}$ and $\theta_h \neq \theta^{\mathsf{base}}_h$ (resp.~$\widetilde{\theta}_h \neq \theta^{\mathsf{base}}_h$), the same bound can be developed for $\mathsf{KL}\big(P^{w, \theta}_h(0 \mymid x_w, 0) \parallel P^{\widetilde{w},\widetilde{\theta} }_h(0 \mymid x_w, 0)  \big)$ and $\mathsf{KL}\big(P^{w, \theta}_h(0 \mymid x_w, 1) \parallel P^{\widetilde{w},\widetilde{\theta} }_h(0 \mymid x_w, 1)  \big)$ (resp.~$\mathsf{KL}\big(P^{w, \theta}_h(0 \mymid x_{\widetilde{w}}, 0) \parallel P^{\widetilde{w},\widetilde{\theta} }_h(0 \mymid x_{\widetilde{w}}, 0)  \big)$ and $\mathsf{KL}\big(P^{w, \theta}_h(0 \mymid x_{\widetilde{w}}, 1) \parallel P^{\widetilde{w},\widetilde{\theta} }_h(0 \mymid x_{\widetilde{w}}, 1)  \big)$).

Summing up the results with the fact in \eqref{eq:chi2-finite-KL-bounded1}, we arrive at
\begin{align}
\sum_{s\in \{x_w, x_{\widetilde{w}}\}, a\in\{0,1\}}  \mathsf{KL}\big(P^{w,\theta}_h(\cdot \mymid, s,a) \parallel P^{ \widetilde{w}, \widetilde{\theta}}_h(\cdot \mymid, s,a) \big) \leq \frac{ 16(c_5)^2  \varepsilon^2 }{c_2 H^3}. \label{eq:case1-KL-bound}
\end{align}

\item When $\frac{c_2}{2H} < \ror \leq 1-c_0$. Following the same pipeline, it then boils down to control the main term as below:
\begin{align}
\mathsf{KL}\big(P^{w,\theta}_h(0 \mymid x_w, 0) \parallel P^{w,\widetilde{\theta} }_h(0 \mymid x_w, 0)  \big) & \leq \frac{(p-q)^2}{q(1-q)} \overset{\mathrm{(i)}}{=}\frac{\Delta^2}{q(1-q)} \notag\\
    & \overset{\mathrm{(ii)}}{=} \frac{ (c_5)^2 \ror^2 \varepsilon^2 }{H^2 q(1-q)} \leq  \frac{ 2(c_5)^2 \ror \varepsilon^2 }{c_0 H^2},
\end{align}
where (i) and (ii) follows from the definitions in \eqref{eq:p-q-defn-infinite} or \eqref{eq:Delta-chosen}. Here, the last inequality arises from 
\begin{align}
1-q &\geq 1-p   = 1- (1+\frac{c_1}{H})\ror \overset{\mathrm{(i)}}{\geq}  c_0 - \frac{c_1}{H} \overset{\mathrm{(ii)}}{\geq}  \frac{c_0}{2} \notag \\
p &\geq q = p -\Delta\overset{\mathrm{(iii)}}{ \geq}  \ror,
\end{align}
where (ii) holds by the definition of $c_1 = \frac{c_0}{2}$, and (iii) follows from \eqref{eq:infinite-p-q-bound}.
Consequently, we arrive at
\begin{align}
\sum_{s\in \{x_w, x_{\widetilde{w}}\}, a\in\{0,1\}}  \mathsf{KL}\big(P^{w,\theta}_h(\cdot \mymid, s,a) \parallel P^{ \widetilde{w}, \widetilde{\theta}}_h(\cdot \mymid, s,a) \big) \leq  \frac{ 8(c_5)^2 \ror \varepsilon^2 }{c_0 H^2}. \label{eq:case2-KL-bound}
\end{align}

\end{itemize}

Summing up \eqref{eq:case1-KL-bound} and \eqref{eq:case2-KL-bound}, we achieve for any $(w,\theta), (\widetilde{w},\widetilde{\theta}) \in \cW\times \Theta^\star$ with  $(w,\theta)\neq (\widetilde{w},\widetilde{\theta}) $ and any time step $h\in[H]$
\begin{align}
\sum_{s\in \{x_w, x_{\widetilde{w}}\}, a\in\{0,1\}}  \mathsf{KL}\big(P^{w,\theta}_h(\cdot \mymid, s,a) \parallel P^{ \widetilde{w}, \widetilde{\theta}}_h(\cdot \mymid, s,a) \big) \leq  \frac{ 16(c_5)^2  \varepsilon^2 }{c_0 c_2 H^2} \max\{ \ror, 1/H \}. \label{eq:KL-bound-sum-final}
\end{align}

Plugging \eqref{eq:KL-bound-sum-final} back to \eqref{eq:KL-summary} and then \eqref{eq:finite-remainder-KL} leads to the following fact:
\begin{align}
p_{\mathrm{e}} &  \geq \frac{1}{2} - \frac{8N}{H}\max_{(w,\theta), (\widetilde{w},\widetilde{\theta}) \in \cW\times \Theta^\star, (w,\theta)\neq (\widetilde{w},\widetilde{\theta}) } \mathsf{KL} \big(\mu^{w,\theta} \mymid \mu^{w,\theta}\big) \notag \\
& \geq \frac{1}{2} - \frac{8N}{H}\max_{(w,\theta), (\widetilde{w},\widetilde{\theta}) \in \cW\times \Theta^\star, (w,\theta)\neq (\widetilde{w},\widetilde{\theta}) }  \frac{1}{SAH} \sum_{h=1}^{H} \sum_{s\in \{x_w, x_{\widetilde{w}}\}, a\in\{0,1\}}  \left[\mathsf{KL}\big(P^{w,\theta}_h(\cdot \mymid, s,a) \parallel P^{ \widetilde{w}, \widetilde{\theta}}_h(\cdot \mymid, s,a) \big)\right] \notag \\
& \geq \frac{1}{2} - \frac{ 128N(c_5)^2  \varepsilon^2 }{c_0 c_2 SAH^3} \max\{ \ror, 1/H \} \geq \frac{1}{4} \label{eq:error-lower-bound}
    \end{align}
    as long as the sample size $N$ of the dataset is selected as
    \begin{align}
    N \leq \frac{c_0 c_2 SAH^3 \min\{1/\ror, H\} }{512(c_5)^2\varepsilon^2}.  \label{eq:error-lower-bound-condition}
    \end{align}

\paragraph{Step 3: summing up the results together.}
We suppose that there exists an estimator $\widehat{\pi}$ such that
\begin{align}
\max_{(w,\theta)\in\cW\times \Theta^\star} \mathbb{P}_{w,\theta} \left[\max_{s\in \cX\cup \cY} \left\{ V_{1}^{\star,\ror, w,\theta}(s) - V_{1}^{\widehat{\pi},\ror, w,\theta}(s) \right\} \geq \varepsilon \right] <\frac{1}{4}, \label{eq:final-goal-over-class}
\end{align}%
then according to \eqref{eq:lower-bound-goal}, we necessarily have
\begin{align}
 \forall w\in\cW: \quad \max_{\theta \in  \Theta^\star} \mathbb{P}_{w,\theta} \left[\left\{ V_{1}^{\star,\ror, w,\theta}(x_w) - V_{1}^{\widehat{\pi},\ror, w,\theta}(x_w) \right\} \geq \varepsilon \right] <\frac{1}{4}. \label{eq: all-s-meet-epsilon}
\end{align}

To meet \eqref{eq: all-s-meet-epsilon} for any $w\in\cW$, we require 
\begin{align}
     \forall \theta\in\Theta^\star:  \mathbb{P}_{w,\theta} \left\{ V_{1}^{\star,\ror, w,\theta}(x_w) - V_{1}^{\widehat{\pi},\ror, w,\theta}(x_w) < \varepsilon\right\} \geq \frac{3}{4},
    \label{eq:assumption-theta-small-LB-finite}
\end{align}
which in view of \eqref{eq:lower-bound-assumption} indicates that we necessarily have
\begin{align}
     \forall \theta\in\Theta^\star: \quad \mathbb{P}_{w,\theta} \left\{  \sum_{h=1}^H \big\|\widehat{\pi}_h(\cdot\mymid x_w) - \pi_h^{\star, w,\theta}(\cdot\mymid x_w) \big\|_1 < \frac{H}{8} \right\} \geq \frac{3}{4}.
\end{align}

As a consequence,  \eqref{eq:consequence-of-sum-H} indicates
\begin{align}
  \forall \theta\in\Theta^\star:  \mathbb{P}_{w,\theta} \left[(\widehat{w},\widehat{\theta}) = (w,\theta) \right] \geq \frac{3}{4} . \label{eq:consequence-of-sum-H-to-final-step}
\end{align}

Applying the fact in \eqref{eq:consequence-of-sum-H-to-final-step} to all $w\in\cW$ leads to one necessarily has
\begin{align}
 \forall (w,\theta)\in \cW \times \Theta^\star: \quad  \mathbb{P}_{w,\theta} \left[(\widehat{w},\widehat{\theta}) = (w,\theta) \right] \geq \frac{3}{4}
\end{align}
to achieve \eqref{eq:final-goal-over-class}.

However, this would contract with \eqref{eq:error-lower-bound} as long as the sample size condition in  \eqref{eq:error-lower-bound-condition} is satisfied. Thus, if the sample size obeys the condition \eqref{eq:error-lower-bound-condition}, we can't achieve an estimate $\widehat{\pi}$ that satisfies \eqref{eq:final-goal-over-class}, which complete the proof.

\subsection{Proof of the auxiliary facts}
\subsubsection{Proof of Lemma~\ref{lem:finite-lb-value-theta} }\label{proof:lem:finite-lb-value-theta}

As all RMDPs within $ \cM(\cW, \Theta)$ are constructed analogously over each $w\in\cW$ and $\theta \in \Theta^\star$, in this section, we shall focus on one specific RMDP $\cM_w^\theta \in \cM(\cW, \Theta)$, whose facts can be carried on for all other RMDPs in $\cM(\cW, \Theta)$ directly.

\paragraph{Step 1: ordering the robust value function over different states.} Before proceeding,  we introduce several facts and notations that are useful throughout this section. First, we observe that for any $\cM_w^\theta$ and any policy $\pi$: at the final step $H+1$, 
\begin{align}
\forall s\in \cX \cup \cY: \quad V_{H+1}^{\pi,\ror, w,\theta}(s) = 0. \label{eq:H+1-step}
\end{align}
Then for the step $H$, we can easily verified that
\begin{align}
	\forall s\in\cY: \quad V_{H}^{\pi,\ror, w,\theta}(s) &= \mathbb{E}_{a\sim\pi_H(\cdot \mymid s)}\left[r_H( s ,a) + \inf_{ \cP \in \unb^{\sigma}(P^{w,\theta}_{H, s, a})}  \cP  V_{H+1}^{\pi,\ror, w,\theta}\right] =1 \notag \\
	\forall s\in\cX: \quad V_{H}^{\pi,\ror, w,\theta}(s) &= \mathbb{E}_{a\sim\pi_H(\cdot \mymid s)}\left[r_H( s ,a) + \inf_{ \cP \in \unb^{\sigma}(P^{w,\theta}_{H, s, a})}  \cP  V_{H+1}^{\pi,\ror, w,\theta}\right] = 0,
\end{align}
which holds by \eqref{eq:H+1-step} and the definition of the reward function (see \eqref{eq:rh-construction-lower-bound-infinite}). The above fact directly indicates that 
\begin{align}
    \forall (s,s') \in \cX \setminus\{x_w\} \times \cY :&\quad  \min_{\widetilde{s}\in\cS} V_{H}^{\pi,\ror, w,\theta}(\widetilde{s}) = V_{H}^{\pi,\ror, w,\theta}(x_w) \leq V_{H}^{\pi,\ror, w,\theta}(s) < V_{H}^{\pi,\ror, w,\theta}(s'), \notag \\
    \forall (s,s')\in \cY \times \cY:& \quad  V_{H }^{\pi,\ror, w,\theta}(s) =V_{H }^{\pi,\ror, w,\theta}(s').  \label{eq:ordering-states-base-case}
\end{align}

Then we introduce a claim which we will proof by induction in a moment as below:
\begin{align}
\forall (h,s,s')\in [H] \times \cX \setminus\{x_w\} \times \cY :&\quad  V_{h}^{\pi,\ror, w,\theta}(x_w) \leq V_{h}^{\pi,\ror, w,\theta}(s) < V_{h }^{\pi,\ror, w,\theta}(s')\notag \\
    \forall (s,s')\in \cY \times \cY:& \quad  V_{h }^{\pi,\ror, w,\theta}(s) =V_{h }^{\pi,\ror, w,\theta}(s'). \label{eq:ordering-states}
\end{align}
Note that the base case when the time step is $H+1$ is verified in \eqref{eq:ordering-states-base-case}. Assuming that the following fact at time step $h+1$ holds
\begin{align}
 \forall (s,s') \in \cX \setminus\{x_w\} \times \cY :&\quad  \min_{\widetilde{s}\in\cS} V_{h+1}^{\pi,\ror, w,\theta}(\widetilde{s}) =V_{h+1}^{\pi,\ror, w,\theta}(x_w) \leq V_{h+1}^{\pi,\ror, w,\theta}(s) < V_{h+1}^{\pi,\ror, w,\theta}(s'),
 \notag \\
    \forall (s,s')\in \cY \times \cY:& \quad  V_{h+1}^{\pi,\ror, w,\theta}(s) =V_{h+1 }^{\pi,\ror, w,\theta}(s'), \label{eq:ordering-states-assumption}
\end{align}
the rest of the proof focuses on proving the same property for time step $h$. For RMDP $\cM_w^\theta \in \cM(\cW, \Theta)$ and any policy $\pi$, we characterize the robust value function of different states separately:
\begin{itemize}
	\item {\em For state $s\in\cY$.} We observe that for any $s\in\cY$,
\begin{align}
V_{h}^{\pi,\ror, w,\theta}(s) &= \mathbb{E}_{a\sim\pi_h(\cdot \mymid s)}\bigg[r_h( s ,a) + \inf_{ \cP \in \unb^{\sigma}(P^{w,\theta}_{h, s, a})}  \cP  V_{h+1}^{\pi,\ror, w,\theta} \bigg] \notag \\
&\overset{\mathrm{(i)}}{=} 1 +  \mathbb{E}_{a\sim \pi_h(\cdot  \mymid s)} \left[ \underline{P}_h^{w,\theta}(s \mymid s, a) V_{h+1}^{\pi,\ror, w,\theta}(s) \right] + \ror  V_{h+1}^{\pi,\ror, w,\theta}(x_w)  \notag \\
&=  1 + (1-\ror) V_{h+1}^{\pi,\ror, w,\theta}(s) + \ror  V_{h+1}^{\pi,\ror, w,\theta}(x_w) , \label{eq:tv-s-value1}
\end{align}
where (i) holds by $r_h(s,a)=1$ for all $s\in \cY$  (see \eqref{eq:rh-construction-lower-bound-infinite}), the fact that $\min_{\widetilde{s}\in\cS} V_{h+1}^{\pi,\ror, w,\theta}(\widetilde{s}) =V_{h+1}^{\pi,\ror, w,\theta}(x_w)$ induced by the induction assumption (cf.~\eqref{eq:ordering-states-assumption}) and the definition of $\underline{P}_h^{w,\theta}(s \mymid s, a)$ in \eqref{eq:infinite-lw-def-p-q}, and the last equality follows from $P^{w,\theta}(s\mymid s,a)=1$ for all $(s,a)\in \cY \times \cA_{\cM}$.
Resorting to the induction assumption in \eqref{eq:ordering-states-assumption}, we have
\begin{align}
 \forall (s,s')\in \cY \times \cY:& \quad  V_{h }^{\pi,\ror, w,\theta}(s) =V_{h }^{\pi,\ror, w,\theta}(s'). \label{eq:induction-result-1}
\end{align}
	\item {\em For state $x_w$.} First, the robust value function at state $x_w$ obeys
\begin{align}
& \quad  V_{h}^{\pi,\ror, w,\theta}(x_w) \notag \\
&= \mathbb{E}_{a\sim\pi_h(\cdot \mymid x_w)}\left[r_h( x_w ,a) + \inf_{ \cP \in \unb^{\sigma}(P^{w,\theta}_{h, x_w, a})}  \cP  V_{h+1}^{\pi,\ror, w,\theta}\right] \notag \\
  & \overset{\mathrm{(i)}}{=}  0 +  \pi_h(\theta_h \mymid x_w )\inf_{ \cP \in \unb^{\sigma}(P^{w,\theta}_{h, x_w, \theta_h})}  \cP  V_{h+1}^{\pi,\ror, w,\theta}   +   \pi_h(1-\theta_h \mymid x_w )\inf_{ \cP \in \unb^{\sigma}(P^{w,\theta}_{h, x_w, 1-\theta_h})}  \cP  V_{h+1}^{\pi,\ror, w,\theta}  \notag \\
  & \overset{\mathrm{(ii)}}{=} \pi_h(\theta_h \mymid x_w)\Big[ \underline{p}  V_{h+1}^{\pi,\ror, w,\theta}(y_w) + \left(1- \underline{p}\right)  V_{h+1}^{\pi,\ror, w,\theta}(x_w)\Big] \notag \\
  &\quad + \pi_h(1-\theta_h \mymid x_w)\Big[ \underline{q}  V_{h+1}^{\pi,\ror, w,\theta}(y_w) + \left(1-\underline{q} \right)  V_{h+1}^{\pi,\ror, w,\theta}(x_w)\Big] \notag\\
    & \overset{\mathrm{(iii)}}{=} x_h^{\pi, w,\theta} V_{h+1}^{\pi,\ror, w,\theta}(y_w) + (1-x_h^{\pi, w,\theta}) V_{h+1}^{\pi,\ror, w,\theta}(x_w) \label{eq:tv-s0-value-def-xw-1} \\
    & \leq (1-\ror) V_{h+1}^{\pi,\ror, w,\theta}(y_w) + \ror  V_{h+1}^{\pi,\ror, w,\theta}(x_w). \label{eq:tv-s0-value-def}
\end{align}
where (i) uses the definition of the robust value function and the reward function in \eqref{eq:rh-construction-lower-bound-infinite}, (ii) uses the induction assumption in \eqref{eq:ordering-states-assumption} so that the minimum is attained by picking the choice specified in \eqref{eq:infinite-lw-p-q-perturb-inf} to absorb probability mass  to state $x_w$, and (iii) holds by plugging in the definition \eqref{eq:finite-x-h-theta} of $x_h^{\pi,w,\theta}$ in (iii). Finally, the last inequality follows from the fact that function $f(x) \defn x V_{h+1}^{\pi,\ror, w,\theta}(y_w) + (1-x) V_{h+1}^{\pi,\ror, w,\theta}(x_w)$ is monotonically increasing with $x$ since $V_{h+1}^{\pi,\ror, w,\theta}(y_w) > V_{h+1}^{\pi,\ror, w,\theta}(x_w)$ (see the induction assumption \eqref{eq:ordering-states-assumption}), and the fact $x_h^{\pi,w,\theta} \leq 1 - \ror$.

\item {\em For state $s\in \cX \setminus \{x_w\}$.}
Then we consider other states $s\in \cX \setminus \{x_w\}$. Before proceeding, analogous to \eqref{eq:finite-x-h-theta}, we define 
\begin{equation}
 x_{\mathsf{base}}^s = (\underline{p}+\Delta) \pi_h(\theta^{\mathsf{base}}_h \mymid s) + (\underline{q}+\Delta) \pi_h(1-\theta^{\mathsf{base}}_h\mymid s).
\end{equation}

Recall that the nominal transition kernel at any state $s\in \cX \setminus \{x_w\}$ are the same $\{P_{h,s,a}^\star\}_{h\in[H]}$ for all $a\in\cA_{\cW}$ associated with the basic $\theta^{\mathsf{base}} \in \Theta$ (see the definitions of the transition kernels in \eqref{eq:Ph-construction-lower-bound-finite-theta} and \eqref{eq:Ph-construction-lower-bound-finite-theta2}).
Consequently, for any $s\in \cX \setminus \{x_w\}$, following the same argument pipeline of \eqref{eq:tv-s0-value-def}, we arrive at
\begin{align}
 V_{h}^{\pi,\ror, w,\theta}(s) &= \pi_h(\theta^{\mathsf{base}}_h \mymid s)\Big[ (\underline{p}+\Delta)  V_{h+1}^{\pi,\ror, w,\theta}(s^{x\rightarrow y}) + \left(1- p -\Delta\right)  V_{h+1}^{\pi,\ror, w,\theta}(s) + \ror V_{h+1}^{\pi,\ror, w,\theta}(x_w) \Big] \notag \\
  &\quad + \pi_h(1-\theta^{\mathsf{base}}_h \mymid s)\Big[ (\underline{q} +\Delta)  V_{h+1}^{\pi,\ror, w,\theta}(s^{x\rightarrow y}) + \left(1-p \right)  V_{h+1}^{\pi,\ror, w,\theta}(s)  + \ror V_{h+1}^{\pi,\ror, w,\theta}(x_w) \Big] \notag \\
  & = x_{\mathsf{base}}^s V_{h+1}^{\pi,\ror, w,\theta}(s^{x\rightarrow y}) + (1-x_{\mathsf{base}}^s -\ror) V_{h+1}^{\pi,\ror, w,\theta}(s) + \ror V_{h+1}^{\pi,\ror, w,\theta}(x_w)\label{eq:tv-s0-value-def-cX-1}\\
  & \overset{\mathrm{(i)}}{=} x_{\mathsf{base}}^s V_{h+1}^{\pi,\ror, w,\theta}(y_w) + (1-x_{\mathsf{base}}^s -\ror) V_{h+1}^{\pi,\ror, w,\theta}(s) + \ror V_{h+1}^{\pi,\ror, w,\theta}(x_w)\label{eq:tv-s0-value-def-cX-1}\\
 & < (1-\ror) V_{h+1}^{\pi,\ror, w,\theta}(s^{x\rightarrow y}) + \ror  V_{h+1}^{\pi,\ror, w,\theta}(s), \label{eq:tv-s0-value-def-cX-2}
\end{align}
where (i) holds by $V_{h+1}^{\pi,\ror, w,\theta}(s) = V_{h+1}^{\pi,\ror, w,\theta}(s')$ for any two states $s,s'\in\cY$ (see \eqref{eq:tv-s0-value-def}), and the last inequality holds by $V_{h+1}^{\pi,\ror, w,\theta}(s) < V_{h+1}^{\pi,\ror, w,\theta}(s^{x\rightarrow y})$ induced by the induction assumption in \eqref{eq:ordering-states-assumption}.

In addition, to compare the robust value function $V_{h}^{\pi,\ror, w,\theta}(x_w)$ to that of other states $s\in\cX\setminus \{x_w\}$, we recall the definitions in \eqref{eq:finite-x-h-theta} and then introduce the following fact
\begin{align}
x_h^{\pi, w, \theta} &= \underline{p} \pi_h(\theta_h\mymid x_w) + \underline{q} \pi_h(1-\theta_h\mymid x_w) \notag \\
& \leq \underline{p} \leq (\underline{p}+\Delta) \pi_h(\theta^{\mathsf{base}}_h \mymid s) + \underline{p}  \pi_h(1-\theta^{\mathsf{base}}_h\mymid s) \notag \\
& =  (\underline{p}+\Delta) \pi_h(\theta^{\mathsf{base}}_h \mymid s) + (\underline{q} + \Delta ) \pi_h(1-\theta^{\mathsf{base}}_h\mymid s)  = x_{\mathsf{base}}^s,
\end{align}
which comes from the fact $p\geq q$ and the facts in \eqref{eq:infinite-lw-p-q-perturb-inf} and \eqref{eq:infinite-p-q-bound}.

With this in mind, continuing from \eqref{eq:tv-s0-value-def-xw-1}, we arrive at that for any $s\in \cX$:
\begin{align}
V_{h}^{\pi,\ror, w,\theta}(x_w)  &= x_h^{\pi, w,\theta} V_{h+1}^{\pi,\ror, w,\theta}(y_w) + (1-x_h^{\pi, w,\theta}) V_{h+1}^{\pi,\ror, w,\theta}(x_w) \notag \\
& \leq x_{\mathsf{base}}^s V_{h+1}^{\pi,\ror, w,\theta}(y_w) + (1-x_{\mathsf{base}}^s) V_{h+1}^{\pi,\ror, w,\theta}(x_w)  \notag \\
& \leq x_{\mathsf{base}}^s V_{h+1}^{\pi,\ror, w,\theta}(y_w) + (1-x_{\mathsf{base}}^s-\ror) V_{h+1}^{\pi,\ror, w,\theta}(s) + \ror V_{h+1}^{\pi,\ror, w,\theta}(x_w) \notag \\
&=  V_{h}^{\pi,\ror, w,\theta}(s) \label{eq:tv-s0-value-def-cX-3}
\end{align}
where the last equality holds by \eqref{eq:tv-s0-value-def-cX-1}.

\end{itemize}

Summing up \eqref{eq:tv-s0-value-def-cX-3}, then \eqref{eq:tv-s-value1}, and \eqref{eq:tv-s0-value-def-cX-2}, we verify the induction property at time step $h$ as below 
\begin{align}
\forall (s,s')\in  \cX \setminus\{x_w\} \times \cY :  V_{h}^{\pi,\ror, w,\theta}(x_w) \leq V_{h}^{\pi,\ror, w,\theta}(s) < V_{h }^{\pi,\ror, w,\theta}(s').
\end{align}

Combined above results with \eqref{eq:induction-result-1}, we confirm the claim in \eqref{eq:ordering-states}.

\paragraph{Step 2: deriving the optimal policy and optimal robust value function.}
We shall characterize the optimal policy and corresponding optimal robust value function for different states separately:
\begin{itemize}
	\item {\em For states in $\cX$.} Recall \eqref{eq:tv-s0-value-def-xw-1}
	\begin{align}
 V_{h}^{\pi,\ror, w,\theta}(x_w) &= x_h^{\pi, w,\theta} V_{h+1}^{\pi,\ror, w,\theta}(y_w) + (1-x_h^{\pi, w,\theta}) V_{h+1}^{\pi,\ror, w,\theta}(x_w)
 \label{eq:x_w-value-function}
	\end{align}

and the fact  $V_{h+1}^{\pi,\ror, w,\theta}(y_w) > V_{h+1}^{\pi,\ror, w,\theta}(x_w)$ in \eqref{eq:ordering-states}. We observe that \eqref{eq:x_w-value-function} is monotonicity increasing with respect to $x_h^{\pi, w,\theta}$, and $x_h^{\pi, w,\theta}$ is also increasing in $\pi_h(\theta_h \mymid x_w)$ (refer to the fact $\underline{p}\geq \underline{q}$ since $p\geq q$; see \eqref{eq:p-q-defn-infinite} and \eqref{eq:infinite-lw-p-q-perturb-inf}). Consequently, the optimal policy and optimal robust value function in state $x_w$ thus obey
\begin{align}
  \forall h\in[H]: \quad   &\pi_h^{\star,w,\theta}(\theta_h \mymid x_w) = 1 \notag \\
  &   V_{h}^{\star,\ror, w,\theta}(x_w) = \underline{p} V_{h+1}^{\star,\ror, w,\theta}(y_w)  + \left[ 1 - \underline{p} \right]  V_{h+1}^{\star,\ror, w,\theta}(x_w) \label{eq:infinite-lb-optimal-policy}.
\end{align}

Similarly, for any state $s\in \cX\setminus \{x_w\}$, recalling \eqref{eq:tv-s0-value-def-cX-1} yields
\begin{align}
V_{h}^{\pi,\ror, w,\theta}(s) &= x_{\mathsf{base}}^s V_{h+1}^{\pi,\ror, w,\theta}(y_w) + (1-x_{\mathsf{base}}^s -\ror) V_{h+1}^{\pi,\ror, w,\theta}(s) + \ror V_{h+1}^{\pi,\ror, w,\theta}(x_w),
\end{align}
which indicates $V_{h}^{\pi,\ror, w,\theta}(s)$ achieves the maximum when $x_{\mathsf{base}}^s = (\underline{p}+\Delta) \pi_h(\theta^{\mathsf{base}}_h \mymid s) + (\underline{q}+\Delta) \pi_h(1-\theta^{\mathsf{base}}_h\mymid s)$ attain the maximum. Therefore, the optimal policy in state $s$ satisfies
\begin{align}
	\pi_h^{\star,w,\theta}(\theta^{\mathsf{base}}_h \mymid s) = 1.
\end{align}

	\item {\em For states $s\in  \cY$.} Recall the transitions in \eqref{eq:Ph-construction-lower-bound-finite-theta} and \eqref{eq:Ph-construction-lower-bound-finite-theta2}. Considering that the action does not influence the state transition for all states $s\in \cY$, without loss of generality, we choose the robust optimal policy obeying
\begin{align}\label{eq:infinite-lower-optimal-pi}
    \forall s \in\cY: \quad \pi_h^{\star,w,\theta}( \theta_h \mymid s) = 1.
\end{align}
\end{itemize}

\subsubsection{Proof of claim \eqref{eq:lower-bound-assumption}} \label{proof:eq:tv-Value-0-recursive}

Recalling \eqref{eq:finite-lemma-value-0-pi-theta} and \eqref{eq:finite-lemma-value-0-pi-theta2}, we first consider a more general form
\begin{align}
&V_{h}^{\star,\ror, w,\theta}(x_w) - V_{h}^{\pi,\ror, w,\theta}(x_w) \notag \\
 & = 
 \underline{p} V_{h+1}^{\star,\ror, w,\theta}(y_w) + (1-\underline{p}) V_{h+1}^{\star,\ror, w,\theta}(x_w) - \left( x_h^{\pi, w,\theta}V_{h+1}^{\pi,\ror, w,\theta}(y_w) + \left[ 1 - x_h^{\pi, w,\theta} \right]  V_{h+1}^{\pi,\ror, w,\theta}(x_w) \right) \notag \\
 & = \left( \underline{p} -  x_h^{\pi, w,\theta} \right) V_{h+1}^{\star,\ror, w,\theta}(y_w) + x_h^{\pi, w,\theta} \left( V_{h+1}^{\star,\ror, w,\theta}(y_w) - V_{h+1}^{\pi,\ror, w,\theta}(y_w)  \right)\notag \\
 & \quad +  (1-\underline{p})\left( V_{h+1}^{\star,\ror, w,\theta}(x_w) - V_{h+1}^{\pi,\ror, w,\theta}(x_w)  \right) - \left( \underline{p} -  x_h^{\pi, w,\theta} \right)V_{h+1}^{\pi,\ror, w,\theta}(x_w) \notag \\
 & = x_h^{\pi, w,\theta} \left( V_{h+1}^{\star,\ror, w,\theta}(y_w) - V_{h+1}^{\pi,\ror, w,\theta}(y_w)  \right) + (1-\underline{p})\left( V_{h+1}^{\star,\ror, w,\theta}(x_w) - V_{h+1}^{\pi,\ror, w,\theta}(x_w)  \right) \notag \\
 &\quad  + \left( \underline{p} -  x_h^{\pi, w,\theta} \right) \left(V_{h+1}^{\star,\ror, w,\theta}(y_w)  - V_{h+1}^{\star,\ror, w,\theta}(x_w)\right) \notag \\
 & \geq (1-\underline{p})\left( V_{h+1}^{\star,\ror, w,\theta}(x_w) - V_{h+1}^{\pi,\ror, w,\theta}(x_w)  \right) + \left( \underline{p} -  x_h^{\pi, w,\theta} \right) \left(V_{h+1}^{\star,\ror, w,\theta}(y_w)  - V_{h+1}^{\star,\ror, w,\theta}(x_w)\right) \notag \\
 & \geq (1-\underline{p})\left( V_{h+1}^{\star,\ror, w,\theta}(x_w) - V_{h+1}^{\pi,\ror, w,\theta}(x_w)  \right) \notag \\
 & \quad +\frac{1}{2}(p-q)\big\|\pi_{h}^{\star,w,\theta}(\cdot\mymid x_w)-\pi_{h}(\cdot\mymid x_w)\big\|_{1} \left(V_{h+1}^{\star,\ror, w,\theta}(y_w)  - V_{h+1}^{\star,\ror, w,\theta}(x_w)\right) \label{eq:finite-lb-recursion-theta}
\end{align}
where the last inequality holds by applying \eqref{eq:infinite-lw-p-q-perturb-inf} and deriving as follows:
\begin{align}
 \underline{p} -  x_h^{\pi, w,\theta}&= \big(\underline{p} -  \underline{q}\big) \big(1-\pi_{h}(\theta_{h}\mymid x_w)\big) = ( p-q)\big(1-\pi_{h}(\theta_{h}\mymid x_w)\big) \notag \\
 &=\frac{1}{2}(p-q)\big(1-\pi_{h}(\theta_{h}\mymid x_w)+\pi_{h}(1-\theta_{h}\mymid x_w )\big)=\frac{1}{2}(p-q)\big\|\pi_{h}^{\star,w,\theta}(\cdot\mymid x_w)-\pi_{h}(\cdot\mymid x_w)\big\|_{1}.
\end{align}

To further control \eqref{eq:finite-lb-recursion-theta},  applying Lemma~\ref{lem:finite-lb-value-theta} yields
\begin{align}
	&V_{h}^{\star,\ror, w,\theta}(y_w)  - V_{h}^{\star,\ror, w,\theta}(x_w) \notag \\
	& = 1 + (1-\ror) V_{h+1}^{\star,\ror, w,\theta}(y_w) + \ror  V_{h+1}^{\star,\ror, w,\theta}(x_w)  - \left( \underline{p} V_{h+1}^{\star,\ror, w,\theta}(y_w) + (1-\underline{p}) V_{h+1}^{\star,\ror, w,\theta}(x_w) \right) \notag \\
	& = 1 + (1 - \underline{p} - \ror) \left( V_{h+1}^{\star,\ror, w,\theta}(y_w)- V_{h+1}^{\star,\ror, w,\theta}(x_w) \right) \notag \\
	& = 1 + (1 - p) \left( V_{h+1}^{\star,\ror, w,\theta}(y_w)- V_{h+1}^{\star,\ror, w,\theta}(x_w) \right) \notag \\
	& = \cdots = \sum_{j=0}^{H-h}(1-p)^j, \label{eq:gap-recursive}
	\end{align}
where the penultimate equality holds by \eqref{eq:infinite-lw-p-q-perturb-inf}.
Then, we consider two cases with respect to the uncertainty level $\ror$ to control \eqref{eq:gap-recursive}, respectively:
\begin{itemize}
	\item {\em When $0<\ror \leq \frac{c_2}{2H}$.} Recall $p =  \begin{cases} \frac{c_2}{H}, &\text{if } \ror\leq \frac{c_2}{2H} \\
    1 + \frac{c_1}{H} \sigma & \text{otherwise}
    \end{cases}$. In this case, applying \eqref{eq:gap-recursive}, we have
	\begin{align}
	&V_{h}^{\star,\ror, w,\theta}(y_w)  - V_{h}^{\star,\ror, w,\theta}(x_w) \notag \\
	& = \sum_{j=0}^{H-h}(1-p)^j \geq \sum_{j=0}^{H-h} \left(1-\frac{c_2}{H}\right)^j = \frac{ 1 - \left(1-\frac{c_2}{H}\right)^{H-h+1}}{c_2/H} \geq \frac{2c_2 (H-h+1)}{3} \label{eq:gap-yw-1}
	\end{align}

Here, the final inequality holds by observing
\begin{align}
\left(1-\frac{c_2}{H}\right)^{H-h+1} \leq \exp\left( -\frac{c_2(H-h+1)}{H} \right)  \leq 1 - \frac{2c_2(H-h+1)}{3H}
\end{align}
where the first inequality holds by noticing $c_2< 0.5$ and then $1-x \leq \exp(-x)$, and the last inequality holds by $\exp(-x) \leq 1 - \frac{2x}{3}$ for any $0\leq x\leq 1/2$.

Plugging above fact in \eqref{eq:gap-yw-1} back to \eqref{eq:finite-lb-recursion-theta}, we arrive at
\begin{align}
&V_{h}^{\star,\ror, w,\theta}(x_w) - V_{h}^{\pi,\ror, w,\theta}(x_w) \notag \\
 & \geq (1-\underline{p})\left( V_{h+1}^{\star,\ror, w,\theta}(x_w) - V_{h+1}^{\pi,\ror, w,\theta}(x_w)  \right) \notag \\
 & \quad +\frac{1}{2}(p-q)\big\|\pi_{h}^{\star,w,\theta}(\cdot\mymid x_w)-\pi_{h}(\cdot\mymid x_w)\big\|_{1}\frac{2c_2 (H-h+1)}{3}.  \label{eq:finite-lb-recursion-theta-recursion}
\end{align}

Then invoking the assumption 
 \begin{align}
    \sum_{h=1}^H \big\|\pi_h(\cdot\mymid x_w) - \pi_h^{\star, w,\theta}(\cdot\mymid x_w) \big\|_1 \ge \frac{H}{8}
\end{align}
in \eqref{eq:theta-assumption} and applying \eqref{eq:finite-lb-recursion-theta-recursion} recursively for $h=1,2,\cdots, H$ yields
\begin{align}
V_{1}^{\star,\ror, w,\theta}(x_w) - V_{1}^{\pi,\ror, w,\theta}(x_w) & \geq \frac{c_2 }{3} \sum_{h=1}^H (1-\underline{p})^{h-1} (p-q)(H-h+1) \big\|\pi_{h}^{\star,w,\theta}(\cdot\mymid x_w)-\pi_{h}(\cdot\mymid x_w)\big\|_{1} \notag \\
& \overset{\mathrm{(i)}}{\geq}  \frac{c_2 }{3} \sum_{h=1}^H (1- \frac{c_2}{H})^{h-1} (p-q)(H-h+1) \big\|\pi_{h}^{\star,w,\theta}(\cdot\mymid x_w)-\pi_{h}(\cdot\mymid x_w)\big\|_{1} \notag \\
& \overset{\mathrm{(ii)}}{\geq}  \frac{c_2 }{6} \sum_{h=1}^H  (p-q)(H-h+1) \big\|\pi_{h}^{\star,w,\theta}(\cdot\mymid x_w)-\pi_{h}(\cdot\mymid x_w)\big\|_{1} \notag \\
& \overset{\mathrm{(iii)}}{=}  \frac{c_2 \Delta }{6} \sum_{h=1}^H h \big\|\pi_{H-h+1}^{\star,w,\theta}(\cdot\mymid x_w)-\pi_{H-h+1}(\cdot\mymid x_w)\big\|_{1} \notag \\
& \overset{\mathrm{(iv)}}{\geq} \frac{c_2 \Delta }{6} \sum_{h=1}^{ \left \lfloor H/16\right \rfloor } 2h \geq  \frac{c_2 \Delta }{6} \left \lfloor H/16\right \rfloor \left( \left \lfloor H/16\right \rfloor+ 1 \right), \label{eq:gap-final-1}
\end{align}
where (i) follows from $1-\underline{p} \geq 1-p = 1 -\frac{c_2}{H}$, and (ii) holds by
\begin{align}
\forall h\in[H]: \quad (1- \frac{c_2}{H})^{h-1} \geq (1- \frac{c_2}{H})^{H} \geq \frac{1}{2}
\end{align}
as long as $c_2 \leq \frac{1}{2}$. Here, (iii) arises from the definition of $p,q$ in \eqref{eq:p-q-defn-infinite}; (iv) can be verified by the fact that for any series $0 \leq x_1, x_2,\cdots, x_H \leq x_{\max}$ that obeys $\sum_{h=1}^H x_h \geq y$, one has
\begin{align}
\sum_{h=1}^H x_h h \geq \sum_{h=1}^{\left \lfloor x_{\max} /y\right \rfloor} x_{\max} h,
\end{align}
and taking $x_h = \big\|\pi_{H-h+1}(\cdot\mymid x_w) - \pi_{H-h+1}^{\star, w,\theta}(\cdot\mymid x_w) \big\|_1 \leq 2 = x_{\max}$ and $y = \frac{H}{8}$.

Consequently, observed from \eqref{eq:gap-final-1}, we have
\begin{align}
V_{1}^{\star,\ror, w,\theta}(x_w) - V_{1}^{\pi,\ror, w,\theta}(x_w) &  \geq  \frac{c_2 \Delta }{6} \left \lfloor H/16\right \rfloor \left( \left \lfloor H/16\right \rfloor+ 1 \right) \geq c_3 \Delta H^2 > \varepsilon
\end{align}
holds for some small enough constant $c_3$ and letting $\Delta = \frac{ \varepsilon}{c_3H^2}$.

\item {\em When $\frac{c_2}{2H} < \ror \leq 1 - c_0$.}
Similarly,  recalling $p =  \begin{cases} \frac{c_2}{H}, &\text{if } \ror\leq \frac{c_2}{2H} \\
    \left(1 + \frac{c_1}{H}\right) \sigma & \text{otherwise}
    \end{cases}$ and invoking \eqref{eq:gap-recursive} gives
	\begin{align}
	&V_{h}^{\star,\ror, w,\theta}(y_w)  - V_{h}^{\star,\ror, w,\theta}(x_w) \notag \\
	&= \sum_{j=0}^{H-h}(1-p)^j = \sum_{j=0}^{H-h} \left(1- \left(1+\frac{c_1}{H} \right)\ror\right)^j \notag \\
	& \geq  \frac{ 1 - \left(1-(1+\frac{c_1}{H})\ror\right)^{H-h+1}}{(1+\frac{c_1}{H})\ror} \geq \frac{c_2 (H-h+1)}{3\ror H}, \label{eq:gap-yw-2}
	\end{align}
where the final inequality holds by observing
\begin{align}
\left(1- \left(1+\frac{c_1}{H} \right)\ror\right)^{H-h+1} &\leq \exp\left( -\left(1+\frac{c_1}{H}\right)\ror(H-h+1) \right) \notag \\
& \overset{\mathrm{(i)}}{\leq}  \exp\left( -\frac{c_2}{2H} \left(1+\frac{c_1}{H}\right) (H-h+1)  \right) \leq 1 - \left(1+\frac{c_1}{H}\right)\frac{c_2(H-h+1)}{3H}.
\end{align}
Here,  (i) holds by observing $\frac{c_2}{2H} < \ror$, and the last inequality holds by $\left(1+\frac{c_1}{H}\right) \leq 2$, $c_2 \leq 0.5$, and the fact $\exp(-x) \leq 1 - \frac{2x}{3}$ for any $0\leq x\leq 1/2$.

Plugging above fact in \eqref{eq:gap-yw-2} back to \eqref{eq:finite-lb-recursion-theta} gives
\begin{align}
&V_{h}^{\star,\ror, w,\theta}(x_w) - V_{h}^{\pi,\ror, w,\theta}(x_w) \notag \\
 & \geq (1-\underline{p})\left( V_{h+1}^{\star,\ror, w,\theta}(x_w) - V_{h+1}^{\pi,\ror, w,\theta}(x_w)  \right) \notag \\
 & \quad +\frac{1}{2}(p-q)\big\|\pi_{h}^{\star,w,\theta}(\cdot\mymid x_w)-\pi_{h}(\cdot\mymid x_w)\big\|_{1}\frac{c_2 (H-h+1)}{3\ror H}.  \label{eq:finite-lb-recursion-theta-recursion2}
\end{align}

Following the same routine to achieve \eqref{eq:gap-final-1}, applying \eqref{eq:finite-lb-recursion-theta-recursion2} recursively for $h= 1,2, \cdots, H$ gives

\begin{align}
V_{1}^{\star,\ror, w,\theta}(x_w) - V_{1}^{\pi,\ror, w,\theta}(x_w) & \geq  \sum_{h=1}^H (1-\underline{p})^{h-1} (p-q) \frac{c_2 (H-h+1)}{6\ror H} \big\|\pi_{h}^{\star,w,\theta}(\cdot\mymid x_w)-\pi_{h}(\cdot\mymid x_w)\big\|_{1} \notag \\
& \overset{\mathrm{(i)}}{=}  \frac{c_2 (p-q)}{6\ror H} \sum_{h=1}^H (1- \frac{c_1}{H})^{h-1} (H-h+1) \big\|\pi_{h}^{\star,w,\theta}(\cdot\mymid x_w)-\pi_{h}(\cdot\mymid x_w)\big\|_{1} \notag \\
& \overset{\mathrm{(ii)}}{\geq}  \frac{c_2 \Delta}{12\ror H}  \left \lfloor H/16\right \rfloor \left( \left \lfloor H/16\right \rfloor+ 1 \right)  \label{eq:gap-final-2}
\end{align} 
where (i) follows from $1-\underline{p} = 1-(p-\ror) = 1 -\frac{c_1}{H}\ror$, and (ii) holds by letting $c_1 \leq \frac{1}{2}$ and following the same routine of \eqref{eq:gap-final-1}.

Consequently, \eqref{eq:gap-final-2} yields
\begin{align}
V_{1}^{\star,\ror, w,\theta}(x_w) - V_{1}^{\pi,\ror, w,\theta}(x_w) &  \geq  \frac{c_2 \Delta}{12\ror H}  \left \lfloor H/16\right \rfloor \left( \left \lfloor H/16\right \rfloor+ 1 \right)  \geq \frac{c_4\Delta H}{\ror} > \varepsilon
\end{align}
holds for some small enough constant $c_4$ and letting $\Delta = \frac{ \ror \varepsilon}{c_4H}$.

\end{itemize}